\pdfoutput=1
\documentclass{article}

\usepackage{amssymb,amsmath,amsthm, mathtools}

\usepackage{algorithm}
\usepackage{algorithmic}
\usepackage[table]{xcolor}
\usepackage{graphicx}
\usepackage{dblfloatfix}

\usepackage{float}
\usepackage[makeroom]{cancel}
\usepackage{enumitem}

\theoremstyle{plain}
\usepackage{makecell}
\usepackage{bm}
\usepackage{setspace}

\makeatletter
\newtheorem*{rep@theorem}{\rep@title}
\newcommand{\newreptheorem}[2]{%
\newenvironment{rep#1}[1]{%
 \def\rep@title{#2 \ref{##1}}%
 \begin{rep@theorem}}%
 {\end{rep@theorem}}}
\makeatother

\newreptheorem{theorem}{Theorem}
\newreptheorem{lemma}{Lemma}
\newreptheorem{corollary}{Corollary}
\newcommand{\xopt}{x^*}
\newcommand{\X}{\mathcalorigin{X}}
\newcommand{\diam}{\mathcalorigin{D}_{\X}}
\newcommand{\fopt}{f(\xopt)}

\newcommand{\F}[2]{F_{#1}(#2)}
\newcommand{\gradF}[2]{\nabla F_{#1}(#2)}

\newcommand{\gradFstoc}[2]{\tilde{\nabla} F_{#1}(#2, \xi_{\mathcal{S}_{t,k}})}

\newcommand{\G}[2]{G_{#1}(#2)}
\newcommand{\gradG}[2]{\nabla G_{#1}(#2)}
\newcommand{\g}[2][]{g \ifx\\#1\\\else _{#1}\fi (#2)}
\newcommand{\gradg}[2][]{\nabla g \ifx\\#1\\\else _{#1}\fi (#2)}
\newcommand{\avggradg}[2][]{\tilde{\nabla}  g \ifx\\#1\\\else _{#1}\fi (#2)}
 
\newcommand{\Astoc}[1][]{A(\xi \ifx\\#1\\\else _{#1}\fi )}
\newcommand{\Astoci}[1][]{A(\xi \ifx\\#1\\\else _{#1}\fi )}
\newcommand{\bstoc}[0]{b (\xi)}

\newcommand{\ifo}{\#(ifo) }
\newcommand{\lmo}{\#(lmo) }
\newcommand{\sfo}{\#(sfo) }

\let\norm\undefined 
\DeclarePairedDelimiter\norm{\lVert}{\rVert}

\newcommand{\normsqr}[1]{\norm{#1}^2}
\newcommand{\dotprod}[2]{\langle #1 ,\: #2  \rangle}
\DeclareMathOperator{\Tr}{Tr}
\newcommand{\Expk}[1]{\mathbb{E}_{t, k}\left[#1\right]}
\newcommand{\Exp}[2]{\mathbb{E}\left[#1 \ifx\\#2\\\else \vert #2\fi \right]}
\newcommand{\bigO}[1]{\mathcalorigin{O}\left(#1\right)}

\newcommand{\defeq}{\vcentcolon=}
\newcommand{\reals}[1]{\mathbb{R}^#1}
\newcommand{\R}{\mathbb{R}}
\renewcommand{\H}{\mathcalorigin{H}}

\newcommand{\lagr}{\mathcalorigin{L}}

\DeclareMathOperator*{\argmin}{arg\,min}
\DeclareMathOperator{\supp}{supp}
\DeclareMathOperator{\dist}{dist}

\allowdisplaybreaks

\def\nn{{ \nonumber }}

\DeclareMathAlphabet{\mathcalorigin}{OMS}{cmsy}{m}{n}

\newtheorem{theorem}{Theorem}
\numberwithin{theorem}{section}

\newtheorem{lemma}{Lemma}
\numberwithin{lemma}{section}

\numberwithin{remark}{section}

\newtheorem{corollary}{Corollary}
\numberwithin{corollary}{section}

\numberwithin{definition}{section}

\newtheorem{assumption}{Assumption}
\numberwithin{assumption}{section}

\renewenvironment{proof}{{\bfseries Proof}}{}

\usepackage{microtype}
\usepackage{graphicx}
\usepackage{subfigure}
\usepackage{booktabs} 
\usepackage[tableposition=below]{caption}

\usepackage{hyperref}

\usepackage[accepted]{icml2020}

\icmltitlerunning{Conditional gradient methods for stochastically constrained convex minimization}

\begin{document}

\twocolumn[
\icmltitle{Conditional gradient methods for \\ stochastically constrained convex minimization}



\icmlsetsymbol{equal}{*}

\begin{icmlauthorlist}
\icmlauthor{Maria-Luiza Vladarean}{epfl}
\icmlauthor{Ahmet Alacaoglu}{epfl}
\icmlauthor{Ya-Ping Hsieh}{epfl}
\icmlauthor{Volkan Cevher}{epfl}
\end{icmlauthorlist}

\icmlaffiliation{epfl}{\'Ecole Polytechnique F\'ed\'erale de Lausanne, Switzerland}

\icmlcorrespondingauthor{Maria-Luiza Vladarean}{maria-luiza.vladarean@epfl.ch}

\icmlkeywords{conditional gradient, frank-wolfe, convex optimization, first-order methods, stochastic constraints, almost sure constraints, sdp, semidefinite programming}

\vskip 0.3in
]



\printAffiliationsAndNotice{}  

\begin{abstract}
We propose two novel conditional gradient-based methods for solving structured stochastic convex optimization problems with a large number of linear constraints. Instances of this template naturally arise from SDP-relaxations of combinatorial problems, which involve a number of constraints that is polynomial in the problem dimension. The most important feature of our framework is that only a subset of the constraints is processed at each iteration, thus gaining a computational advantage over prior works that require full passes. Our algorithms rely on variance reduction and smoothing used in conjunction with conditional gradient steps, and are accompanied by rigorous convergence guarantees. Preliminary numerical experiments are provided for illustrating the practical performance of the methods.
\end{abstract}

\section{Introduction}
\label{sec:intro}

We study the following optimization template:
\begin{align}
    \label{eq:prob-form}
    &\min\limits_{x \in \X} \; f(x) \defeq \mathbb{E} \left[ f(x, \xi) \right] \nn \\
            &\Astoc x \in \bstoc \text{ almost surely},
\end{align}
where $ f(x, \xi): \mathbb{R}^{d} \to \mathbb{R}$ are random convex functions with $L_f$-Lipschitz gradient, $\X $ is a convex and compact set of $\R^d$, $\Astoc$ is an $m \times d$ matrix-valued random variable, and $\bstoc$ is a closed and projectable random convex set in $\R^m$.

Stochastically constrained convex optimization problems have recently gained interest in the machine learning community, as they provide a convenient and powerful framework for handling instances subject to a large, or even infinite number of constraints. For example, convex feasibility and optimal control problems have variables lying in a possibly infinite intersection of stochastic, projectable constraint sets, and hence are tackled through this lens by \citet{patrascu2017nonasymptotic}. \citet{xu2018primal} also studies the minimization of a stochastic objective controlled by a very large number of stochastic functional constraints, with application to stochastic linear programming. Finally, put forth by \citet{fercoq2019almost}, extensions to situations where the number of constraints is unknown (e.g.\ online settings) can be modeled by a template highly similar to \eqref{eq:prob-form}, thus addressing important applications such as online portofolio optimization. 


In this paper, we are interested in a class of applications which can benefit from being cast under template~\eqref{eq:prob-form}, namely semidefinite programs (SDPs) with a large number of linear constraints, such as arise in combinatorial optimization. A prominent example in machine learning is the $k$-means clustering problem, whose SDP relaxation comprises $\bigO{d^2}$ linear contraints where $d$ is the number of data samples \cite{peng2007approximating}. Maximum a posteriori estimation \cite{huang2014scalable}, quadratic assignment \cite{burer2005local}, k-nearest neighbor classification \cite{weinberger2009distance} and Sparsest cut \cite{arv} are other relevant SDP instances with linear constraints of order $\bigO{d^2}$ or $\bigO{d^3}$. Coupled with large input dimensions, such SDPs become problematic for most existing methods, due to the high cost of processing the constraints in-full during optimization. 

In contrast, casting such SDPs into \eqref{eq:prob-form} suggests a simple solution: treat the linear constraints stochastically by only accessing a random subset at each iteration, then solve \eqref{eq:prob-form} using cheap gradient methods. However, the bottleneck in executing this idea is that existing methods require the constraint $\mathcal{X}$ to posses an efficient projection oracle, whereas projecting onto the semidefinite cone amounts to full singular value decompositions, an operation that is prohibitively expensive even when the problem dimension is moderate. We hence ask:
\begin{quote}
\emph{Does a scalable method exist for solving \eqref{eq:prob-form} when the set $\mathcal{X}$ does not have an efficient projection oracle?}
\end{quote}

The present work resolves the above challenge in the positive. To this end, we borrow tools from the conditional gradient methods (CGM) \cite{originalfw, revisitingfw}, which rely on the generally cheaper \emph{linear minimization oracles} (lmo), rather than their projection counterparts. In particular, as the Lanczos method enables an efficient lmo computation for the spectrahedron \cite{arora2005fast}, CGMs have already been proposed for solving SDPs \cite{revisitingfw, garber2016sublinear, yurtsever2018conditional, locatello2019stochastic}. However, none of these methods can handle the constraints stochastically.

In a nutshell, our approach relies on \emph{homotopy smoothing} of the stochastic constraints in conjunction with CGM steps and a carefully chosen variance reduction procedure. Our analysis gives rise to two fully stochastic algorithms for solving problem~\eqref{eq:prob-form} without projections onto $\X$. The first of the methods, H-SFW1, relies on a single sample (or fixed batch size) for computing the variance-reduced gradient and converges at a cost of $\mathcal{O}(\epsilon^{-6})$ lmo calls and $\mathcal{O}(\epsilon^{-6})$ stochastic first-order oracle (sfo) calls. The second, H-SPIDERFW, uses batches of increasing size under the SPIDER variance reduction scheme~\cite{fang2018spider} and attains a theoretical complexity of $\mathcal{O}(\epsilon^{-2})$ lmo calls and $\mathcal{O}(\epsilon^{-4})$ sfo calls. The difference in convergence rates emphasizes the trade-off between between the computational cost per-iteration and the number of iterations required to reach the constrained optimum.

\section{Related Work}
\label{sec:related-work}
The present work lies at the intersection of several lines of research, whose relevant literature we describe in the following sections. 

\paragraph{Proximal Methods for Almost Sure Constraints.}
Problems of similar formulation to~\eqref{eq:prob-form} have been addressed in prior literature under the assumption of an efficient projection oracle over $\X$. Works such as ~\cite{patrascu2017nonasymptotic, xu2018primal, fercoq2019almost} solve these problems via stochastic proximal methods and attain a complexity of $\mathcal{O}(\epsilon^{-2})$ sfo calls, which is known to be optimal even for unconstrained stochastic optimization. In particular,  \citet{patrascu2017nonasymptotic} study convex constrained optimization, where the constraints are expressed as a (possibly infinite) intersection of stochastic, closed, convex and projectable sets $X_{\xi}$. Problem~\eqref{eq:prob-form} can be partly cast to this template, with $A(\xi)X \in b(\xi)$ being the homologues of $X_{\xi}$. However, our additional set $\X$ does not allow for efficient projections, making this framework inapplicable. 

\citet{xu2018primal} solves a convex constrained optimization problem over a convex set $\mathcal{X}$, subject to a large number of convex functional constraints $f_j, \, j = 1 \ldots M$. The functions $f_j$ are sampled uniformly at random during optimization, which corresponds to a finitely sampled instance of problem~\eqref{eq:prob-form} for affine $f_j$. However, we meet again with the limiting condition that projections onto $\mathcal{X}$ are computationally expensive in our setting.

Finally, \citet{fercoq2019almost} study convex problems subject to a possibly infinite number of almost sure linear inclusion constraints, a template which closely resembles ours. The limitation, however, lies in their inclusion of a proximal-friendly component in the objective used to perform stochastic proximal gradient steps. This assumption does not hold for our problem formulation.

\paragraph{Conditional Gradient Methods for Constrained Optimization.}
CGM was first proposed in the seminal work of~\citet{originalfw} and its academic interest has witnessed a resurgence in the past decade. 
The advantage of CGMs lies in the low per-iteration cost of the lmo, alongside their ability to produce sparse solutions. 
In comparison to projection-based approaches, the lmo is cheaper to compute for several important domains, amongst which the spectrahedron, polytopes emerging from combinatorial optimization, and $\ell_p$ norm-induced balls~\cite{garber2016projection}. 
Consequently, CG-type methods have been studied under varying assumptions in~\cite{hazan2008sparse, clarkson2010coresets, hazan2012projection,revisitingfw, lan2013complexity, balasubramanian2018zeroth}, and have been incorporated as cheaper subsolvers into algorithms which originally relied on projection oracles~\cite{lan2016conditional, liu2019nonergodic}.

CGMs have been further extended to the setting of convex composite minimization via the Augmented Lagrangian framework in~\cite{gidel2018frank, silveti2019generalized, yurtsever2019conditional}. Most relevant to our work, CGM-based quadratic penalty methods have been studied for convex problems with constraints of the form $Ax - b \in \mathcalorigin{K}$, where $\mathcalorigin{K}$ is a closed, convex set ~\cite{yurtsever2018conditional, locatello2019stochastic}. We compare our methods against the latter two in Section~\ref{sec:comparison}.

\paragraph{Variance Reduction.}
Stochastic variance reduction (VR) methods have gained popularity in recent years following their initial study by~\cite{roux2012stochastic, johnson2013accelerating,mahdavi2013mixed}. The VR technique relies on averaging schemes to reduce the variance inherent to stochastic gradients, with several different flavors having emerged in the past decade: SAG~\cite{schmidt2017minimizing}, SVRG~\cite{johnson2013accelerating}, SAGA~\cite{defazio2014saga}, SVRRG++~\cite{allen2016improved}, SARAH~\cite{nguyen2017sarah} and SPIDER~\cite{fang2018spider}. Such methods outperform the classical SGD under the finite sum model, a fact which led to their widespread use in large-scale applications and their further inclusion into other stochastic optimization algorithms (see for example~\cite{xiao2014proximal, svrf}).

Relevant to our setting, VR has been studied in the context of CGMs for convex minimization by \cite{mokhtari, svrf, locatello2019stochastic, spiderfw, zhang2019one}. 
The sfo complexity of these methods varies depending on the VR scheme, with the best guarantee being of order $\mathcal{O}(\epsilon^{-2})$~\cite{zhang2019one, spiderfw}. For a thorough comparison of the complexities, we refer the reader to Section~6 of \cite{spiderfw}.

\section{Preliminaries}
\label{sec:prelim}

\paragraph{Notation.}
We use $\| \cdot \|$ to express the Euclidean norm and $\dotprod{\cdot}{\cdot}$ to denote the corresponding inner product. The distance between a point $x$ and a set $\X$ is defined as $\dist(x, \X) \defeq \inf_{y \in \X} \norm{y - x}$. The indicator function of a set $\X$ is given by $\delta_{\X}(x) = 0, \text{ if } x\in \mathcal{X}$, and $\delta_{\mathcal{X}}(x)=+\infty$ otherwise. We denote by $\diam \defeq \max_{(x, y) \in \X \times \X} \norm{x - y}$ the diameter of a compact set~$\X$.

For the probabilistic setting, we denote by $\xi$ an element of our sample space and by $P(\xi)$ its probability measure. Unless stated otherwise, expectations will be taken with respect to $\xi$. We use $[n]$ to denote $\{1, 2, \dotsc n\}$.

Given a function $f: \R^d \to \R$ and $L > 0$, we say that $f$ is $L$-smooth if $\nabla f$ is Lipschitz continuous, which is defined as $\|\nabla f(x) - \nabla f(y)\| \leq L \|{x - y}\|,  \forall x, y \in \R^d$.

Following the same setup as in~\cite{fercoq2019almost}, the space of random variables used in this work is
\begin{equation*}
\H = \left\{ y(\xi)_{\xi} \in \reals{m} \mid \; \xi \in \mathbb{R}^n, \: \Exp{\normsqr{y(\xi)_{\xi}}}{} < + \infty \right\},
\end{equation*}
where the associated scalar product is given by \mbox{$\dotprod{x}{z} \defeq \Exp{x(\xi)^Tz(\xi)}{} = \int x(\xi)^Tz(\xi) dP(\xi)$}.

\paragraph{Smoothing.}

\citet{nesterov2005smooth} proposes a technique for obtaining smooth approximations parametrized by $\beta$, of a nonsmooth and convex function $g$. The resulting smoothed approximations take the following form:
\begin{equation*}
g_{\beta}(x) = \max_y \dotprod{y}{x} - g^*(y) - \frac{\beta}{2} \normsqr{y},
\end{equation*}
where $g^*(y) = \sup_{z}\dotprod{z}{y} - g(z)$ is the Fenchel conjugate of~$g$. Note that $g_{\beta}$ is convex and $\frac{1}{\beta}$-smooth.
The present work focuses on the case when $g(\cdot, \xi) = \delta_{b(\xi)}(\cdot)$. Smoothing the indicator function is studied in the context of proximal methods by~\citet{tran2018smooth, fercoq2019almost} and for deterministic CGM by~\citet{yurtsever2018conditional}. Of particular note is that when $g(x) = \delta_{\X}(x)$, the smoothed function becomes $g_{\beta}(x) = \frac{1}{2\beta} \dist(x, \X)^2$.

\paragraph{Optimality Conditions.}
We denote by $\xopt$ a solution to problem~\eqref{eq:prob-form} and say that $x$ is an $\epsilon$-solution for~\eqref{eq:prob-form} if it satisfies
\begin{equation}
\mathbb{E}\left[ \lvert f(x, \xi) - \fopt \rvert \right] \leq \epsilon, ~~ \sqrt{\mathbb{E} \left[ \dist(\Astoc x, \bstoc)^2 \right] } \leq \epsilon.\label{eq: optim}
\end{equation}
\paragraph{Oracles.}
Our complexity results are given relative to the following oracles:
\begin{itemize}
\item \textbf{Stochastic first order oracle (sfo):}
For a stochastic function $\mathbb{E}\left[ f(\cdot, \xi)\right]$ with $\xi \sim P$, the sfo returns a pair $(f(x, \xi), \nabla f(x, \xi))$ where $\xi$ is an i.i.d. sample from $P$~\cite{nemirovsky1983problem}.

\item \textbf{Incremental first order oracle (ifo):}
For finite-sum problems, the ifo takes an index $i \in [n]$ and returns a pair $(f_i(x), \nabla f_i(x))$.

\item \textbf{Linear minimization oracle (lmo):} The linear minimization oracle of set $\X$ is given by $\text{lmo}_{\X}(y) = \argmin_{x \in \X} \dotprod{x}{y}$  and  is assumed to be efficient to compute throughout this paper. This is the main projection-free oracle model for CGM-type methods.
\end{itemize}

\section{Algorithms \& Convergence}
\label{sec:algs-and-conv}

We now describe our proposed methods for solving~\eqref{eq:prob-form}, H-1SFW and H-SPIDER-FW, and provide their theoretical convergence guarantees. 

\subsection{Challenges and High-Level Ideas}
\label{subsec:high-level.ideas}
Problem~\eqref{eq:prob-form} can be rewritten equivalently as:
\begin{equation}
 \label{eq:obj-2}
        \min_{x \in \X}  F(x) \defeq \mathbb{E}  \left[f(x, \xi) +  \delta_{b(\xi)}(A(\xi)x)\right].
\end{equation}
Note that, in this form, our objective is non-smooth due to the indicator function. In order to leverage the conditional gradient framework, we \emph{smooth} $\delta_{b(\xi)}(A(\xi)x)$ through the technique described in Section~\ref{sec:prelim}, thus obtaining a surrogate objective $F_{\beta}$. For notational simplicity, we refer to the smoothed stochastic indicator as:
\begin{align}
 &g_{\beta}(\Astoc x) = \frac{1}{2\beta} \dist(\Astoc x, \bstoc) ^2.\label{eq: stoc_grad}
\end{align}

The minimization problem in terms of the smoothed objective thus becomes:
\begin{equation}
 \label{eq:approxim}
    \min_{x \in \X}  F_{\beta}(x) \defeq \mathbb{E} \left[  f(x, \xi) + g_{\beta} (\Astoc x)\right],
\end{equation}
with $\underset{\beta \rightarrow 0}{\lim} F_{\beta}(x) = F(x)$. A natural idea is to optimize smooth approximations $F_{\beta}$ which are progressively more accurate representations of $F$. To this end, we apply \emph{conditional gradient} steps in conjunction with decreasing the smoothness parameter $\beta$, practically emulating a homotopy transformation. As the iterations unfold our algorithms in fact approach the optimum of the original objective $F(x)$, as stated theoretically in  Sections~\ref{subsec:conv-mokhtari}~and~\ref{subsec:conv-spiderfw}.

However, the aforementioned idea faces a technical challenge: decreasing the smoothing parameter $\beta$ impacts the variance of the stochastic gradients $\nabla_x g_{\beta}(A(\xi)x)$, which increases proportionally. This issue has previously been signaled in the work of \cite{fercoq2019almost}, where the authors address a similar setting using stochastic proximal gradient steps. Here, the problem is further aggravated by the use of lmo calls over $\mathcal{X}$, as it is well-known that CGMs are sensitive to non-vanishing gradient noise \cite{mokhtari}.

Our solution is to simply perform VR on the stochastic gradients and theoretically establish a rate for $\beta \rightarrow 0$ in order to counteract the exploding variance. Precisely, we show how two different VR schemes can be successfully used within the homotopy framework:
\begin{itemize}
    \item H-1SFW uses one stochastic sample to update a gradient estimator at every iteration, following the technique introduced in~\cite{mokhtari}. Depending on computational resources, the single-sample model can be extended to a fixed batch size with the same convergence guarantees. 
    \item H-SPIDER-FW uses stochastic minibatches of increasing size to compute the gradient estimator, using the technique proposed in~\cite{fang2018spider}.
\end{itemize} 

The theoretical results characterizing our algorithms are presented in sections~ref{sec:h1sfw}~and~\ref{sec:hspfw}. First, we state the rate at which the $\beta$-dependent stochastic gradient noise vanishes under each VR scheme in lemmas~\ref{lem:mokhtari} and ~\ref{lem:spiderfw-finite-sum-var}. The main convergence theorems~\ref{thm:mokhtari}~and~\ref{thm:spiderfw} then describe the performance of our algorithms in terms of the quantity $\mathbb{E}\left[ S_{\beta_k}(x_k, \xi)\right] \defeq \mathbb{E}\left[ F_{\beta_k}(x_k, \xi) - \fopt\right]$, called the \emph{smoothed gap}. Finally, in corollaries~\ref{corollary:mokhtari-residual-and-feasibility}~and~\ref{cor:spiderfw} we translate the aforementioned results into guarantees over the objective residual and constraint feasibility. All proofs are deferred to the appendix due to lack of space.

\subsection{Technical Assumptions}
\begin{assumption}
The stochastic functions $f(\cdot, \xi)$ are convex and $L_f$-smooth. This further implies that $f(x)$ is $L_f$-smooth. 
\end{assumption} 

\begin{assumption}
\label{ass:fstoc_finite_var}
The stochastic gradients $\nabla f(x, \xi)$ are unbiased and have a uniform variance bound $\sigma_f^2$. Formally,
\vspace{-5mm}
\begin{align}
        &\mathbb{E} \left[ \nabla f(x, \xi) \right]  = \nabla f(x) \nn \\
        & \mathbb{E} \left[ \normsqr{\nabla f(x, \xi) - \nabla f(x)} \right] \leq \sigma_f^2 < +\infty. \label{eq: var_bound2}
\end{align}

\end{assumption}

\begin{assumption}
The domain $\mathcal{X}$ is convex and compact, with diameter $\diam$.
\end{assumption}

\begin{assumption}
Slater's condition holds for problem~\eqref{eq:obj-2}. Specifically, letting $G: \H \rightarrow \R \cup \{\infty\}, \; G(Ax) \defeq \mathbb{E}\left[ \delta_{\bstoc}(\Astoc x)\right]$, with the linear operator $A: \R^d \rightarrow \H$ defined as $(Ax)(\xi) \defeq A(\xi)x,\; \forall x$, we require that
\begin{equation*}
    0 \in \mathrm{sri}\left( \mathrm{dom}(G) - A\,\mathrm{dom}(f)\right),
\end{equation*}
where $\mathrm{sri}$ is the strong relative interior of the set~\cite{bauschke2011convex}.
\end{assumption}

\begin{assumption}
The spectral norm of the stochastic linear operator $A(\xi)$ is uniformly bounded by a constant $L_A$:
\vspace{-1mm}
\begin{equation}
    L_A \defeq \sup_{\xi} \norm{A(\xi)}^2 < +\infty. \nn
\end{equation}
This assumption is also made in ~\cite{fercoq2019almost}.
\end{assumption}

\subsection{H(omotopy)-1SFW}
\label{sec:h1sfw}
We now describe our first algorithm which relies on the VR scheme proposed in \cite{mokhtari}, and whose advantage lies in a simple update rule and single-loop structure. 

\subsubsection{Gradient Estimator Model}
\label{subsub:mokhtari-grad-estim}
We denote the gradient estimator by $d_k$, and remark that it is biased with respect to the true gradient $\nabla F_{\beta}(x_k)$ and exhibits a vanishing variance. This scheme achieves VR while conveniently considering only one stochastic constraint at a time.
The estimator update rule is given by
\begin{equation*}
    d_k = (1 - \rho_k) d_{k-1} + \rho_k \nabla F_{\beta_k}(x_k, \xi_k),
\end{equation*}
where $\nabla F_{\beta_k}(x_k, \xi_k) = \nabla f(x_k, \xi_k) + \nabla g_{\beta_k}(A(\xi_k)X_k)$, and $\rho_k$ is a decaying convex combination parameter. The proposed method is provided via pseudocode in Algorithm~\ref{alg:homotopy-mokhtari}.

\begin{algorithm}[tb]
\begin{spacing}{1.3}
\begin{algorithmic}
\STATE \textbf{Input:} $x_1 \in \X, \beta_0 > 0, P(\xi)$
\FOR{$k = 1, 2, \dots, $}
\STATE Set $\rho_k$, $\beta_k$ and $\gamma_k$; sample $\xi_k \sim P(\xi)$
\STATE $d_k = (1-\rho_k)d_{k-1}+ \rho_k \nabla_x F_{\beta_k}(x_k, \xi_k)$
\STATE $w_k = \text{lmo}_{\X}(d_k)$
\STATE $x_{k+1} = x_k + \gamma_k (w_k - x_k)$.
\ENDFOR
\end{algorithmic}
\end{spacing}
\caption{H-1SFW}
\label{alg:homotopy-mokhtari}
\end{algorithm}

\subsubsection{Convergence Results}
\label{subsec:conv-mokhtari}

Before stating the results, we remark that Lemma~\ref{lem:mokhtari} is the counterpart of Lemma~1~in~\cite{mokhtari} and its proof follows a similar route, up to bounding $\beta$-dependent quantities. It is worth noting that in our case, handling the stochastic linear inclusion constraints results in a rate surcharge factor of $\bigO{k^{1/3}}$. 

\vspace{3mm}
\begin{lemma}
\label{lem:mokhtari}
Let $\rho_k = \frac{3}{(k+5)^{2/3}}, ~~ \gamma_k = \frac{2}{k + 1}, ~~ \beta_k = \frac{\beta_0}{(k+1)^{1/6}}, \, \beta_0 > 0$ in Algorithm~\ref{alg:homotopy-mokhtari}. Then, for all $k$,
\begin{equation*}
\mathbb{E} \left[ \| \nabla F_{\beta_k}(x_k) - d_k \|^2 \right] \leq \frac{C_1}{(k + 5)^{1/3}},
\end{equation*}
where \scalebox{0.92}{$\!\begin{aligned}[t]
   C_1=\max \Bigg\{ 6&^{1/3}\|\nabla F_{\beta_0}(x_0)-d_0\|^2, \\ 
   &2\left[  18\sigma_f^2 +  112L_f^2\diam^2 +  \frac{522L_A^2\diam^2}{\beta_0^2} \right]\Bigg\}
  \end{aligned}$}  
\end{lemma}

\vspace{3mm}

\begin{theorem}
\label{thm:mokhtari}
Consider Algorithm~\ref{alg:homotopy-mokhtari} with parameters $\rho_k = \frac{3}{(k+5)^{2/3}}, ~~ \gamma_k = \frac{2}{k + 1}, ~~ \beta_k = \frac{\beta_0}{(k+1)^{1/6}}, \, \beta_0 > 0$ (identical to Lemma~\ref{lem:mokhtari}). Then, for all $k$,
\begin{equation*}
\mathbb{E}\left[ S_{\beta_{k}}(x_{k+1}) \right] \leq \frac{C_2}{k^{1/6}},
\end{equation*}
where \scalebox{0.92}{$C_2 = \max \left\{ S_0(x_1), \;b=2 \diam\sqrt{C_1} + 2\diam^2\left( L_f + \frac{L_A}{\beta_0} \right)\right\}$} and $C_1$ is defined in Lemma~\ref{lem:mokhtari}.
\end{theorem}

\vspace{3mm}

\begin{corollary}
\label{corollary:mokhtari-residual-and-feasibility}
The expected convergence in terms of objective suboptimality and feasibility of Algorithm~\ref{alg:homotopy-mokhtari} is, respectively, 
\begin{align*}
    &\mathbb{E}\left[ \lVert f(x_k, \xi) - \fopt \rVert \right] &&\hspace{-7mm}\in \bigO{k^{-1/6}} \\
    &\sqrt{\mathbb{E} \left[ \dist(\Astoc x_k, \bstoc)^2 \right] } &&\hspace{-7mm}\in \bigO{k^{-1/6}}.
\end{align*} 
Consequently, the oracle complexity is $\sfo \in \bigO{\epsilon^{-6}}$ and $\lmo \in \bigO{\epsilon^{-6}}$.
\end{corollary}

\subsection{H(omotopy)-SPIDER-FW}
\label{sec:hspfw}
Our second algorithm presents a more complex VR scheme, which improves on the complexity of H-1SFW. The method relies on the SPIDER estimator originally proposed under the framework of Normalized Gradient Descent in~\citep{fang2018spider} and further studied for CGMs in ~\cite{spiderfw}. Different from Section~\ref{subsec:conv-mokhtari}, the results that follow distinguish two scenarios: the first is customary to VR methods such as SVRG~\cite{johnson2013accelerating} or SARAH~\cite{nguyen2017sarah} and assumes a finite-sum form of $f$; the second, different from most other VR schemes, caters to objectives of the form $f(x) = \mathbb{E}\left[ f(x, \xi)\right]$ where $\xi \sim P(\xi)$, and can handle a potentially infinite number of stochastic functions of~\eqref{eq:prob-form}.

\subsubsection{Gradient Estimator Model}
We denote the SPIDER gradient estimator by $v_{t,k}$, and remark that it is also biased relative to $\nabla F_{\beta_k}(x_k)$ and exhibits a vanishing variance. This scheme achieves VR through the use of increasing-size minibatches. The estimator update rule is given by
\begin{align}
\label{eq:spider_vr}
 v_{t, k} = v_{t, k-1} &- \tilde{\nabla} F_{\beta_{t, k-1}}(x_{t, k-1}, \xi_{\mathcal{S}_{t,k}}) \nn \\
                     &\hspace{8mm}+ \tilde{\nabla} F_{\beta_{t, k}}(x_{t, k}, \xi_{\mathcal{S}_{t,k}}),
\end{align}
where $\tilde{\nabla} F_{\beta_{t, k}}(x_{t, k}, \xi_{\mathcal{S}_{t,k}}) = \tilde{\nabla} f(x_k, \xi_{\mathcal{S}_{t,k}}) + \tilde{\nabla} g_{\beta_{t, k}}(A(\xi_{\mathcal{S}_{t,k}})x_{t,k})$ defines the averaged gradient over a minibatch of size $\lvert \mathcal{S}_{t, k} \rvert$. 

The double indexing used in~\eqref{eq:spider_vr} hints at the double-loop structure of the algorithm, a format similar to most VR-based methods. The method is structured similarly to SPIDER-FW from~\cite{spiderfw}, and proceeds in two steps: the outer loop computes an accurate gradient estimator and sets the batch size for the inner iterations. The inner-loop then iteratively `refreshes' this gradient according to~\eqref{eq:spider_vr} and performs homotopy steps on $\beta$ using a theoretically-determined schedule. The proposed method is provided via pseudocode in Algorithm~\ref{alg:homotopy-spider-fw}.

\begin{algorithm}[tb]
\begin{spacing}{1.3}
\begin{algorithmic}
\STATE \textbf{Input:} $\bar{x}_1 \in \X, \beta_0 > 0, P(\xi)$
\FOR{$t = 1, 2, \dots, T$}
	\STATE $x_{t, 1} = \bar{x}_t$
	\STATE Compute $\gamma_{t, 1}, \beta_{t, 1},  K_t$; sample $\xi_{\mathcalorigin{Q}_t} \overset{\text{\tiny i.i.d}}{\sim} P(\xi)$
	\STATE $v_{t, 1} = \tilde{\nabla} F_{\beta_{t, 1}}(x_{t, 1}, \xi_{\mathcalorigin{Q}_t})$
	\STATE  $w_{t, 1} \in \text{lmo}_{\X}(v_{t, 1})$
	\STATE  $x_{t, 2} = x_{t, 1} + \gamma_{t, 1}(w_{t, 1} - x_{t, 1})$
	\FOR {$k = 2, \dots, K_t$}
		\STATE Compute $\gamma_{t, k}, \beta_{t, k}$;  sample $\xi_{\mathcalorigin{S}_{t, k}} \overset{\text{\tiny i.i.d}}{\sim} P(\xi)$ 
		\STATE $\!\begin{aligned}[t] v_{t, k} = v_{t, k - 1} &- \tilde{\nabla} F_{\beta_{t, k - 1}}(x_{t, k - 1}, \xi_{\mathcalorigin{S}_{t, k}}) \\
		                                                                             &\hspace{7mm}+ \tilde{\nabla} F_{\beta_{t, k}}(x_{t, k }, \xi_{\mathcalorigin{S}_{t, k}})\end{aligned}$
		\STATE $w_{t, k} \in \text{lmo}_{\X}(v_{t, k})$
		\STATE $x_{t, k + 1} = x_{t, k} + \gamma_{t, k}(w_{t, k} - x_{t, k})$
	\ENDFOR
	\STATE Set $\bar{x}_{t+1} = x_{t, K_t + 1}$
\ENDFOR
\end{algorithmic}
\end{spacing}
\caption{H-SPIDER-FW}
\label{alg:homotopy-spider-fw}
\end{algorithm}

\subsubsection{Convergence Results}
\label{subsec:conv-spiderfw}
Again, we remark that Lemma~\ref{lem:spiderfw-finite-sum-var} is the counterpart of Lemma 4, Appendix C  in~\cite{spiderfw}. However in this case, our proof takes a different, more tedious route, as the latter result does not accommodate homotopy steps. In comparison, the bound we obtain depends linearly on the total iteration count, whereas the lemma of~\cite{spiderfw} depends only on the outer loop counter $K_t$.

\begin{lemma}[Estimator variance for finite-sum problems]
\label{lem:spiderfw-finite-sum-var}
Consider Algorithm~\ref{alg:homotopy-spider-fw}, and let $\xi$ be finitely sampled from set $[n]$, $\xi_{\mathcalorigin{Q}_t} = [n]$ and $\xi_{\mathcalorigin{S}_{t, k}}, \text{ such that } |\mathcalorigin{S}_{t, k}| = K_t= 2^{t-1}$. Also, let $\gamma_{t, k} = \frac{2}{K_t + k}, ~\beta_{t, k} = \frac{\beta_0}{\sqrt{K_t + k}}, \, \beta_0 > 0$. Then, for a fixed $t$ and for all $k \leq K_t$,
    \begin{equation*}
        \Exp{\normsqr{\gradF{\beta_{t, k}}{x_{t, k}} - v_{t, k}}}{} \leq \frac{C_1}{K_t + k}, 
    \end{equation*}
where $C_1 = 2\diam^2 \left(8L_f^2 + \frac{98L_A^2}{\beta_0^2}\right)$.
\end{lemma}

\vspace{3mm}
\begin{lemma}[Estimator variance for general expectation problems]
\label{lem:spiderfw-expectation-var}
Consider Algorithm~\ref{alg:homotopy-spider-fw} and let $\xi \sim P(\xi)$ and $\xi_{\mathcalorigin{Q}_t} \text{ such that } \lvert  \mathcalorigin{Q}_t \rvert = \lceil \frac{2K_t}{\beta_{t, 1}^2}\rceil$. Also, let $\xi_{\mathcalorigin{S}_{t, k}}, \text{ such that } |\mathcalorigin{S}_{t, k}| = K_t= 2^{t-1}$, $\gamma_{t, k} = \frac{2}{K_t + k}, ~\beta_{t, k} = \frac{\beta_0}{\sqrt{K_t + k}}, \, \beta_0 > 0$. Then, for a fixed $t$ and for all $k \leq K_t$,
    \begin{equation*}
        \Exp{\normsqr{\gradF{\beta_{t, k}}{x_{t, k}} - v_{t, k}}}{} \leq \frac{C_2}{K_t + k}, 
    \end{equation*}
where $C_2 =  16L_f^2\diam^2 + 2L_A^2\diam^2 \left( \frac{98}{\beta_0^2} + 1\right) + 2\beta_0^2\sigma_f^2$.
\end{lemma}

\vspace{3mm}

\def\EE{{\mathbb{E}}}
\begin{theorem} 
\label{thm:spiderfw}
Consider Algorithm~\ref{alg:homotopy-spider-fw} with parameters $\gamma_{t, k} = \frac{2}{K_t + k}$, $\beta_{t, k} = \frac{\beta0}{\sqrt{K_t+k}}, \, \beta_0 > 0$, and $\xi_{\mathcalorigin{S}_{t, k}}, \text{ such that } |\mathcalorigin{S}_{t, k}| = K_t= 2^{t-1}$. Then,
\begin{itemize}[itemsep=15pt]
\item For $\xi$ be finitely sampled from set $[n]$, $\xi_{\mathcalorigin{Q}_t} = [n]$ and $\forall t \in \mathbb{N}, \, 1 \leq k \leq 2^{t-1}$,
                        \begin{equation*}
                                \Exp{S_{\beta_{t, k}}(x_{t, k + 1})}{}\leq \frac{C_3}{\sqrt{K_t + k + 1}},
                        \end{equation*}
where $\!\begin{aligned}[t] C_3 = &\max \Bigg\{S_{\beta_{1, 0} }( x_{1, 1}), \\ &\hspace{-10mm}2\diam^2 L_f + 2\diam^2 \sqrt{16L_f^2 + \frac{196L_A^2}{\beta_0^2}} +  \frac{2\diam^2 L_A}{\beta_0}  \Bigg\}; \end{aligned}$

\item For $\xi \sim P(\xi)$, $\xi_{\mathcalorigin{Q}_t} \text{ such that } \lvert  \mathcalorigin{Q}_t \rvert = \lceil \frac{2K_t}{\beta_{t, 1}^2}\rceil$ and \mbox{$\forall t \in \mathbb{N}, \, 1 \leq k \leq 2^{t-1}$},
                        \begin{equation*}
                                \Exp{S_{\beta_{t, k}}(x_{t, k + 1})}{} \leq \frac{C_4}{\sqrt{K_t + k + 1}},
                        \end{equation*}
where $\!\begin{aligned}[t] C_4 = &\max \Bigg\{ S_{\beta_{1,0}}(x_{1, 1}), \; 2\diam^2 L_f +  \frac{2\diam^2 L_A}{\beta_0}  \\ &\hspace{-18mm}  + 2\diam \sqrt{16L_f^2\diam^2 + 2L_A^2\diam^2 \left( \frac{98}{\beta_0^2} + 1\right) + 2\beta_0^2\sigma_f^2}\Bigg\}.\end{aligned}$
\end{itemize}
\end {theorem}

\vspace{3mm}

\begin{corollary}
\label{cor:spiderfw}
The expected convergence in terms of objective suboptimality and feasibility of Algorithm~\ref{alg:homotopy-spider-fw} is, respectively, 
\begin{align*}
    &\mathbb{E}\left[ \lVert f(x_{t, k}) - \fopt \rVert \right] &&\hspace{-5mm}\in \bigO{(K_t + k)^{-1/2}}\\
    &\sqrt{\mathbb{E} \left[ \dist(\Astoc x_{t, k}, \bstoc)^2 \right] } &&\hspace{-5mm}\in \bigO{(K_t + k)^{-1/2}}
\end{align*} 
for both the finite-sum and the general expectation setting, up to constants. Consequently, the oracle complexities are given by $\ifo \in \bigO{n \log_2(\epsilon^-2)+ \epsilon^{-4}}$ and $\lmo \in \bigO{\epsilon^{-2}}$ for the finite-sum setting, and by $\sfo \in \bigO{\epsilon^{-4}}$ and $\lmo \in \bigO{\epsilon^{-2}}$ for the more general expectation setting.
\end{corollary}

\subsection{Discussion}
\label{sec:comparison}
\paragraph{Rate Degradation in the Absence of Projection Oracles.} 
Compared to proximal methods for solving~\eqref{eq:prob-form}, our algorithms require $\mathcal{O}(\epsilon^{-2})$ times more sfo calls to reach an $\epsilon$-solution. This is well-known for CG-based methods: for instance, solving a fully deterministic version of~\eqref{eq:prob-form} using the Augmented Lagrangian framework has a gradient complexity of $\mathcal{O}(\epsilon^{-1})$~\cite{xu2017accelerated}, whereas the best known complexity for CG-based algorithms is $\mathcal{O}(\epsilon^{-2})$~\cite{yurtsever2018conditional}. 

\paragraph{Comparison with SHCGM \cite{locatello2019stochastic}.} The state-of-the-art for solving \eqref{eq:prob-form} is the half-stochastic method SHCGM~\cite{locatello2019stochastic}, in which stochasticity is restricted to the objective function $f$, while the constraints are processed deterministically. This algorithm attains an $\mathcal{O}(\epsilon^{-3})$ sfo complexity and an $\mathcal{O}(\epsilon^{-3})$ lmo complexity, by resorting to the same VR scheme as H-1SFW applied only to $f(x, \xi)$. Since SHCGM handles the constraints deterministically, it does not face the challenge of exploding variance as $\beta \rightarrow 0$.

Our analysis shows that handling the $\beta$-dependence of the gradient noise comes at the price of H-1SFW being $\mathcal{O}(\epsilon^{-3})$ times more expensive in terms of both oracles. In contrast, owing to a more powerful variance-reduction scheme, H-SPIDER-FW attains only an $\mathcal{O}\left(\epsilon \right)$-times worse sfo complexity, while improving by an $\mathcal{O}\left(\epsilon \right)$ factor in terms of the lmo complexity. Given that an lmo call is generally more expensive than that of an sfo, we have in fact \emph{improved} the complexity over the state-of-the-art, while being the first to process linear constraints stochastically. Moreover, we note that the lmo complexity of H-SPIDER-FW is on the same order as its fully deterministic counterpart, the HCGM~\cite{yurtsever2018conditional}.

\paragraph{The Role of VR.} The choice of VR technique dictates the worst-case convergence guarantees of our methods, a fact which is apparent from the discrepancy between the variance bounds of Lemmas~\ref{lem:mokhtari} and \ref{lem:spiderfw-finite-sum-var}-~\ref{lem:spiderfw-expectation-var}, respectively: $\mathcal{O}(k^{-1/3})$ for $d_k$ vs. $\mathcal{O}(k^{-1})$ for $v_{t, k}$. This signals the existence of a trade-off: a more intricate way of handling stochastic penalty-type constraints can ensure the better convergence guarantees of H-SPIDER-FW, while a simpler VR scheme comes at the cost of the rather pessimistic ones of H-1SFW. Fortunately, as shown in the Section~\ref{sec:experiments}, the simple H-1SFW greatly outperforms its worst-case guarantees.

\section{Numerical Experiments}
\label{sec:experiments}

\begin{figure*}[ht!]
\centering
\begin{minipage}[t]{.66\textwidth}
\setlength{\lineskip}{0pt}
  \includegraphics[width=.5\linewidth]{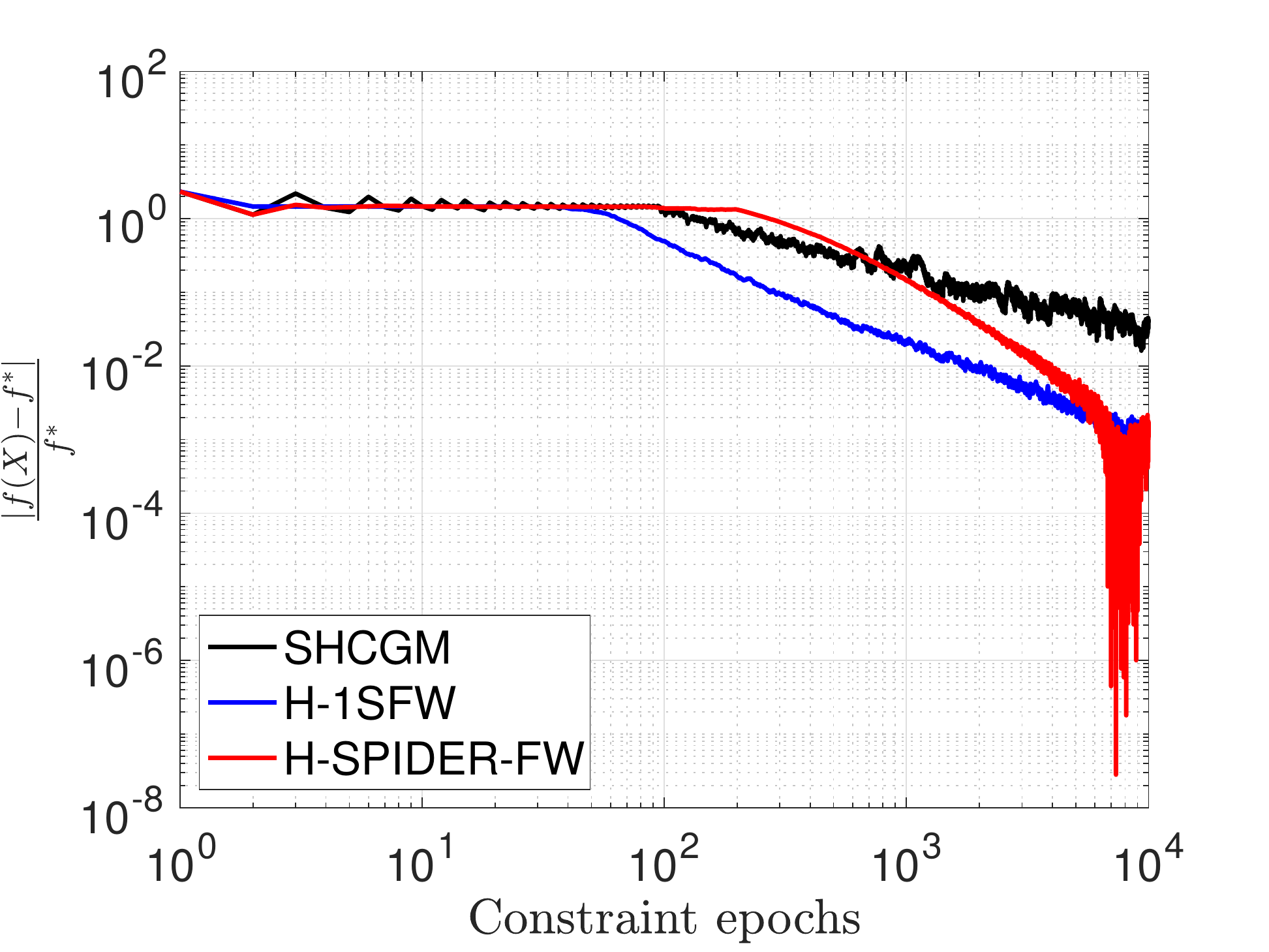}
  \includegraphics[width=.5\linewidth]{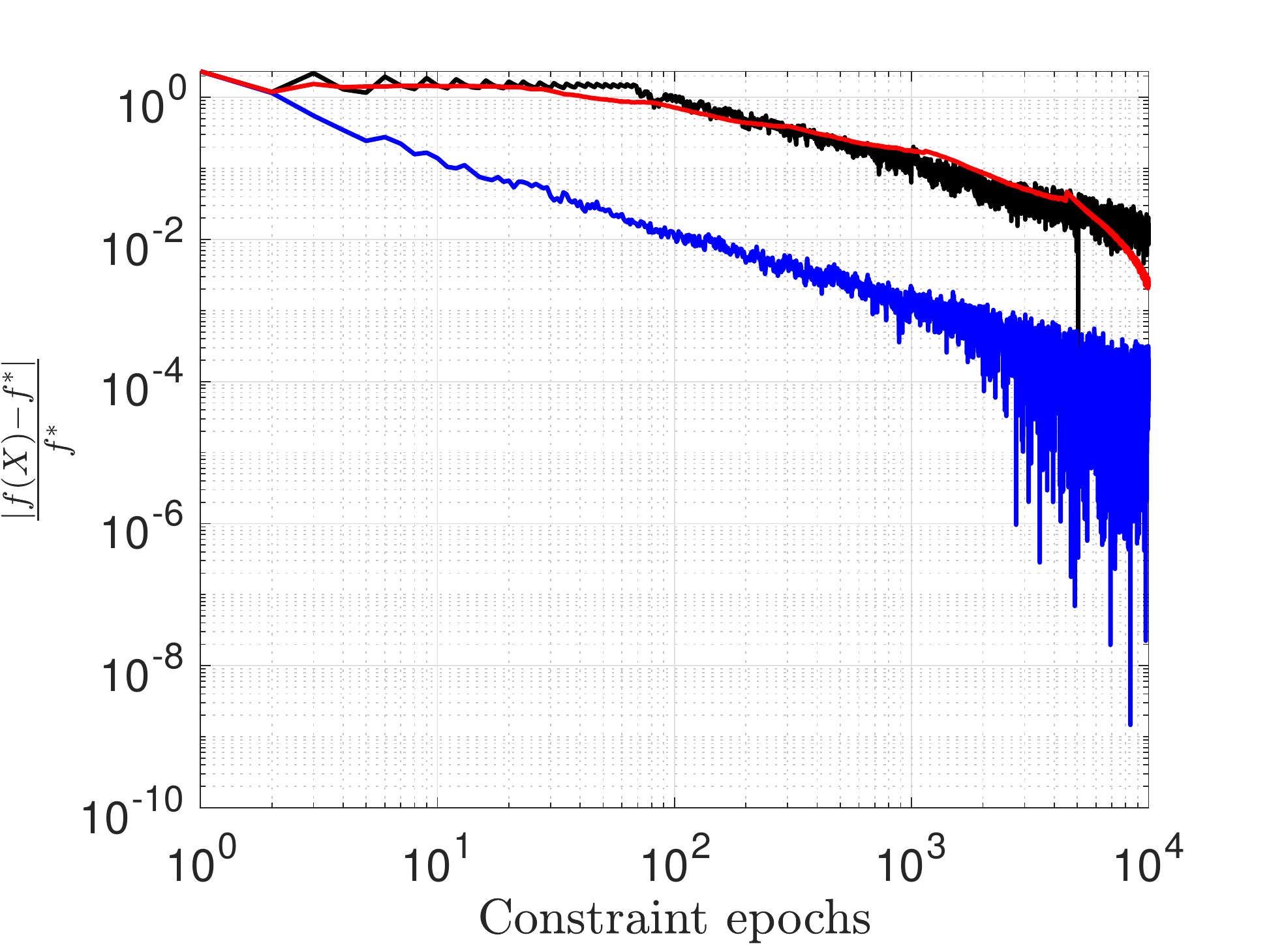}\par
  \includegraphics[width=.5\linewidth]{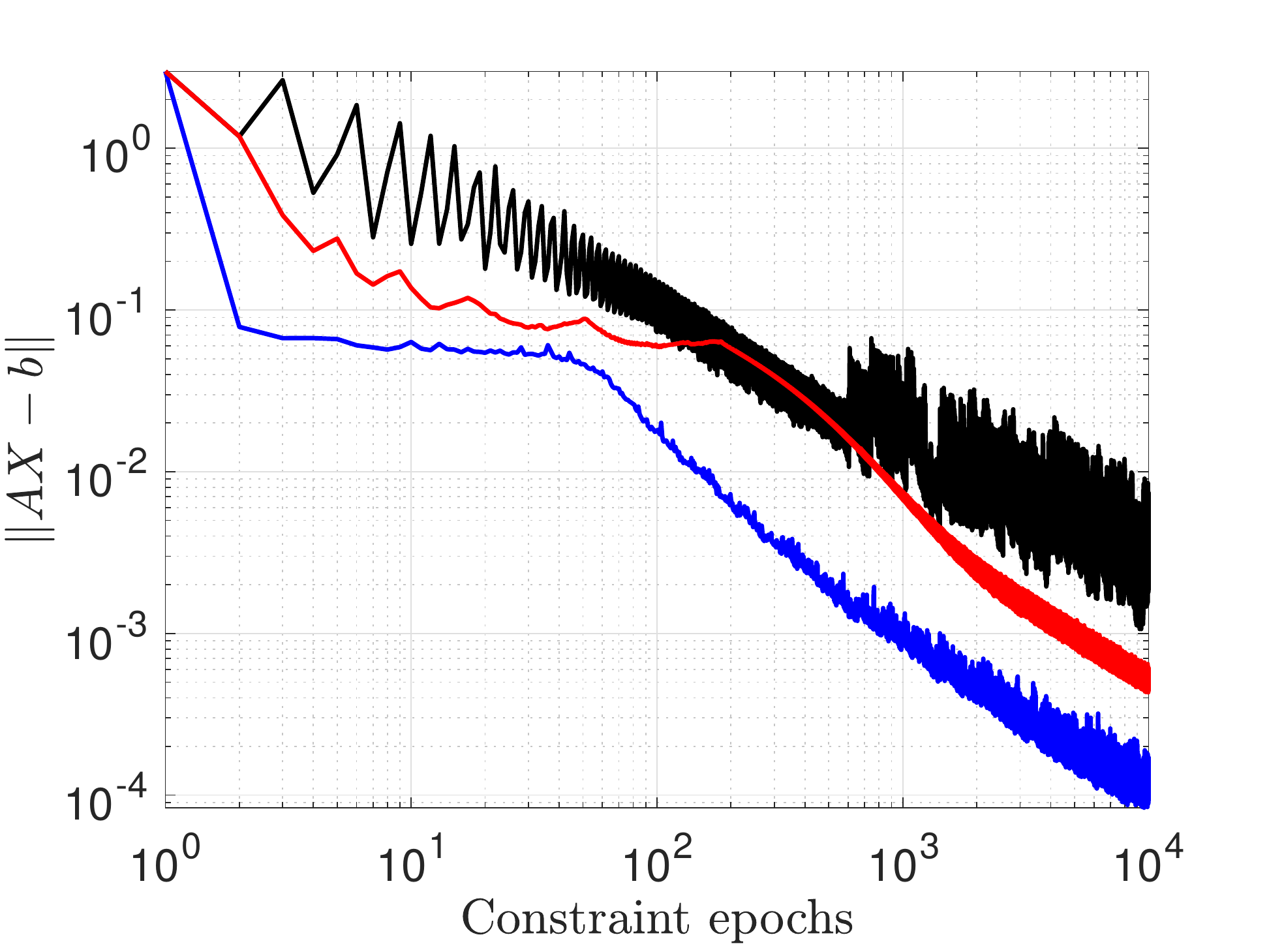}
  \includegraphics[width=.5\linewidth]{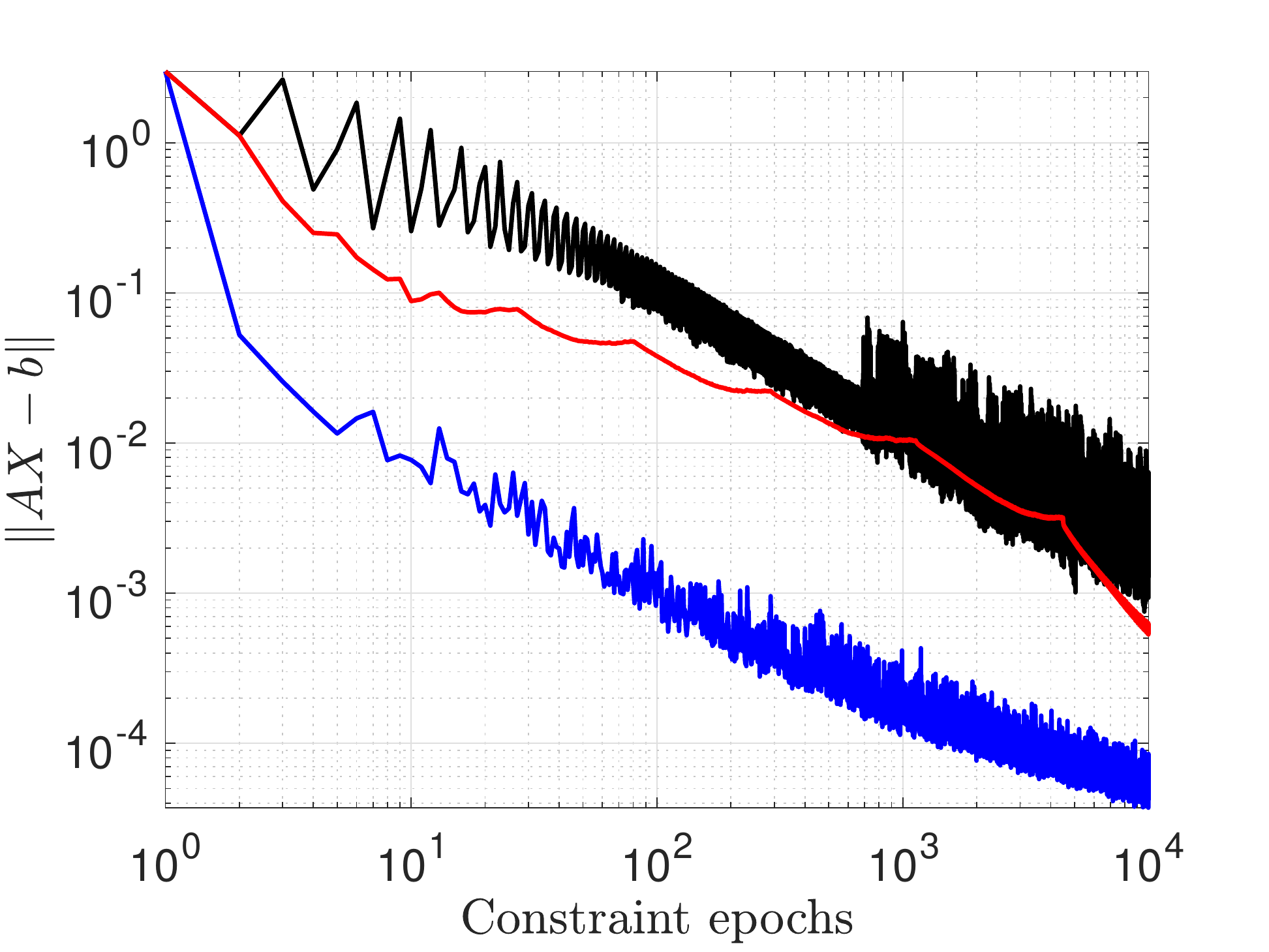}
\caption{Synthetic SDPs, with each column showing the convergence in objective suboptimality (top) and in feasibility (bottom) for a specific problem. The left hand-side column corresponds to a problem with $\texttt{5e2}$ constraints, while the right hand-side one to a problem with $\texttt{5e3}$ constraints.}\label{fig:conv1e4}
\end{minipage}
\hspace{1.7mm}
\begin{minipage}[t]{.319\textwidth}
\setlength{\lineskip}{0pt}
  \includegraphics[width=\linewidth]{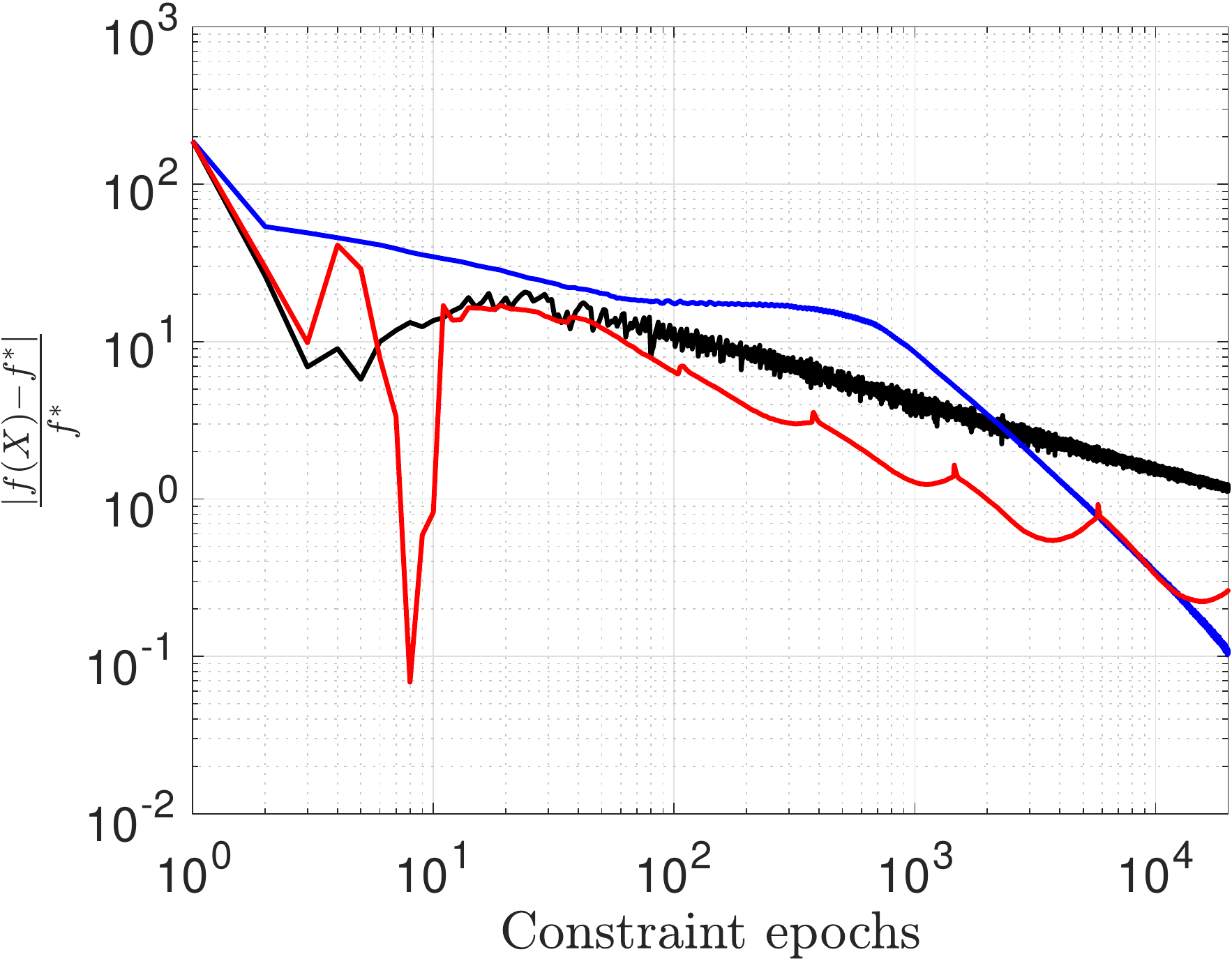}\par
  \includegraphics[width=\linewidth]{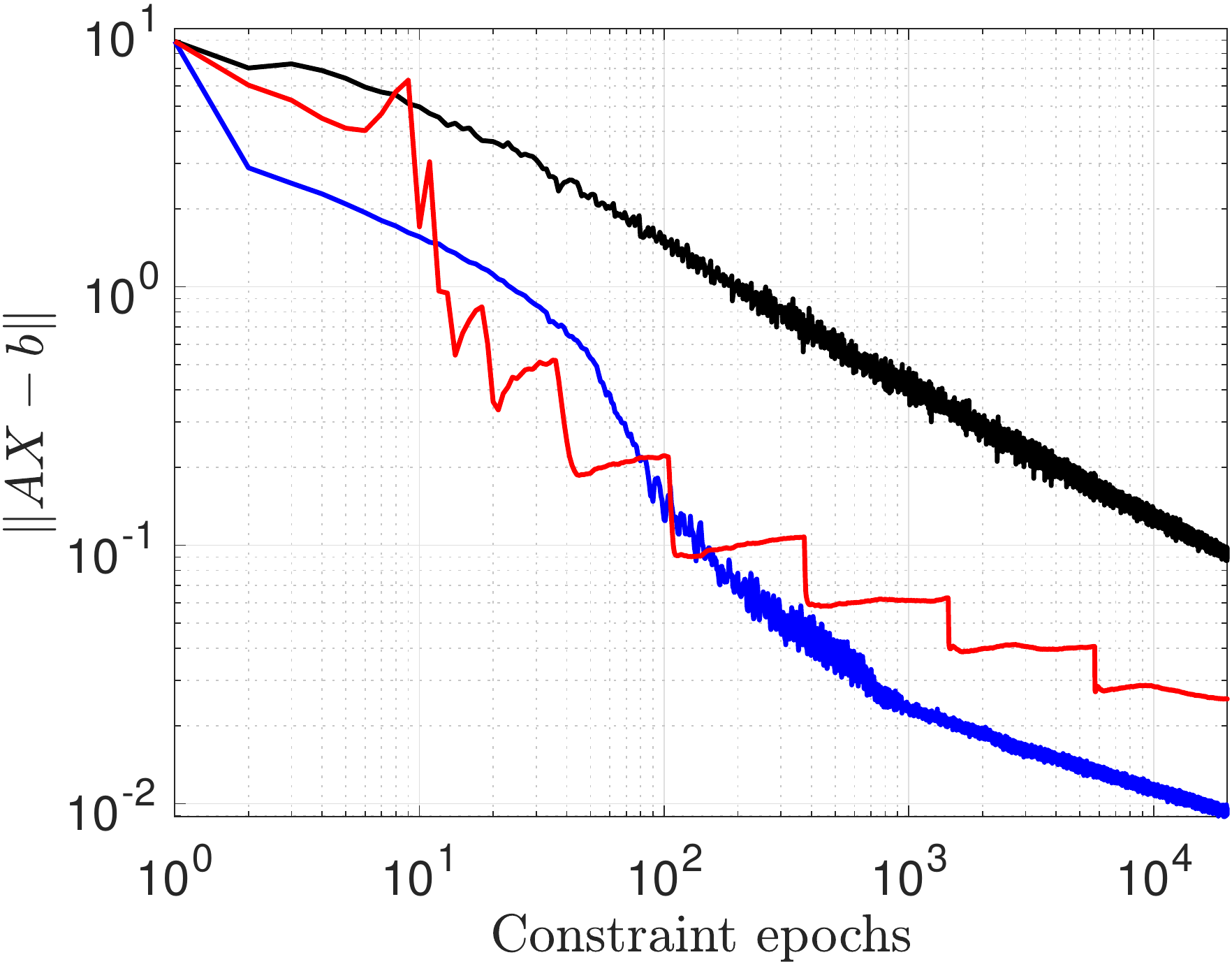}
\caption{The K-means SDP relaxation, with convergence in objective suboptimality (top) and in feasibility (bottom).}\label{fig:mnist_small}
\end{minipage}
\end{figure*}

\begin{figure*}[hb!]
\centering
\begin{minipage}[t]{\textwidth}
\setlength{\lineskip}{0pt}
\includegraphics[width=.33\linewidth]{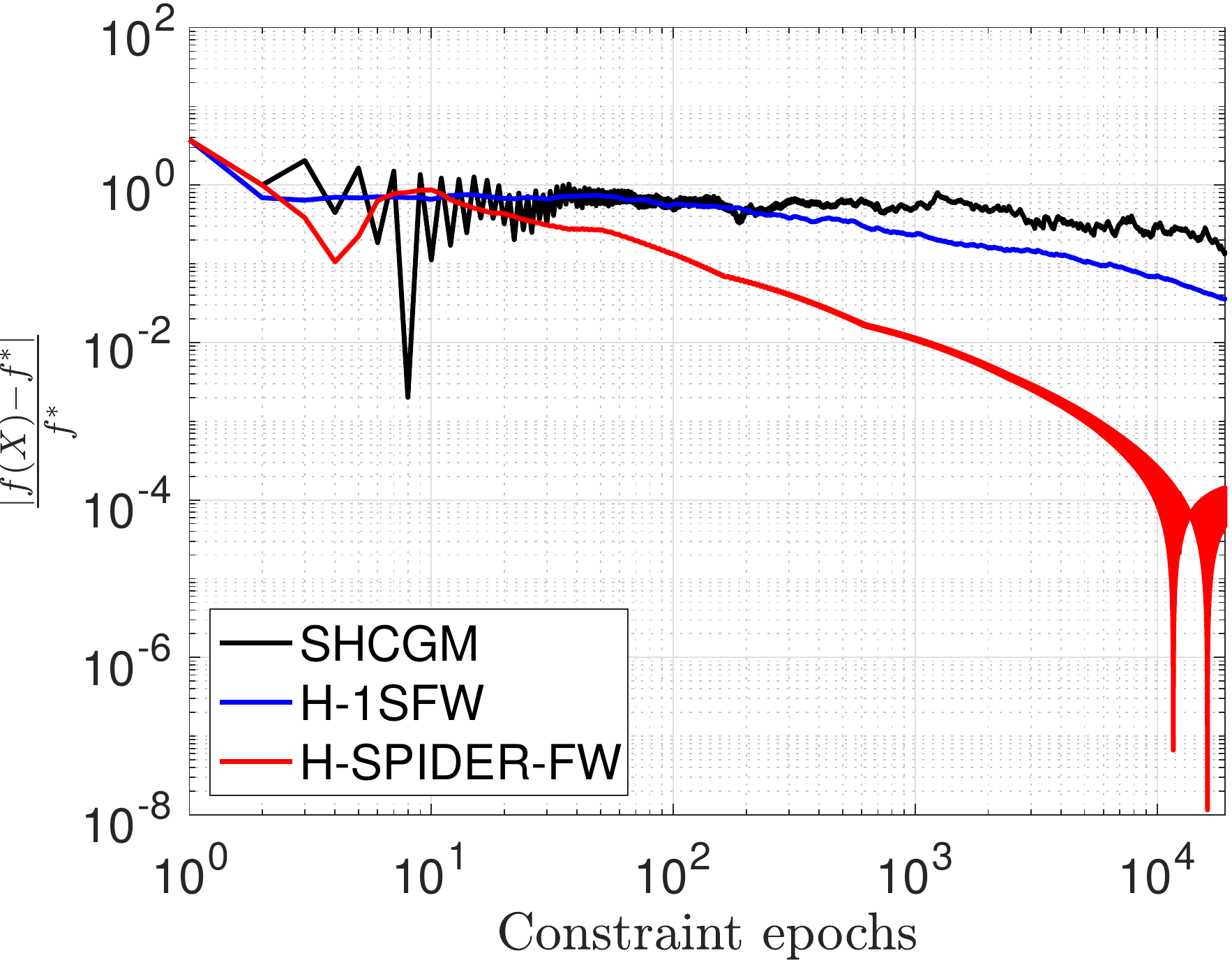}
\includegraphics[width=.33\linewidth]{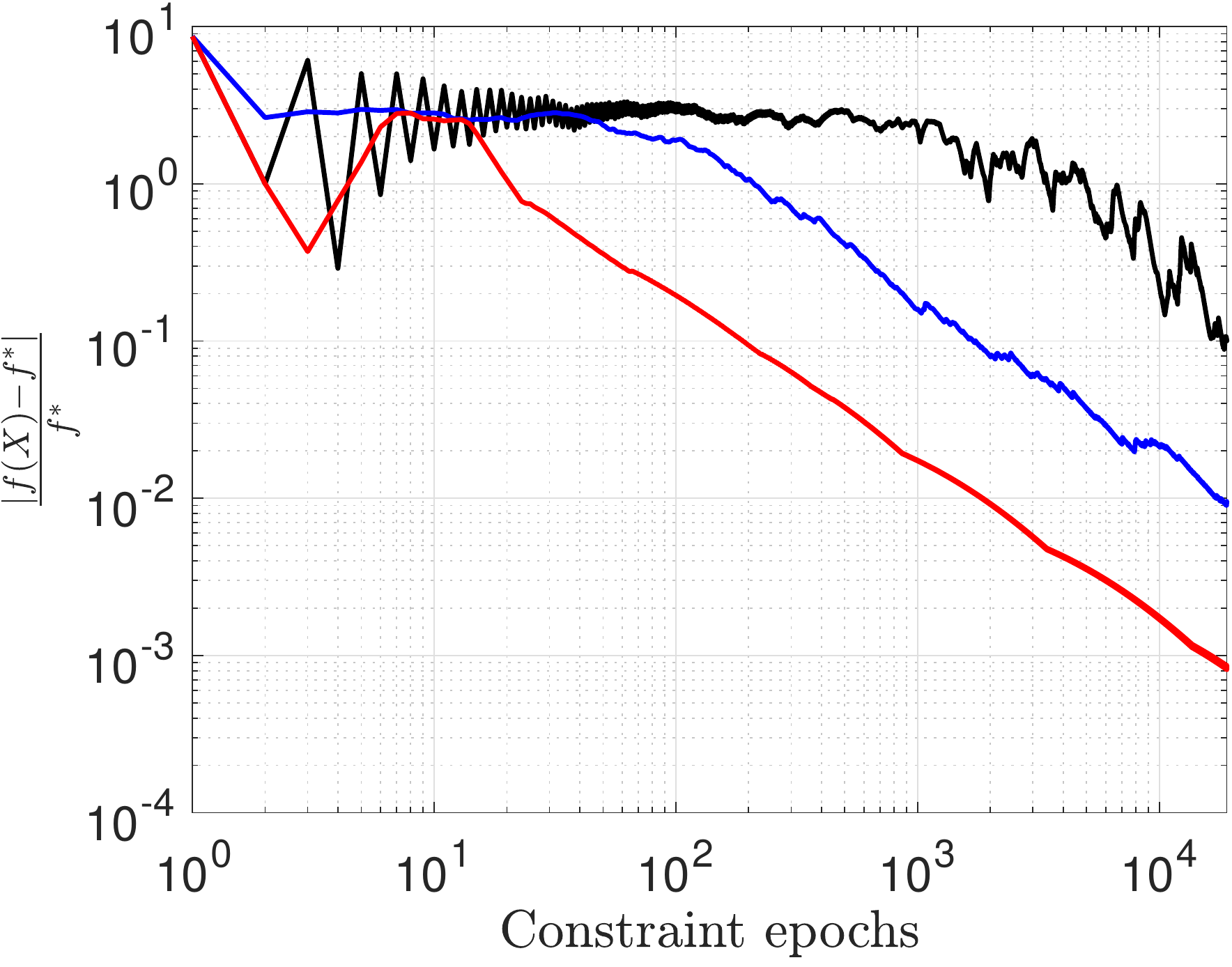}
\includegraphics[width=.33\linewidth]{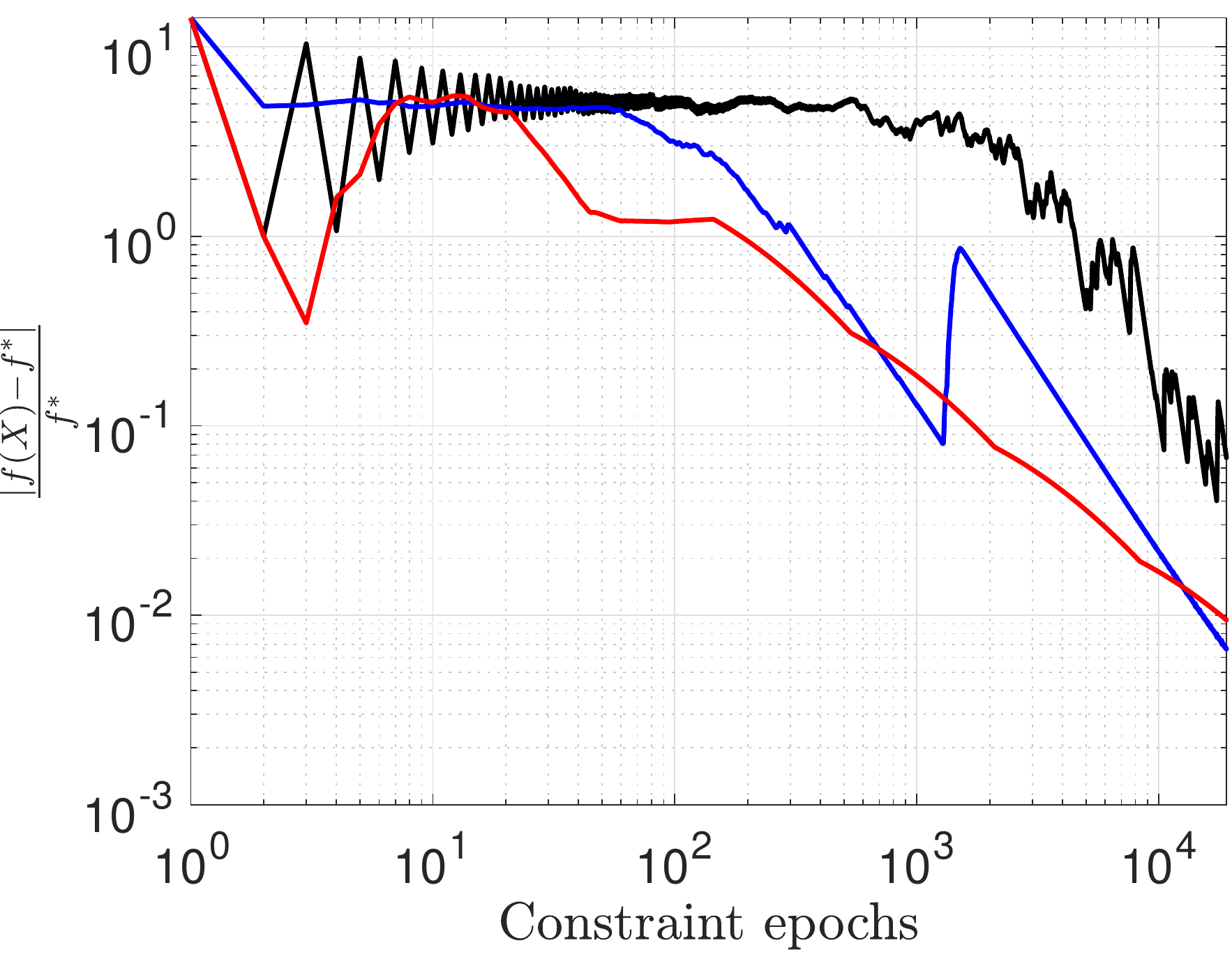}\par
\includegraphics[width=.33\linewidth]{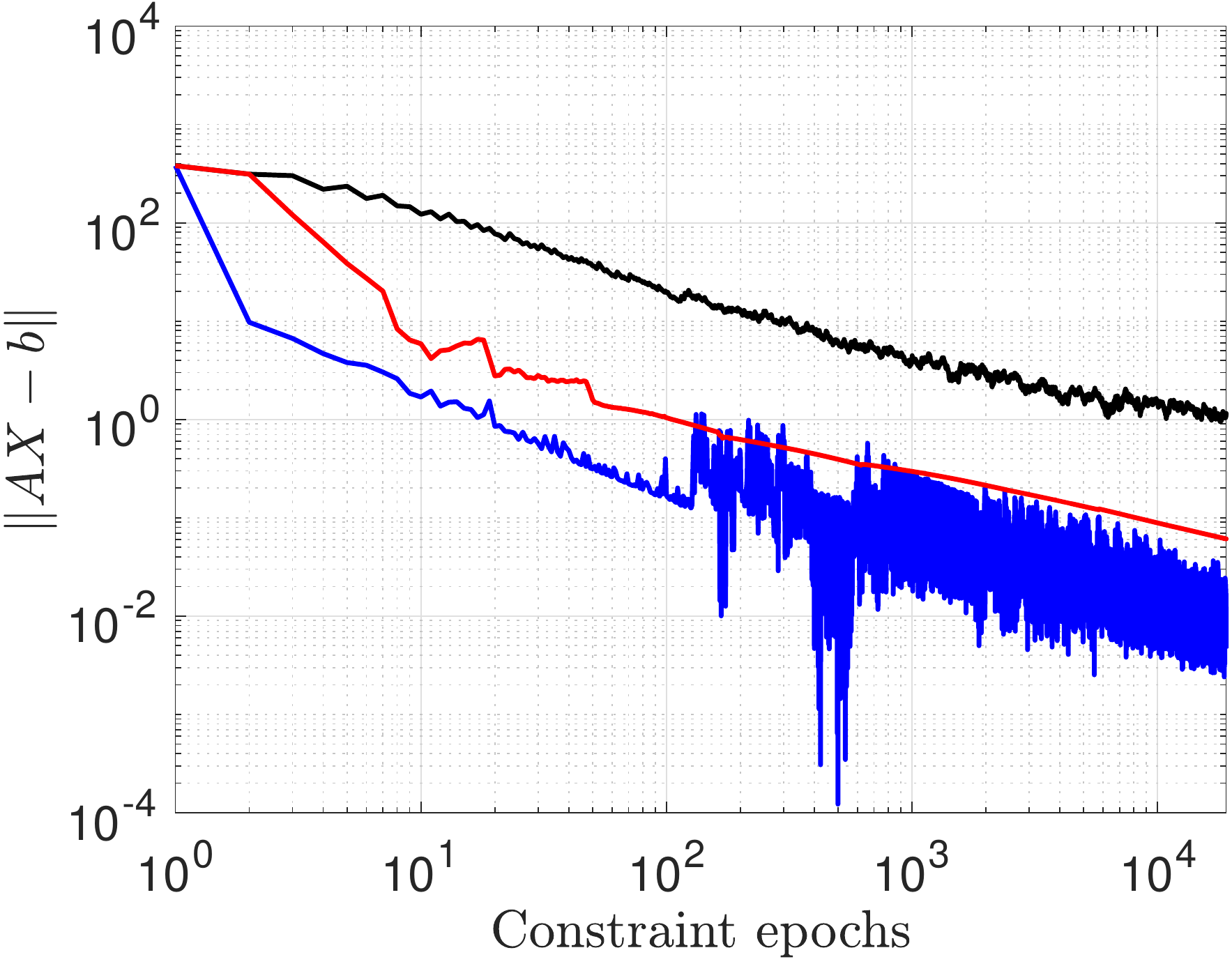}
\includegraphics[width=.33\linewidth]{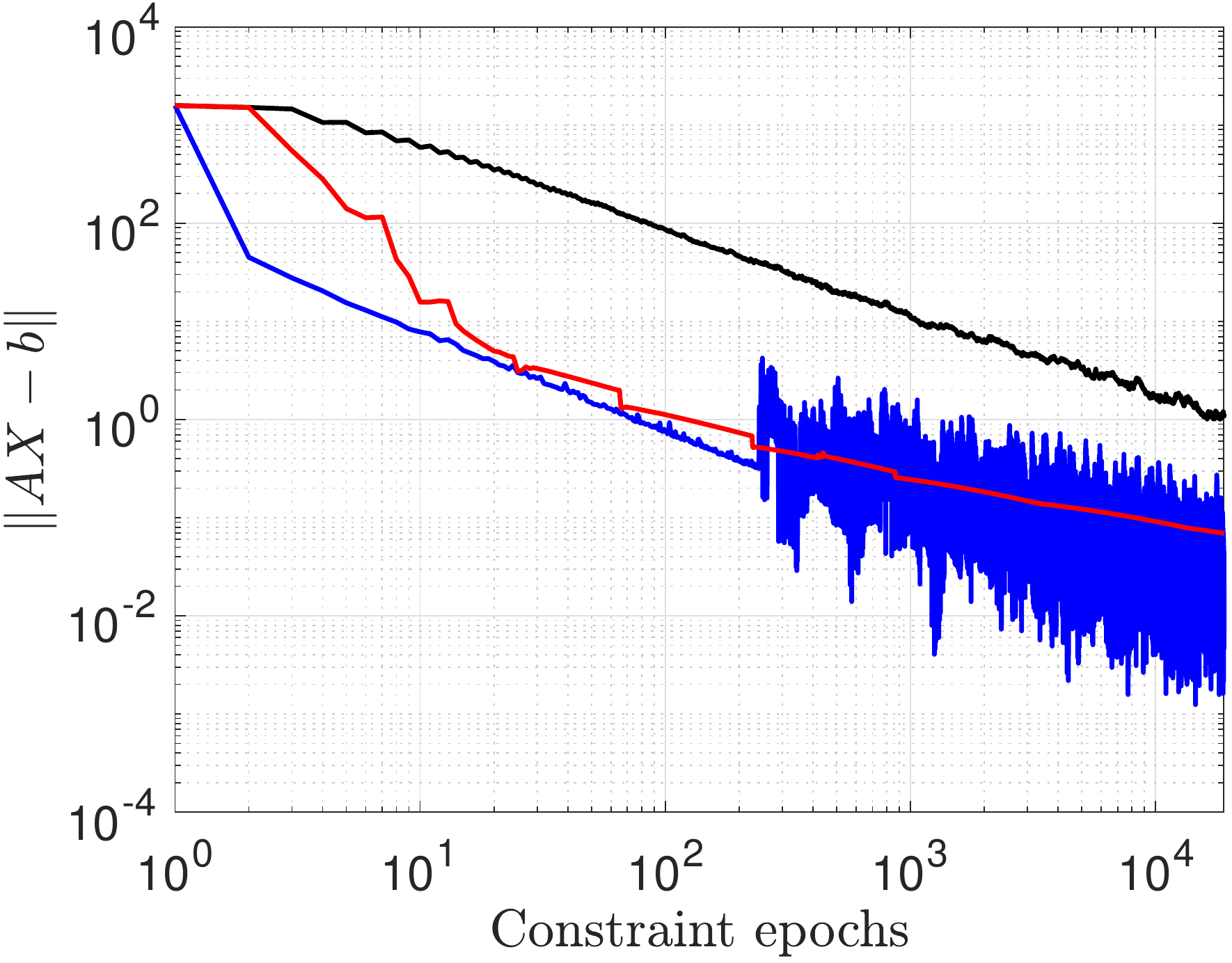}
\includegraphics[width=.33\linewidth]{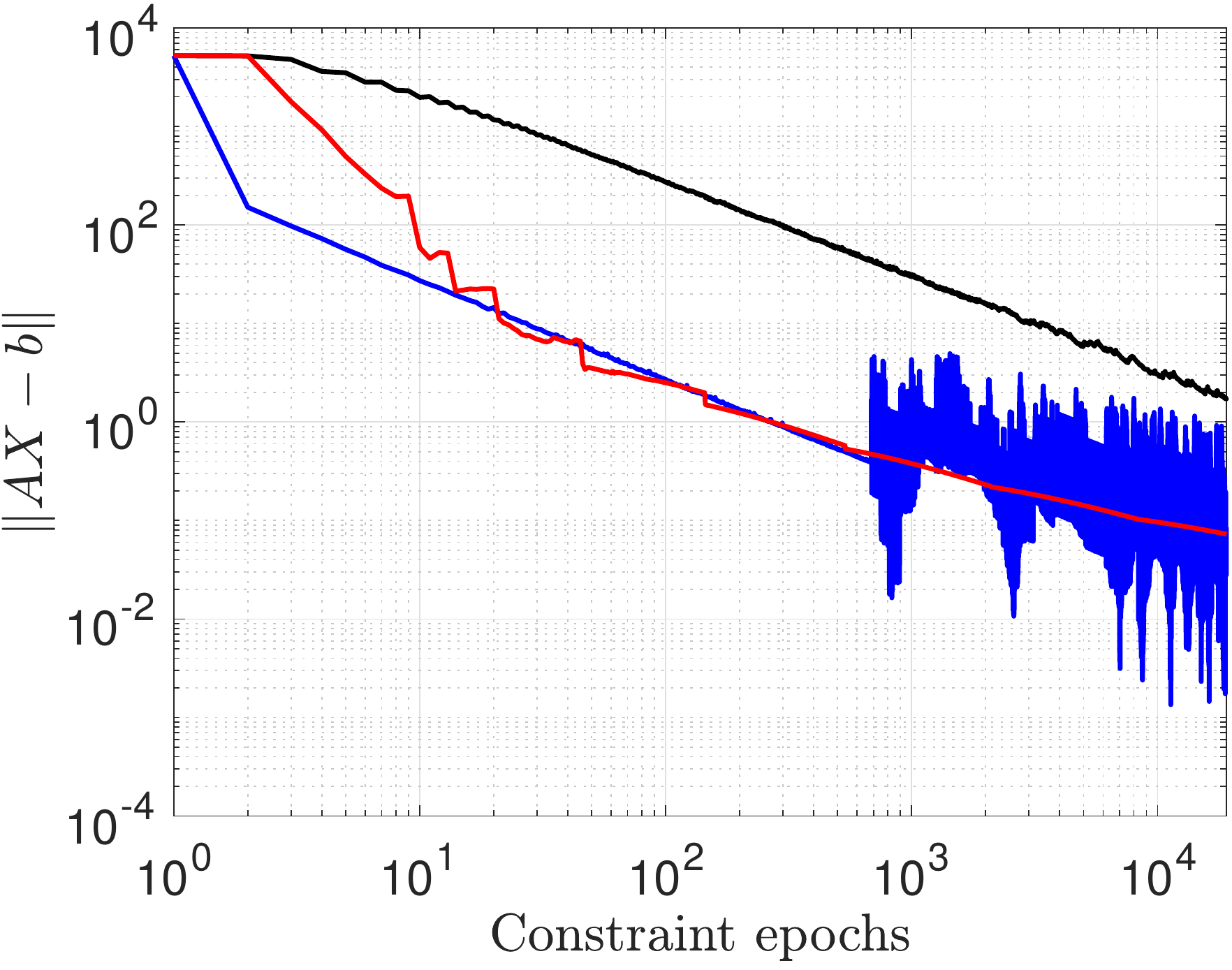}
\caption{The Sparsest Cut-associated SDP relaxation, where each column shows the convergence in objective suboptimality (top) and feasibility (bottom) for a specific problem. From left to right, the results correspond to graphs \emph{mammalia-primate-association-13}, \emph{insecta-ant-colony1-day37} and \emph{insecta-ant-colony4-day10}, sorted by increasing size. }\label{fig:sparsest_cut}
\end{minipage}
\end{figure*}

For demonstrating the empirical efficiency of our algorithms, we apply them to three problem instances: synthetically-generated SDPs, the K-means clustering SDP relaxation and the Sparsest Cut-associated SDP.

\textbf{Evaluation Metrics:} {Our experiments subscribe to a finite-sum template, where we define $f(x) \defeq \sum_{i = 1}^{n_1} f_i(x)$ and $g_{\beta}(Ax) = \sum_{i = 1}^{n_2} g_{i, \beta}(A_i^Tx)$. The objective convergence is recorded as $|f(x) - f^\star|$, with $f^\star \defeq \fopt$. Due to imperfect feasibility, the value of $f(x)$ can overshoot $f^\star$, since the constrained optimum is not the global one. This usually appears as the increase of $|f(x) - f^\star|$ immediately after a significant drop when the quantity $f(x) - f^\star$ becomes negative; then the decreasing trend restarts, as the objective and constraints re-balance. Such a phenomenon is common for homotopy-based methods, see for instance \cite{yurtsever2018conditional}. Lastly, the feasibility is recorded as $\| Ax- b \|$.}

\textbf{Baseline:} {To the best of our knowledge, the HCGM \cite{yurtsever2018conditional} and the SHCGM \cite{locatello2019stochastic} are the only algorithms which tackle SDPs under the conditional gradient framework. The latter represents the empirical state-of-the-art and we choose it as the baseline for our experiments.
}

\subsection{Synthetic SDP Problems}
\label{sec:exp-synth}
This proof-of concept experiment aims to show the performance of our fully stochastic methods, given a fixed problem dimension and an increasing set of constraints. We consider the synthetic SDP:
\begin{align}
    &\pmb{\min\limits_{\substack{ X \in \mathbb{S}^d_+ \\ \text{tr}(X) \leq \frac{1}{d}}}} &&\hspace{-12mm}\left\langle C,  X \right\rangle  \nn \\
                &\textbf{subject to } &&\hspace{-12mm}\text{tr}(A_iX) = b_i, i = 1 \ldots n \nn
\end{align}
where the entries of $A_i$ and $C$ are generated from $\mathcal{U}(0, 1)$, and $b_i = \langle A_i, X^*\rangle$ for a fixed $X^*$. We perform uniform sampling on the pairs $(A_i, b_i)$ for computing their stochastic gradients in our algorithms. We fix the dimension to be $d=20$ and vary the size of constraints with $n = \texttt{5e2}$ and $\texttt{5e3}$.

For a fair comparison, we sweep the parameter $\beta_0$ for the three algorithms in the range $\left[\texttt{1e-7}, \texttt{1e1}\right]$. We settle for $\texttt{1e-7}$, $\texttt{1e-7}$ and $\texttt{1e-5}$ for SHCGM, H-1SFW and H-SPIDER-FW, respectively. For H-1SFW and SHCGM, we choose the batchsize to be 1\% of the data.

Figure~\ref{fig:conv1e4} illustrates the outcome of the experiments, where we observe a clear improvement of the stochastic algorithms over the baseline with a stable margin throughout the test cases.

Interestingly, H-1SFW exhibits strong empirical performance on the synthetic data, much better than its theoretical worst-case bound. A possible explanation is that the entries of $C$ and $A_i$ are generated from a ``benign'' distribution and \emph{concentrate} around its mean \cite{ledoux2001concentration}. In such scenarios, even a small subset of constraints allows for effective variance reduction. For comparison, we provide an additional set of results for synthetic SDPs generated from a less well-behaved distribution in \mbox{Appendix~\ref{sec:app-synth-exp}}. Nevertheless, we observe the same good performance of H-1SFW even with real data, in the next sections.


Regarding H-SPIDER-FW, we observe that the suboptimality and feasibility decrease at the rate $k^{-\frac{1}{2}}$ and $k^{-\frac{3}{4}}$, respectively, which is better than the worst-case bounds in Theorem~\ref{thm:spiderfw}.

\subsection{The K-means Clustering Relaxation}
\label{sec:stochastic-k-means}

\renewcommand{\arraystretch}{1.5}
\setlength{\tabcolsep}{6pt}
\begin{table*}[tp]
\caption{Details of the Network Repository \cite{nr} graphs used in the experiments.}
\label{table:graphs}
\centering
\begin{tabular}{ >{\centering\arraybackslash}  m{4.9cm}  >{\centering\arraybackslash}  m{0.8cm}  >{\centering\arraybackslash}  m{0.8cm}  >{\centering\arraybackslash}  m{1.7cm}  >{\centering\arraybackslash}  m{1.7cm}  >{\centering\arraybackslash}  m{2.1cm}  >{\centering\arraybackslash} m{2.2cm}} 
\toprule
 \textbf{Graph name} & $\pmb{\lvert V \rvert}$ &$\pmb{\lvert E \rvert}$ & \textbf{\shortstack{Avg. \\ node degree }} &\textbf{\shortstack{Max. \\ node degree }}  & \textbf{\shortstack{USC SDP \\ dimension}} & \textbf{\shortstack{ USC SDP \\ \#  constraints}} \\ 
\midrule
  mammalia-primate-association-13 & 25 & 181 & 14 & 19 & $X \in \R^{25 \times 25}$ & $\sim \texttt{6.90e3}$\\
  insecta-ant-colony1-day37 & 55 & 1k & 42 & 53 & $X \in \R^{55 \times 55}$ &  $\sim \texttt{7.87e4}$ \\
  insecta-ant-colony4-day10 & 102 & 4k & 79 & 99 & $X \in \R^{102 \times 102}$ & $\sim \texttt{5.15e5}$ \\
\bottomrule
\end{tabular}
\end{table*}

We consider the unsupervised learning task of partitioning $d$ data points into $k$ clusters. We adopt the SDP formulation in \cite{peng2007approximating}, which amounts to solving:
\begin{align}
\label{eq:sdp_clustering}
&\pmb{\min_{X\in\mathcal{X} }} &&\hspace{-34mm} \left\langle  C, X \right\rangle \nn\\  
&\textbf{subject to }  &&\hspace{-34mm} X\vec{1} = \vec{1},\nn\\
&&\hspace{-12mm} X_{i,j} \geq 0, \;\; 1 \leq i,j \leq d.
\end{align}
Here, $C \in \R^{d\times d}$ is the Euclidean distance matrix of the $d$ data points, $\mathcal{X} = \{  X \in \R^{d\times d}: X\succeq 0,\ \textup{tr}(X) \leq k \}$, $\vec{1}$ is the all 1's vector. Notice that the number of linear constraints in \eqref{eq:sdp_clustering} is $\mathcal{O}(d^2)$.

In order to compare against existing work, we adopt the MNIST dataset ($k=10$) \cite{lecun-mnist} with $d = 10^3$ samples and perform data preprocessing as in \cite{mixon2016clustering}. The very same setup appeared in several works \cite{mixon2016clustering, yurtsever2018conditional, locatello2019stochastic}, with SHCGM \cite{locatello2019stochastic} showing the best practical performance.

We perform parameter sweeping on $\beta_0 \in \left[\texttt{1e-7}, \texttt{1e2}\right]$ for H-1SFW and H-SPIDER-FW, and settle for \texttt{5e-2} and \texttt{6e0}, respectively. For SHCGM, we adopt the same hyperparameter as in \cite{locatello2019stochastic}. The batchsize for H-1SFW and SHCGM is set to 5\%.

The comparison of our algorithms against SHCGM is reported in Figure~\ref{fig:mnist_small}. H-1SFW and H-SPIDER-FW converge at a comparable rate, with both clearly overtaking the baseline with regards to objective suboptimality and feasibility convergence.

\subsection{Computing an $\ell_2^2$ Embedding for the Uniform Sparsest Cut Problem}
\label{sec:sp-cut}

The Uniform Sparsest Cut problem (USC) aims to find a bipartition $(S, \bar{S})$ of the nodes of a graph $G = (V, E)$, $\vert V \vert = d$, which minimizes the quantity 
\begin{equation*}
\label{sceq}
       \frac{E(S, \bar{S})}{\lvert S \rvert \lvert \bar{S} \rvert}, 
\end{equation*}
where $E(S, \bar{S})$ is the number of edges connecting $S$ and $\bar{S}$. This problem is of broad interest, with applications in areas such as VLSI layout design, topological design of communication networks and image segmentation, to name a few. Relevant to machine learning, it appears as a subproblem in hierarchical clustering algorithms \citep{dasgupta2016cost, chatziafratis2018hierarchical}.

Computing such a bipartition is NP-hard and intense research has gone into designing efficient approximation algorithms for this problem. In the seminal work of \citet{arv} an $\bigO{\sqrt{\log d}}$ approximation algorithm is proposed for solving USC, which relies on finding a \emph{well-spread} $\ell_2^2$ geometric representation of $G$ where each node $i\in V$ is mapped to a vector $v_i$ in $\mathbb{R}^d$. In this experimental section we focus on solving the SDP that computes this geometric embedding, as its high number of triangle inequality constraints ($\bigO{d^3}$) makes it a suitable candidate for our framework. The canonical formulation of the SDP is given below (for the original formulation, see \mbox{Appendix~\ref{subsec:app-usc}}).
\vspace{2mm}

\scalebox{0.98}{
    $\!\begin{aligned}
   &\pmb{\min_{X \in \X}}&&\hspace{-66mm}  \langle L, X \rangle \\
   &\textbf{subject to} &&\hspace{-66mm} d\Tr(X) - \Tr(\mathbf{1}_{d\times d}X) = \frac{d^2}{2} \\
   && X_{i,j} + X_{j,k} - X_{i, k} - X_{j,j} \leq 0, \;\; \forall\ i, j, k \in V  \\
  \end{aligned}$  
  }
  
\vspace{3mm}
Here, $L$ represents the Laplacian of $G$, $\X = \{  X \in \R^{d\times d}: X\succeq 0,\ \textup{tr}(X) \leq d \}$ and $X_{i,j} = \dotprod{v_{i}}{v_j}$ gives the geometric embedding of the nodes. We run our algorithms on three graphs of different sizes from the Network Repository dataset \cite{nr}, whose details are summarized in Table~\ref{table:graphs}. Note the cubic dependence of the number of constraints relative to the number of nodes. We perform parameter sweeping on $\beta_0 \in [\texttt{1e-5}, \texttt{1e5}]$ using the smallest graph, \emph{mammalia-primate-association-13}, and keep the same parameters for all the experiments. The values of $\beta_0$ for SHCGM, H-1SFW and H-SPIDER-FW are \texttt{1e2}, \texttt{1e-2} and \texttt{1e1} respectively, and the batch size for both H-1SFW and SHCGM is set to 5\%.

Figure~\ref{fig:sparsest_cut} depicts the outcomes of the experiments, with both our algorithms consistently outperforming SHCGM and H-SPIDER-FW attaining the fastest convergence. A possible explanation is that, given the much larger number of constraints relative to the problem dimension ($\mathcal{O}\left( n^3  \right)$ v.s $\mathcal{O}\left( n^2  \right)$), H-SPIDER-FW's increasing minibatches readily reach an adequate balance between feasibility enforcement and objective minimization.

\section*{Acknowledgements}

The authors are grateful to Mehmet Fatih Sahin and Alp Yurtsever for the helpful discussions throughout the development of this paper.

This work was partially supported by the Swiss National Science Foundation (SNSF) under  grant number 200021\_178865 / 1; the Army Research Office under Grant Number W911NF-19-1-0404; the European Research Council (ERC) under the European Union's Horizon 2020 research and innovation programme (grant agreement no 725594 - time-data).

\bibliography{references}
\bibliographystyle{icml2020}

\newpage

\appendix
\captionsetup[table]{skip=10pt}
\onecolumn

{\Huge Appendix}

\section{Additional Experiment Information}
In this section we provide some omitted experiment details.

\subsection{Experiment Setup}
The experiments presented in this paper were implemented in MATLAB R2019b and executed on a 2,9 GHz 6-Core Intel Core i9 CPU with 32 GB RAM. For retrieving the values of $\fopt$ we used the code of~\citep{mixon2016clustering} which relies on SDPNAL+~\cite{yang2015sdpnal} for the clustering experiments, and CVX for the Sparsest Cut ones. The code is included in the supplemental material. 

\subsection{Additional results for synthetic SDPs}
\label{sec:app-synth-exp}
The setup for these experiments is the same as that of Section~\ref{sec:exp-synth}, but with a different distribution for generating $A_i$ and $C$. Specifically, we use the heavy-tailed Stable distribution with parameters $(\alpha=1.5, \beta=0, \gamma=10, \delta=0)$. We sweep $\beta_0$ for all three algorithms in the range $\left[\texttt{1e-7}, \texttt{1e-1}\right]$ and settle for $\texttt{1e-5}$, $\texttt{1e-7}$, $\texttt{1e-6}$ for SHCGM, H-1SFW and H-SPIDER-FW, respectively. The results are depicted in Figure~\ref{fig:add-synth}.

We observe that, given this more difficult distribution, all methods are comparable in terms of convergence speed for both objective suboptimality and feasibility, with H-SPIDER-FW having an edge over the other two. 

\begin{figure*}[hb!]
\centering
\begin{minipage}[t]{\textwidth}
\setlength{\lineskip}{0pt}
\includegraphics[width=.33\linewidth]{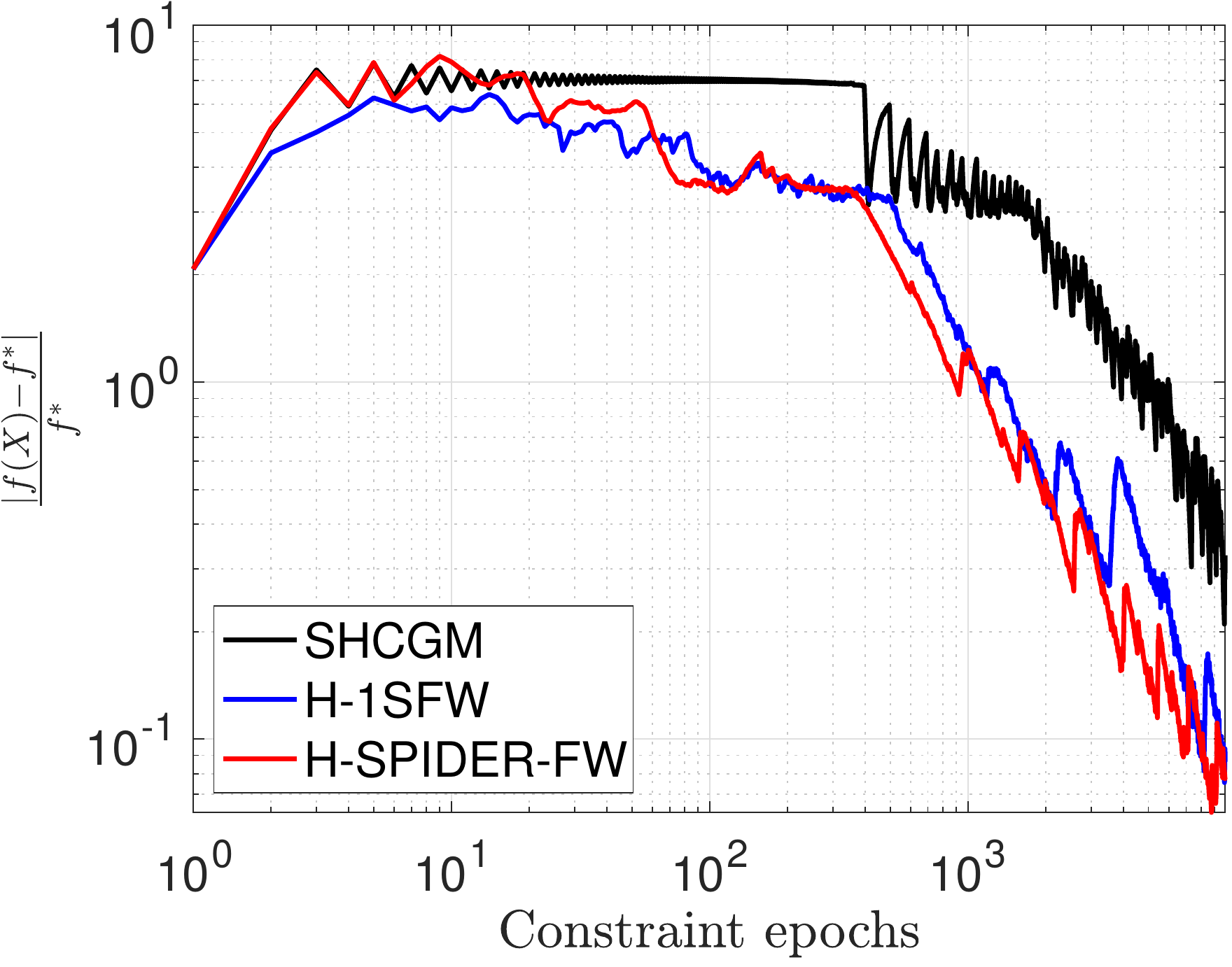}
\includegraphics[width=.33\linewidth]{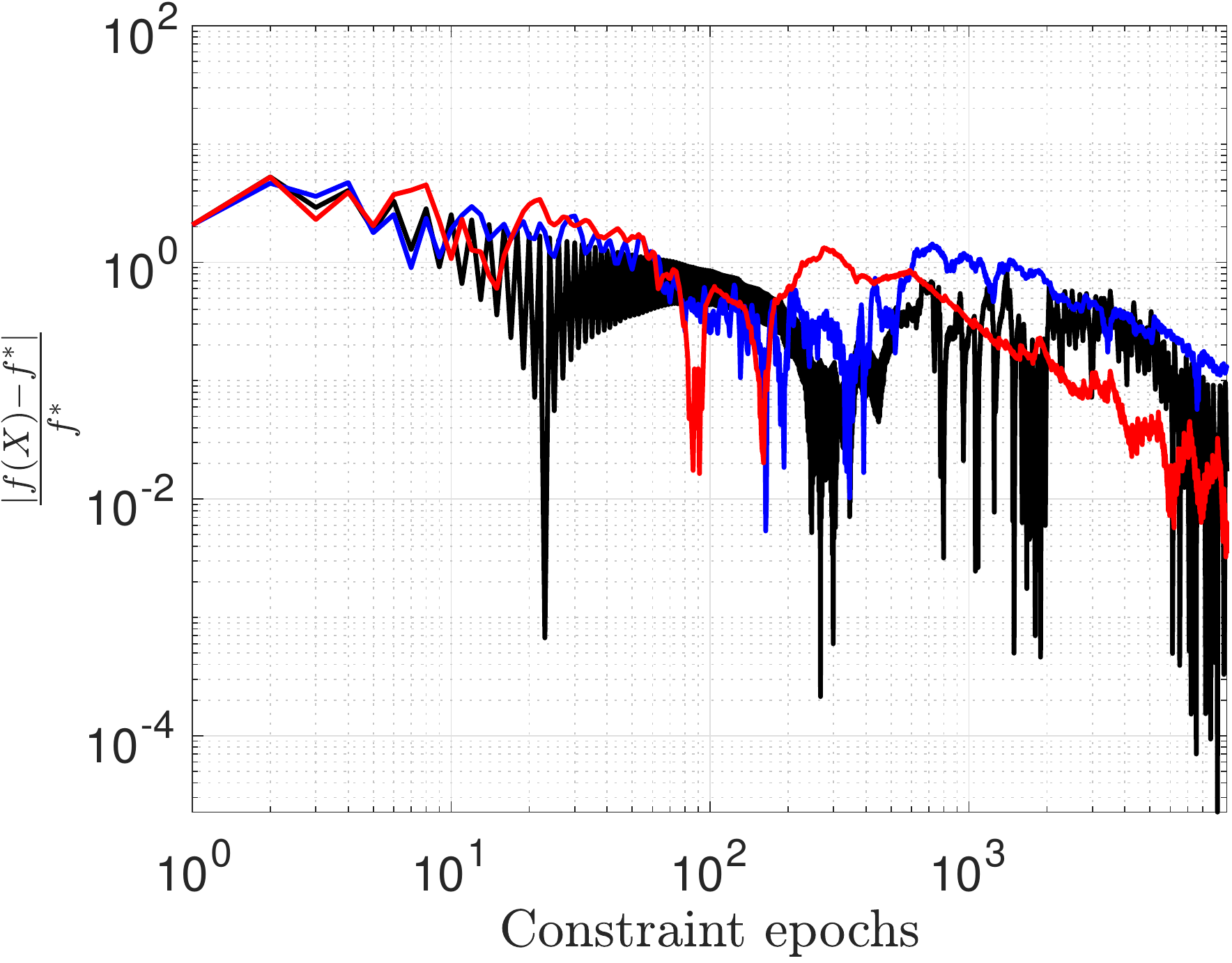}
\includegraphics[width=.33\linewidth]{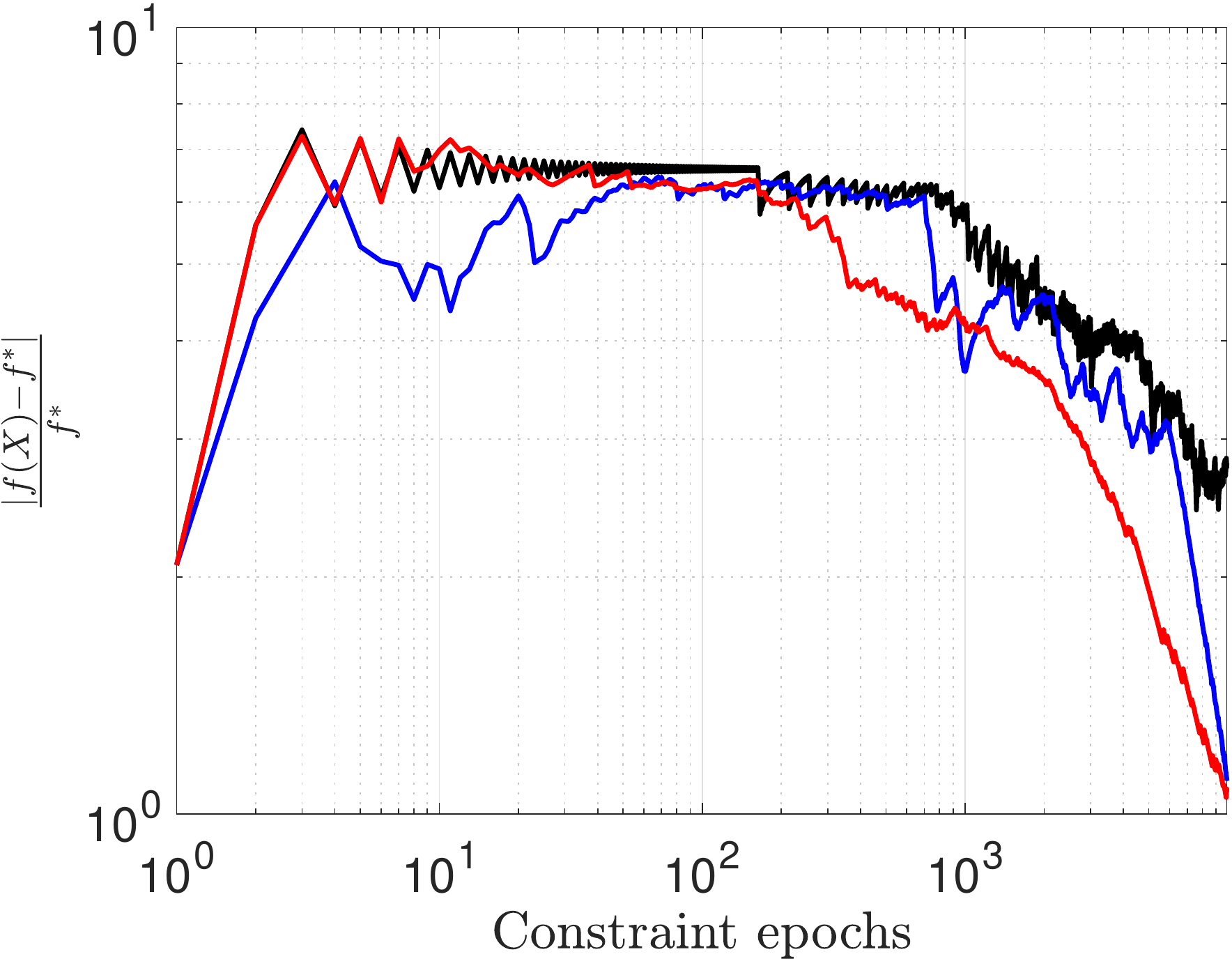}\par
\includegraphics[width=.33\linewidth]{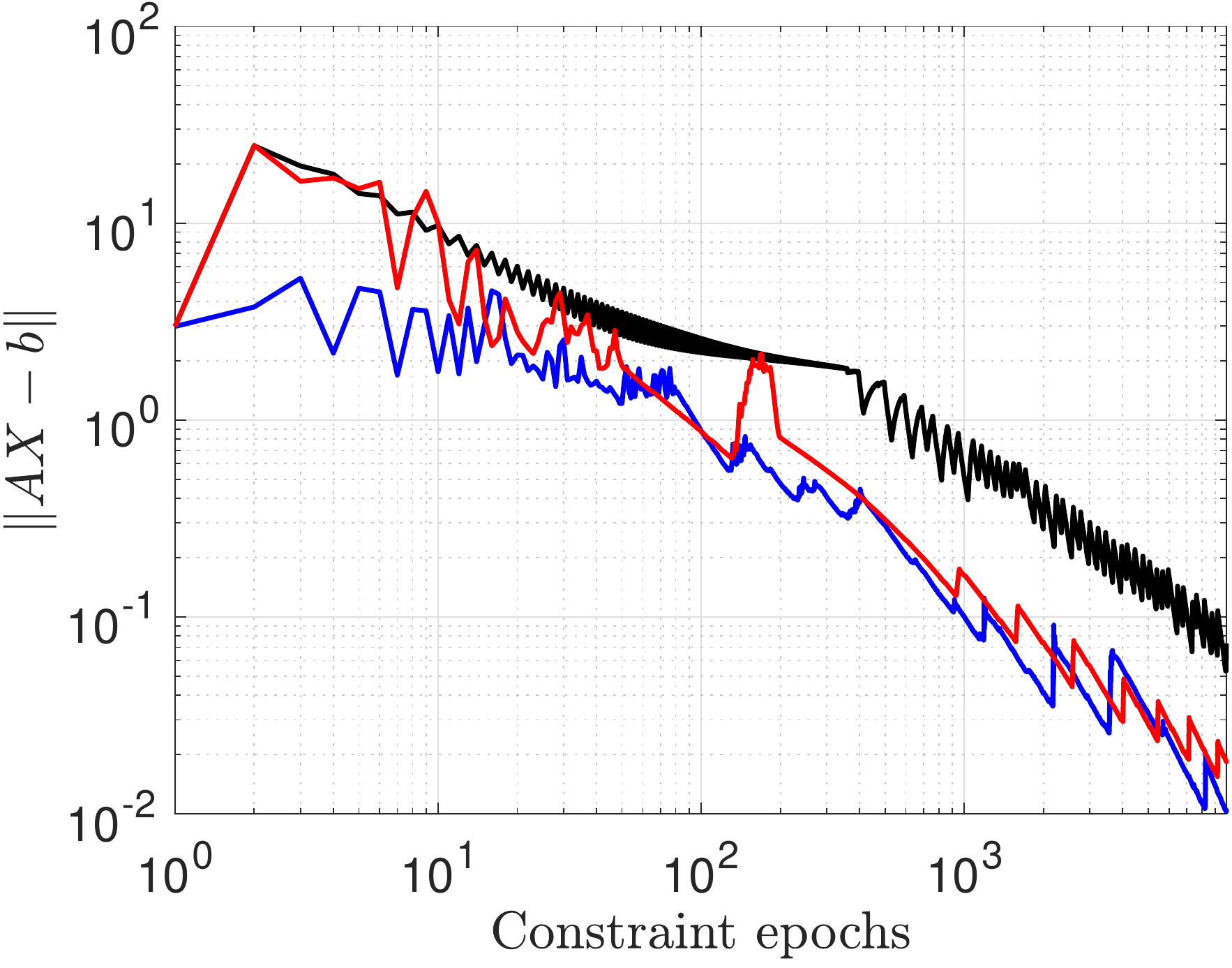}
\includegraphics[width=.33\linewidth]{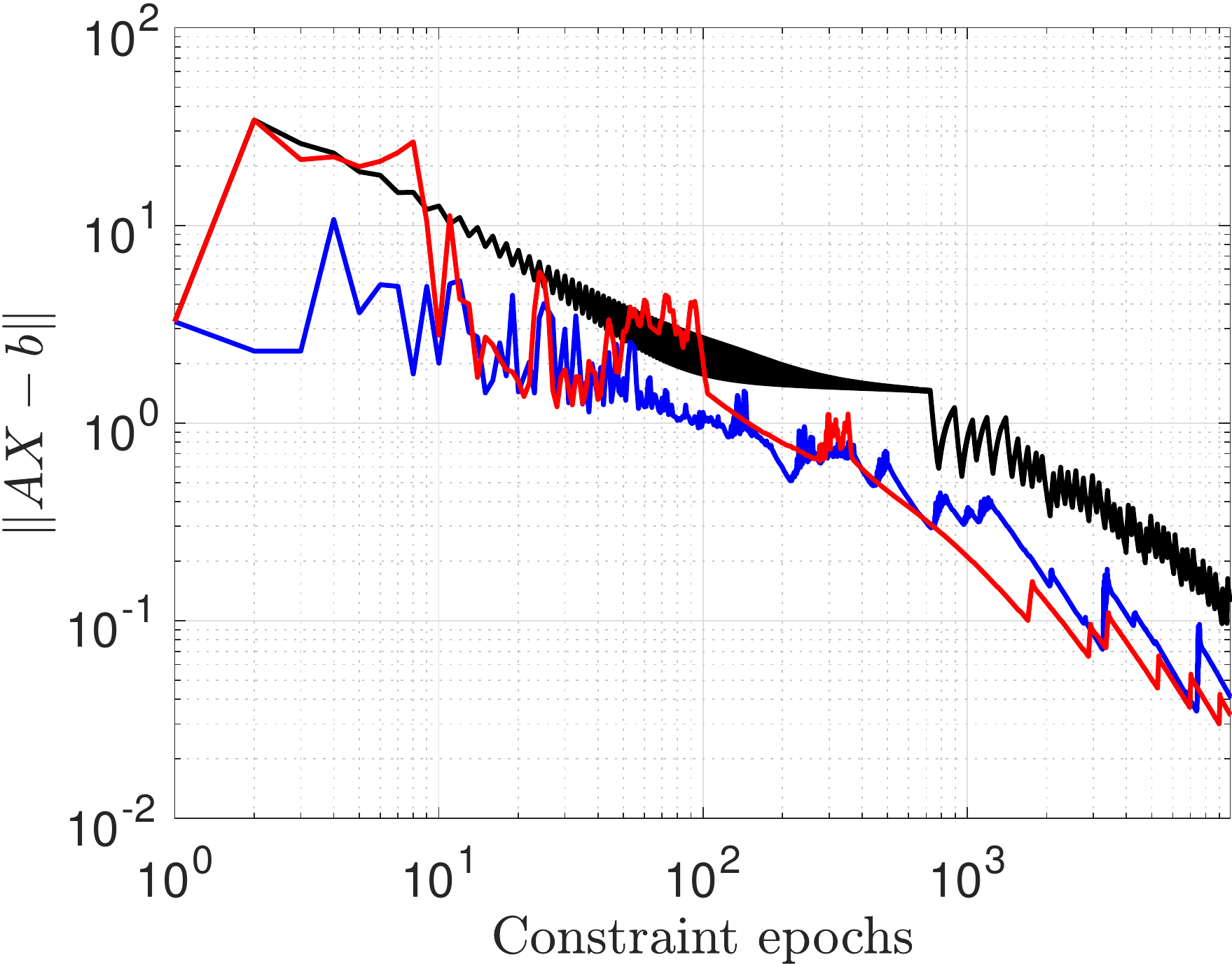}
\includegraphics[width=.33\linewidth]{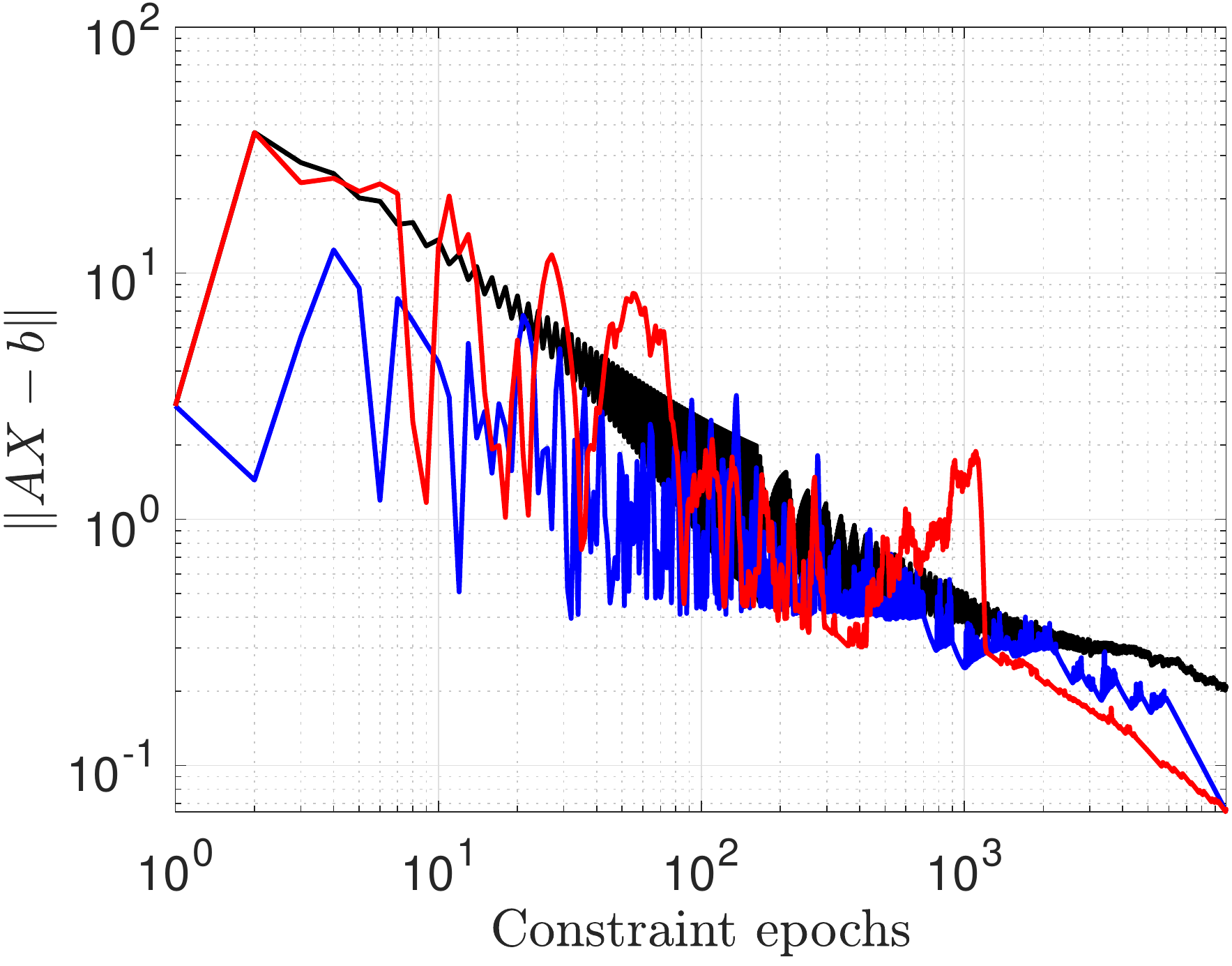}
\caption{Synthetic SDPs, with each column showing the convergence in objective suboptimality (top) and in feasibility (bottom) for a specific problem. From left to right, the columns depict the results for problems with $\texttt{5e2}$, $\texttt{1e3}$ and $\texttt{5e3}$ constraints.}\label{fig:add-synth}
\end{minipage}
\end{figure*}

\subsection{The Uniform Sparsest Cut SDP}
\label{subsec:app-usc}

The left column of Table~\ref{table:sdp} provides the original SDP formulation of~\citep{arv} for finding the $\ell_2^2$ embedding of nodes $i \in V$; the right column contains the corresponding canonical formulation. In our experiments we use the latter formulation to which we add the trace constraint $\mathrm{tr}(X) \leq d$. This additional constraint does not change the optimal objective~\cite{iyengar2010feasible}.

\def\Tr{{ \textup{Tr}  }}
\setlength{\tabcolsep}{10pt}
\begin{table}[h!]
\caption{SDP formulations for retrieving the $\ell_2^2$ embedding of graph nodes.}
\label{table:sdp}
\centering
\begin{tabular}{ll}
\toprule
 \textbf{Original SDP} & \textbf{Canonical SDP}  \\ 
\midrule
 \parbox{8cm}{ \vspace{4mm}\scalebox{0.85}{
 $\!\begin{aligned}
   \textbf{minimize} & \quad \frac{1}{d^2}\sum_{(i,j) \in E} \lvert v_i - v_j\rvert^2
   \\
   \\
   \textbf{subject to} & \quad \sum_{\substack{i, j \in V \\ i \neq j} } \lvert v_i - v_j\rvert^2 = d^2
   \\
   & \quad \lvert v_i - v_j\rvert^2 + \lvert v_j - v_k\rvert^2 \geq \lvert v_i - v_k\rvert^2 \quad \forall\ i, j, k \in V \\
   \ \\
  \end{aligned}$}} & \parbox{8cm}{
    \vspace{6mm} \scalebox{0.85}{
    $\!\begin{aligned}
   \textbf{minimize}& \quad \Tr(LX)
   \\
   \\
   \textbf{subject to} & \quad d\Tr(X) - \Tr(\mathbf{1}_{d\times d}X) = \frac{d^2}{2}
    \\
   \\
   &\vspace{2mm}\quad X_{i,j} + X_{j,k} - X_{i, k} - X_{j,j} \leq 0 \quad \forall\ i, j, k \in V  \\
   \ \\
  \end{aligned}$  
  }} \\ 
\bottomrule
\end{tabular}
\end{table}

\vspace{5mm}

\section{Omitted proofs}

\subsection{Preliminaries}
\label{app:prelim}

We begin by introducing some new notation used throughout the proofs and state some simple technical observations:

\begin{enumerate}[itemsep=10pt]
\item \textbf{Notation}
            \begin{itemize}[itemsep=15pt]
		    \item $L_A \defeq \sup_\xi \| A(\xi) \|^2$;
		    \item $f(x) \defeq \mathbb{E} \left[ f(x, \xi)\right]$. 
		    \item From above it follows that Assumption~\ref{ass:fstoc_finite_var} can be rewritten as: \mbox{$\mathbb{E} \left[ \nabla f(x, \xi) \right]  = \nabla f(x)$} and \mbox{$\mathbb{E} \left[ \| \nabla f(x, \xi) - \nabla f(x)\|^2\right] \leq \sigma_f^2 < +\infty$}; 
		    \item $\g{\Astoc x} \defeq \delta_{\{\bstoc\}}(\Astoc x)$;
		    \item $\!
		    \begin{aligned}[t]
		                        & \g[\beta]{\Astoc x}  && \defeq \frac{1}{2\beta} \mathrm{dist}(\Astoc x, \bstoc)^2 \\
 & && =  \frac{1}{2\beta} \normsqr{\Astoc x - \Pi_{\bstoc}(\Astoc x)}, \\
 &  \text{where } && \Pi_{\bstoc}(\Astoc x)) = \argmin\limits_{y \in \bstoc} \normsqr{\Astoc x - y}; \text{also, } g_\beta \text{ is } \frac{1}{\beta} \text{-smooth}
               \end{aligned}$, 
		    \item $\G{\beta}{Ax} \defeq \Exp{\g[\beta]{\Astoc x}}{}, \;\; \nabla \G{\beta}{Ax} \defeq \Exp{\nabla \g[\beta]{\Astoc x}}{}$, where $A: \R^d \rightarrow \H$ is a linear operator such that $(Ax)\xi = \Astoc x$ and $G_{\beta}: \H \rightarrow \R \cup \{ \infty \}$.
		    \item $ F_{\beta_k}(x, \xi) \defeq f(x, \xi) + g_{\beta_k}(x, \xi), \;\; \nabla F_{\beta_k}(x, \xi) \defeq \nabla f(x, \xi) + \nabla g_{\beta_k}(x, \xi)$
		    \item We annotate averaged stochastic quantities with the symbol  $\sim$. For example, the averaged stochastic gradient of the constraints is expressed as $\tilde{\nabla}_x \g[\beta]{\Astoc x}$;
		    \item The optimal value of the dual problem at $\Astoc x$ is denoted as $\lambda^*_{\beta}(\Astoc x) \defeq \frac{1}{\beta} (\Astoc x - \Pi_{\bstoc}(\Astoc x)) \label{eq:def-y-beta}$;
		    \item The smoothed gap is defined as $S_{\beta}(x) \defeq F_{\beta}(x) - \fopt$.
		\end{itemize}
\item \textbf{Technical observations}
            \begin{enumerate}[label=\alph*., itemsep=15pt]
                \item \label{list:technical-obs-Gexp}From the definition of $G_{\beta}$: 
                        \begin{align*}
                            \nabla_x \G{\beta}{Ax} &= \Exp{\nabla_x \g[\beta]{\Astoc x}}{} \\
                                                                &= \Exp{A^T(\xi) \nabla \g[\beta]{\Astoc x}}{} \\
                                                                &= \Exp{\frac{1}{\beta}A^T(\xi) \left( \Astoc x - \Pi_{\bstoc}(\Astoc x) \right)}{};
                            \end{align*}
                \item \label{list:technical-obs-diff-Gs} Form smoothness of $G_{\beta}$, iterate update rule and non-expansiveness of projections: 
                                \begin{align*}
                                \normsqr{ \nabla G_{\beta}(Ax_{k+1}) - &\nabla G_{\beta}(Ax_k) } \\
                                        &= \left\lVert\frac{1}{\beta}\Exp{A^T(\xi) \left( \Astoc x_{k+1} - \Pi_{\bstoc}(\Astoc x_{k+1}) \right) - A^T(\xi) \left( \Astoc x_k - \Pi_{\bstoc}(\Astoc x_k) \right)}{} \right\rVert^2\\
                                &\leq \frac{1}{\beta^2} \mathbb{E} \left[ \left\lVert A^T(\xi) \Astoc \left( x_{k+1} - x_k\right)  + A^T(\xi) \left( \Pi_{\bstoc}(\Astoc x_k) - \Pi_{\bstoc}(\Astoc x_{k+1}) \right) \right\rVert^2\right] \\
                                &\leq \frac{1}{\beta^2} \mathbb{E} \left[ 2\left\lVert A^T(\xi) \Astoc \left( x_{k+1} - x_k\right) \right\rVert^2  + 2\left\lVert A^T(\xi) \left( \Pi_{\bstoc}(\Astoc x_k) - \Pi_{\bstoc}(\Astoc x_{k+1}) \right) \right\rVert^2\right] \\
                               &\leq \frac{2\gamma_k^2L_A^2\diam^2}{\beta^2} + \frac{2}{\beta^2}\mathbb{E} \left[\left\lVert A(\xi) \right\rVert^2 \left\lVert\Pi_{\bstoc}(\Astoc x_k) - \Pi_{\bstoc}(\Astoc x_{k+1})  \right\rVert^2\right] \\
                              &\leq \frac{2\gamma_k^2L_A^2\diam^2}{\beta^2} + \frac{2}{\beta^2}\mathbb{E} \left[\left\lVert A(\xi) \right\rVert^2 \left\lVert \Astoc x_k - \Astoc x_{k+1}  \right\rVert^2\right] \\
                              &\leq \frac{4\gamma_k^2L_A^2\diam^2}{\beta^2} ; \\
                                \end{align*}                
			    \item \label{list:technical-obs-g-variance} Variance of  $g_{\beta} (A(\xi)x, \xi)$:
			            \begin{align*}
                			    \mathbb{E} \left[ \left\lVert \nabla g_{\beta} (A(\xi)x, \xi) - \nabla G_{\beta}(Ax) \right\rVert^2 \right] &= \mathbb{E} \left[ \left\lVert \nabla g_{\beta}(A(\xi)x, \xi) \right\rVert ^2 - \left\lVert \nabla G_{\beta}(Ax) \right\rVert ^2 \right]\\
                			    &\leq \mathbb{E} \left[ \left\lVert \nabla g_{\beta}(A(\xi)x, \xi) \right\rVert ^2 \right]\\
                			    &\leq \frac{1}{\beta^2}\mathbb{E} \left[ \left\lVert A(\xi) \right\rVert^2 \left\lVert A(\xi)x - \Pi_{\bstoc}(\Astoc x) \right\rVert ^2 \right]\\
                			    &\leq \frac{1}{\beta^2} \mathbb{E} \left[ \left\lVert A(\xi) \right\rVert^2 \left\lVert A(\xi)x - \Astoc \xopt \right\rVert ^2 \right]\\
                			            &\leq \frac{L_A^2 D_{\X}^2}{\beta^2},  
                		\end{align*}
                		where we used the definition of $G_{\beta}$ and $\left\lVert A(\xi)x - \Pi_{\bstoc}(\Astoc x) \right\rVert ^2 \leq \left\lVert A(\xi)x - \Astoc \xopt \right\rVert ^2$
                \item \label{list:technical-obs-g-smoothness} Smoothness constant of $g_{\beta}(\Astoc x)$ and $F_{\beta}(x, \xi)$:
                \begin{align*}
                    \norm{\nabla g_{\beta}(\Astoc x) - \nabla g_{\beta}(\Astoc y)} &= \norm{\frac{A^T(\xi)}{2\beta} \left( \Astoc x - \Pi_{\bstoc}(\Astoc x) \right) - \frac{A^T(\xi)}{2\beta} \left( \Astoc y - \Pi_{\bstoc}(\Astoc y) \right)} \\
                        &\leq \frac{L_A}{2\beta} \norm{x - y} + \frac{\norm{\Astoc}}{2\beta} \norm{\Pi_{\bstoc}(\Astoc y)  -  \Pi_{\bstoc}(\Astoc x) } \\
                        &\leq \frac{L_A}{2\beta} \norm{x - y} + \frac{\norm{\Astoc}}{2\beta} \norm{\Astoc y  -  \Astoc x } \\
                        &\leq \frac{L_A}{\beta} \norm{x - y} 
                \end{align*}
                This implies that $F_{\beta}(x, \xi)$ is $(L_f + \frac{L_A}{\beta})$-smooth.
                
                \item  Properties of $g_{\beta}$ (results from \mbox{\textbf{Lemma 10}~in~\cite{tran2018smooth}}):
                    \begin{enumerate}[itemsep=15pt]
                        \item \label{list:technical-obs-trandinh-relationto-g} $g(z_1) \geq g_{\beta}(z_2) + \dotprod{\nabla g_{\beta}(z_2)}{z_1 - z_2} + \frac{\beta}{2} \normsqr{\lambda^*_{\beta}(z_2)}$
                        \item \label{list:technical-obs-trandinh-relationto-prev-beta} $\g[\beta_k]{\Astoc x_k} \leq \g[\beta_{k - 1}]{\Astoc x_k} + \frac{\beta_{ k-1} - \beta_k}{2} \normsqr{\lambda^*_{\beta_{k}}(\Astoc x_{k})} $
                    \end{enumerate}
        \end{enumerate}
\end{enumerate}

\medskip

\vspace{5mm}
Secondly, we restate \mbox{\textbf{Lemma 3.1}} from \citep{fercoq2019almost} for completeness, as we rely on it for translating the convergence rates from the smoothed gap onto objective suboptimality and feasibility.
\begin{lemma}[Restatement of \textbf{Lemma 3.1} from \citep{fercoq2019almost}]
\label{lemma:smoothed-gap}
\ \newline
Let $(x^*, \lambda^*)$ be a saddle point of \mbox{$\displaystyle \lagr(x, \lambda) \defeq f(x) + \int \langle A(\xi)x, \lambda(\xi)\rangle - \supp_{\bstoc}(\lambda(\xi)) \mu(d\xi)$}, where \mbox{$\supp_{\X}(x) \defeq \sup_{y \in \X}\dotprod{y}{x}$}. Then the following holds:
\begin{enumerate}
    \item $\displaystyle S_{\beta}(x) \geq -\frac{\beta}{2} \normsqr{\lambda^*}$
    \item    $\displaystyle F(x) - F(x^*) \geq  -\frac{1}{4\beta} \int \text{dist}(A(\xi)x, \bstoc)^2 dP(\xi) - \beta \norm{\lambda^*}^2$
    \item    $\displaystyle F(x) - F(x^*) \leq  S_{\beta}(x)$
    \item    $\displaystyle \int \text{dist}(A(\xi)x,  \chi(\xi))^2 dP(\xi) \leq  4\beta^2 \|\lambda^*\|^2 + 4 \beta S_{\beta}(x)$
\end{enumerate}
\end{lemma}

\vspace{5mm}
Finally, we adapt \textbf{Lemma 17}~in ~\cite{mokhtari} for use in our convergence proofs, and provide the proof below.
\begin{lemma}[Adaptation of \textbf{Lemma 17}~in ~\cite{mokhtari}]
\label{lemma: recursion}
\ \newline
Let $0 < \alpha \leq 1$, $1\leq  \beta \leq 2$, $b \geq 0$, $c > 1$, $t_0 \geq 0$. Let $\phi_k$ be a sequence of real numbers satisfying
\begin{equation}
\phi_k \leq (1-\frac{c}{(k+k_0)^{\alpha}})\phi_{k-1} + \frac{b}{(k+k_0)^\beta}.
\end{equation}

Then, the sequence $\phi_k$ converges to zero at the rate

\begin{equation}
\phi_k \leq \frac{Q}{(k+1+k_0)^{\beta-\alpha}},
\end{equation}
when $\alpha = 1$, $1< \beta \leq 2$, or $\alpha = \frac{2}{3}$ , $\beta = 1$, where $Q = \max (\phi_0 (k_0+1)^{\beta-\alpha}, b/(c-1))$.
\end{lemma}

\begin{proof}
We use induction.
By the definition of $Q$, $\phi_0 \leq Q/((k_0+1)^{\beta - \alpha})$, so the base step holds.
Now assume it holds for $k$ and check for $k+1$.
To ease the notation let $y = k+1+k_0$.
When $\alpha = 1$,
\begin{align*}
\phi_{k+1} \leq \left( 1 - \frac{c}{y} \right) \frac{Q}{y^{\beta-1}} + \frac{b}{y^\beta} = \left(1-\frac{c}{y} \right) \frac{Q}{y^{\beta-1}} + \frac{(c-1)Q}{y^\beta}= \frac{Q}{y^{\beta-1}} - \frac{Q}{y^\beta} \leq \frac{Q}{(y+1)^{\beta-1}},
\end{align*}
where the last step follows since $1\leq \beta \leq 2$, i.e. $\frac{y-1}{y^{\beta}} \leq \frac{1}{(y+1)^{\beta-1}} \iff \frac{(y-1)(y+1)^\beta}{(y+1)y^\beta} \leq 1$ and $\frac{(y-1)(y+1)^\beta}{(y+1)y^\beta} \leq \frac{(y-1)(y+1)^2}{(y+1)y^2} \leq 1$, since $\beta \leq 2$.

For general $\alpha, \beta$, we get $\frac{1}{y^{\beta-\alpha}} - \frac{1}{y^\beta} \leq \frac{1}{(y+1)^{\beta-\alpha}} \iff \frac{y^\alpha-1}{y^\beta} \leq \frac{(y+1)^\alpha}{(y+1)^\beta}$. If $\alpha = 2/3, \beta = 1$, then $\frac{y^{2/3}-1}{y} \leq \frac{(y+1)^{2/3}}{(y+1)} \iff \frac{(y^{2/3}-1)(y+1)^{1/3}}{y} \leq 1 \iff \frac{(y^{2/3}-1)^3(y+1)}{y^3} \leq 1 \iff \frac{(y^2 - 3y^{4/3} + 3y^{2/3}-1)(y+1)}{y^3} \leq 1 \iff \frac{(y^3 + y^2 - 3y^{7/3} - 3y^{4/3} + 3y^{5/3} + 3y^{2/3}-y-1}{y^3} \leq 1$ which holds for $y\geq1$.\qed
 
\end{proof}

\subsection{ANALYSIS OF H-1SFW}
This section provides the omitted proofs of Section~\ref{subsec:conv-mokhtari} in the main text. We start with a supporting lemma, needed for the proof of Lemma~\ref{lem:mokhtari}.

\vspace{5mm}

\begin{lemma}
\label{lem: mokh}
Let $d_k = (1-\rho_k)d_{k-1} + \rho_k \nabla F_{\beta_k}(x_k, \xi_k), $ $\rho_k \in [0, 1]$. Then, for all $k$,
\begin{align}\label{eq: var_bound_lemma}
\hspace{-5mm}\mathbb{E}_k \left[ \| \nabla F_{\beta_k}(x_k) - d_k \|^2 \right] &\leq  (1-\frac{\rho_k}{2}) \| \nabla F_{\beta_{k-1}}(x_{k-1}) - d_{k-1} \|^2 + 2 \rho_k^2 \left( \sigma_f^2 + \frac{L_A^2 D_{\X}^2}{\beta_k^2} \right)  \nn \\[2mm]
&\hspace{2mm} + \frac{2}{\rho_k} \Bigg[ 2L_f^2\gamma_{k-1}^2 \diam^2 + 2 L_A^2 \diam^2  \Bigg[ \left( \frac{1}{\beta_k} - \frac{1}{\beta_{k-1}} \right)^2 + \frac{4\gamma_{k-1}}{\beta_{k-1}} \left\vert \frac{1}{\beta_k} - \frac{1}{\beta_{k-1}} \right\vert +  \frac{4\gamma_{k-1}^2}{\beta_{k-1}^2} \Bigg] \Bigg], \notag\\
\end{align}
where $\mathbb{E}_k[\cdot] = \mathbb{E}[\cdot | \mathcalorigin{F}_k]$ and $\mathcalorigin{F}_k$ is a $\sigma$-algebra measuring all sources of randomness up to step $k$.
\end{lemma}

\begin{proof}
We use the definition $d_k = (1-\rho_k)d_{k-1} + \rho_k \nabla F_{\beta_k}(x_k, \xi_k) $ to write the difference
\begin{align}
\| \nabla F_{\beta_k}(x_k) &- d_k \|^2 \nn \\[2mm]
&= \| \nabla F_{\beta_k}(x_k) - (1-\rho_k)d_{k-1} - \rho_k \nabla F_{\beta_k}(x_k, \xi_k) \|^2 \notag\\[2mm]
&=\| \nabla F_{\beta_k}(x_k) + (1-\rho_k)\nabla F_{\beta_{k-1}}(x_{k-1}) - (1-\rho_k)\nabla F_{\beta_{k-1}}(x_{k-1}) - (1-\rho_k)d_{k-1} - \rho_k \nabla F_{\beta_k}(x_k, \xi_k) \|^2 \notag\\[2mm]
&=\| \rho_k (\nabla F_{\beta_k}(x_k) - \nabla F_{\beta_k}(x_k, \xi_k) ) + (1-\rho_k)( \nabla F_{\beta_k}(x_k) - \nabla F_{\beta_{k-1}}(x_{k-1})) \notag \\[2mm]
&+ (1-\rho_k)(\nabla F_{\beta_{k-1}}(x_{k-1}) - d_{k-1}) \|^2 \notag\\[2mm]
&= \rho_k^2 \| \nabla F_{\beta_k}(x_k) - \nabla F_{\beta_k}(x_k, \xi_k) \|^2 + (1-\rho_k)^2 \| \nabla F_{\beta_k}(x_k) - \nabla F_{\beta_{k-1}}(x_{k-1}) \|^2 \notag \\[2mm]
&+ (1-\rho_k)^2 \| \nabla F_{\beta_{k-1}}(x_{k-1}) - d_{k-1} \|^2 \notag \\[2mm] 
&+2\rho_k (1-\rho_k) \langle \nabla F_{\beta_k}(x_k) - \nabla F_{\beta_k}(x_k, \xi_k), \nabla F_{\beta_k}(x_k) - \nabla F_{\beta_{k-1}}(x_{k-1}) \rangle \notag \\[2mm]
&+2\rho_k (1-\rho_k) \langle \nabla F_{\beta_k}(x_k) - \nabla F_{\beta_k}(x_k, \xi_k), \nabla F_{\beta_{k-1}}(x_{k-1}) - d_{k-1} \rangle \notag\\[2mm]
&+2(1-\rho_k)^2 \langle \nabla F_{\beta_k}(x_k) - \nabla F_{\beta_{k-1}}(x_{k-1}), \nabla F_{\beta_{k-1}}(x_{k-1}) - d_{k-1} \rangle \nn
\end{align}

\vspace{5mm}
We remark that $\mathbb{E}_k \left[ \nabla F_{\beta_k}(x_k, \xi_k)\right] = \nabla F_{\beta_k}(x_k)$ so that first two linear terms are $0$. We now take expectations conditioned on $\mathcalorigin{F}_k$,
\begin{align}\label{eq: var1}
\mathbb{E}_k \left[ \| \nabla F_{\beta_k}(x_k) - d_k \|^2 \right]& \nn\\[2mm]
&\hspace{-15mm}= \rho_k^2 \mathbb{E}_k \left[  \| \nabla F_{\beta_k}(x_k) - \nabla F_{\beta_k}(x_k, \xi_k) \|^2 \right] + (1-\rho_k)^2  \| \nabla F_{\beta_k}(x_k) - \nabla F_{\beta_{k-1}}(x_{k-1}) \|^2 \notag \\[2mm]
&\hspace{-15mm}+ (1-\rho_k)^2 \| \nabla F_{\beta_{k-1}}(x_{k-1}) - d_{k-1} \|^2 +2(1-\rho_k)^2 \langle \nabla F_{\beta_k}(x_k) - \nabla F_{\beta_{k-1}}(x_{k-1}), \nabla F_{\beta_{k-1}}(x_{k-1}) - d_{k-1} \rangle 
\end{align}

\vspace{5mm}
Invoking the variance bound from Technical~observation~\ref{list:technical-obs-g-variance} from Section~\ref{app:prelim}, we have:
\begin{align}
    \mathbb{E}_k \left [ \| \nabla F_{\beta_k}(x_k) - \nabla F_{\beta_k}(x_k, \xi_k) \|^2 \right] &\leq 2\mathbb{E}_k \left[ \| \nabla f(x_k) - \nabla f(x_k, \xi_k) \|^2\right] + 2\mathbb{E}_k \left[ \| G_{\beta_k}(Ax_k) - \nabla g_{\beta_k} (A(\xi)x_k, \xi_k)  \|^2\right] \\
                &\leq 2 \left( \sigma_f^2 + \frac{L_A^2 D_{\X}^2}{\beta_k^2} \right)
\end{align}

\vspace{5mm}
For the linear term, we use Young's inequality for some $\sigma_k > 0$ to get
\begin{align}
2(1-\rho_k)^2 \langle \nabla F_{\beta_k}(x_k) - \nabla F_{\beta_{k-1}}(x_{k-1}), \nabla F_{\beta_{k-1}}(x_{k-1}) - d_{k-1} \rangle \notag\\
&\hspace{-70mm}\leq (1-\rho_k)^2 \sigma_k \| \nabla F_{\beta_{k-1}}(x_{k-1}) - d_{k-1} \|^2 + (1-\rho_k)^2 (1/\sigma_k) \| \nabla F_{\beta_k}(x_k) - \nabla F_{\beta_{k-1}}(x_{k-1}) \|^2
\end{align}

\vspace{5mm}

For the $\| \nabla F_{\beta_k}(x_k) - \nabla F_{\beta_{k-1}}(x_{k-1}) \|^2$ term, we use the iterate update rule and Technical~observation~\ref{list:technical-obs-diff-Gs} to get:
\begin{align}
\| \nabla F_{\beta_k}(x_k) - &\nabla F_{\beta_{k-1}}(x_{k-1}) \|^2  \nn \\[2mm]
&= \| \nabla f(x_k) - \nabla f(x_{k-1}) + \nabla G_{\beta_k}(Ax_k) - \nabla G_{\beta_{k-1}}(Ax_{k-1})\|^2 \nn \\[2mm]
&\leq 2 \| \nabla f(x_k) - \nabla f(x_{k-1}) \|^2 + 2 \| \nabla G_{\beta_k}(Ax_k) - \nabla G_{\beta_{k-1}}(Ax_{k}) + \nabla G_{\beta_{k-1}}(Ax_{k}) - \nabla G_{\beta_{k-1}}(Ax_{k-1}) \|^2 \notag\\[2mm]
&\leq 2L_f^2\| x_k - x_{k-1}\|^2 + 2\| \nabla G_{\beta_k}(Ax_k) - \nabla G_{\beta_{k-1}}(Ax_{k}) \|^2 + 2\| \nabla G_{\beta_{k-1}}(Ax_{k}) - \nabla G_{\beta_{k-1}}(Ax_{k-1}) \|^2 \notag \\[2mm]
&\hspace{40mm}+ 4 \| \nabla G_{\beta_k}(Ax_k) - \nabla G_{\beta_{k-1}}(Ax_{k})\| \| \nabla G_{\beta_{k-1}}(Ax_{k}) - \nabla G_{\beta_{k-1}}(Ax_{k-1}) \| \notag \\[2mm]
&\leq 2L_f^2\gamma_{k-1}^2 \diam^2 + 2 L_A^2 \diam^2  \left( \frac{1}{\beta_k} - \frac{1}{\beta_{k-1}} \right)^2 +  \frac{8\gamma_{k-1}^2 L_A^2 \diam^2}{\beta_{k-1}^2} + 8\gamma_{k-1} \left\vert \frac{1}{\beta_k} - \frac{1}{\beta_{k-1}} \right\vert \frac{L_A^{2}\diam^2}{\beta_{k-1}}
\end{align}

\vspace{5mm}
Putting everything back into~\eqref{eq: var1}:
\begin{align}
\mathbb{E}_k \left[ \| \nabla F_{\beta_k}(x_k) - d_k \|^2 \right]& \nn\\[2mm]
            &\hspace{-30mm}= \rho_k^2 \mathbb{E}_k \left[\| \nabla F_{\beta_k}(x_k) - \nabla F_{\beta_k}(x_k, \xi_k) \|^2\right] + (1-\rho_k)^2(1+\sigma_k^{-1}) \| \nabla F_{\beta_k}(x_k) - \nabla F_{\beta_{k-1}}(x_{k-1}) \|^2 \notag \\[2mm]
&\hspace{40mm}+ (1-\rho_k)^2(1+\sigma_k) \| \nabla F_{\beta_{k-1}}(x_{k-1}) - d_{k-1} \|^2 \notag\\[2mm]
            &\hspace{-30mm}\leq (1-\rho_k)^2(1+\sigma_k) \| \nabla F_{\beta_{k-1}}(x_{k-1}) - d_{k-1} \|^2 + 2\rho_k^2 \left( \sigma_f^2 + \frac{L_A^2 \diam^2}{\beta_k^2} \right) \nn\\[2mm]
            &\hspace{-28mm}+ (1-\rho_k)^2(1+\sigma_k^{-1}) \Bigg[ 2L_f^2\gamma_{k-1}^2 \diam^2 + 2 L_A^2 \diam^2  \Bigg[ \left( \frac{1}{\beta_k} - \frac{1}{\beta_{k-1}} \right)^2 + \frac{4\gamma_{k-1}}{\beta_{k-1}} \left\vert \frac{1}{\beta_k} - \frac{1}{\beta_{k-1}} \right\vert +  \frac{4\gamma_{k-1}^2}{\beta_{k-1}^2} \Bigg] \Bigg] .
\end{align}

\vspace{5mm}
Using the facts that $ \rho_k \leq 1$, $(1-\rho_k)^2 \leq (1-\rho_k)$,  $(1-\rho_k)(1+\frac{\rho_k}{2}) \leq (1-\rho_k/2)$, $(1-\rho_k)(1+\frac{2}{\rho_k}) \leq \frac{2}{\rho_k}$ and setting $\sigma_k \defeq \frac{\rho_k}{2}$, we get:
\begin{align}
\mathbb{E}_k \left[ \| \nabla F_{\beta_k}(x_k) - d_k \|^2 \right] &\leq  (1-\frac{\rho_k}{2}) \| \nabla F_{\beta_{k-1}}(x_{k-1}) - d_{k-1} \|^2 + 2 \rho_k^2 \left( \sigma_f^2 + \frac{L_A^2 D_{\X}^2}{\beta_k^2} \right)  \nn \\[2mm]
&\hspace{2mm} + \frac{2}{\rho_k} \Bigg[ 2L_f^2\gamma_{k-1}^2 \diam^2 + 2 L_A^2 \diam^2  \Bigg[ \left( \frac{1}{\beta_k} - \frac{1}{\beta_{k-1}} \right)^2 + \frac{4\gamma_{k-1}}{\beta_{k-1}} \left\vert \frac{1}{\beta_k} - \frac{1}{\beta_{k-1}} \right\vert +  \frac{4\gamma_{k-1}^2}{\beta_{k-1}^2} \Bigg] \Bigg] \notag \qed
\end{align}
\end{proof}

\vspace{10mm}

\begin{replemma}{lem:mokhtari}
Let $\rho_k = \frac{3}{(k+5)^{2/3}}, ~~ \gamma_k = \frac{2}{k + 1}, ~~ \beta_k = \frac{\beta_0}{(k+1)^{1/6}}, \, \beta_0 > 0$ in Algorithm~\ref{alg:homotopy-mokhtari}. Then, for all $k$,
\begin{equation}
\mathbb{E} \left[ \| \nabla F_{\beta_k}(x_k) - d_k \|^2 \right] \leq \frac{C_1}{(k + 5)^{1/3}},
\end{equation}
where $C_1=\max\left( 6^{1/3}\|\nabla F_{\beta_0}(x_0)-d_0\|^2, 2\left[  18\sigma_f^2 +  112L_f^2\diam^2 +  \frac{522L_A^2\diam^2}{\beta_0^2} \right]\right)$.
\end{replemma}

\begin{proof}

We apply the expectation with respect to the whole history to~\eqref{eq: var_bound_lemma} and estimate the rate of $\left\vert \frac{1}{\beta_k} - \frac{1}{\beta_{k-1}} \right\vert$:
\begin{align}
\mathbb{E} \left[ \| \nabla F_{\beta_k}(x_k) - d_k \|^2\right] &\leq (1-\frac{\rho_k}{2}) \mathbb{E} \left[\| \nabla F_{\beta_{k-1}}(x_{k-1}) - d_{k-1} \|^2\right] + 2 \rho_k^2 \left( \sigma_f^2 + \frac{L_A^2 \diam^2}{\beta_k^2} \right) \nn \\
&+ \frac{2}{\rho_k} \Bigg[ 2L_f^2\gamma_{k-1}^2 \diam^2 + 2 L_A^2 \diam^2  \Bigg[ \left( \frac{1}{\beta_k} - \frac{1}{\beta_{k-1}} \right)^2 + \frac{4\gamma_{k-1}}{\beta_{k-1}} \left\vert \frac{1}{\beta_k} - \frac{1}{\beta_{k-1}} \right\vert +  \frac{4\gamma_{k-1}^2}{\beta_{k-1}^2} \Bigg] \Bigg]  \notag 
\end{align}

\begin{align}
0 \leq \frac{1}{\beta_k} - \frac{1}{\beta_{k-1}} &= \frac{(k+1)^{1/6} - (k)^{1/6}}{\beta_0} \notag\\
&= \frac{1}{\beta_0 \left[ (k+1)^{5/6} + (k+1)^{4/6}k^{1/6} + (k+1)^{3/6}k^{2/6} +(k+1)^{2/6}k^{3/6} + (k+1)^{1/6}k^{4/6}  + k^{5/6} \right]} \notag\\
&\leq \frac{1}{6\beta_0k^{5/6}} \nn
\end{align}

\vspace{5mm}

Replacing the parameter rates we further get:
\begin{align}
\mathbb{E} \left[\| \nabla F_{\beta_k} (x_k) - d_k \|^2 \right] & \nn \\[2mm]
        &\hspace{-30mm}\leq \left(1-\frac{3}{2(k+5)^{2/3}}\right) \mathbb{E} \left[ \| \nabla F_{\beta_{k-1}}(x_{k-1}) - d_{k-1} \|^2\right] + \frac{18}{(k+5)^{4/3}} \left( \sigma_f^2 + \frac{L_A^2\diam^2 (k+1)^{2/6}}{\beta_0^2}\right) \nn \\[3mm]
            &\hspace{30mm}+ \frac{2 (k+5)^{2/3}}{3} \left[ \frac{8L_f^2\diam^2 }{k^2} + \frac{2L_A^2\diam^2}{\beta_0^2} \left( \frac{1}{36 k^{10/6}} + \frac{4}{3k^{10/6}} + \frac{16}{k^{10/6}} \right)\right] \nn \\[3mm]
            &\hspace{-30mm}\leq \left( 1-\frac{3}{2(k+5)^{2/3}} \right) \mathbb{E}\left[\| \nabla F_{\beta_{k-1}}(x_{k-1}) - d_{k-1} \|^2\right]  + \frac{18\sigma_f^2}{k + 5} + \frac{18L_A^2\diam^2}{\beta_0^2(k + 5)} + \frac{2 (k+5)^{2/3}}{3k^{10/6}}\left( 8L_f^2\diam^2 +  \frac{36L_A^2\diam^2}{\beta_0^2}\right) \nn \\[3mm]
            &\hspace{-30mm}\leq \left( 1-\frac{3}{2(k+5)^{2/3}} \right) \mathbb{E}\left[\| \nabla F_{\beta_{k-1}}(x_{k-1}) - d_{k-1} \|^2 \right] + \frac{18\sigma_f^2}{k + 5} + \frac{18L_A^2\diam^2}{\beta_0^2(k + 5)} + \frac{14}{k + 5}\left( 8L_f^2\diam^2 +  \frac{36L_A^2\diam^2}{\beta_0^2}\right) \label{eq:upper-bd-with-same-order} \\[3mm]
            &\hspace{-30mm}= \left( 1-\frac{3}{2(k+5)^{2/3}} \right) \mathbb{E}\left[\| \nabla F_{\beta_{k-1}}(x_{k-1}) - d_{k-1} \|^2\right] + \frac{1}{k + 5} \left(18\sigma_f^2 +  112L_f^2\diam^2 +  \frac{522L_A^2\diam^2}{\beta_0^2} \right) \nn
\end{align}
where line~\eqref{eq:upper-bd-with-same-order} follows form the fact that 
\begin{equation*}
\frac{(k + 5)^{2/3}}{k^{10/6}} = \frac{(k + 5)^{4/6}}{k^{10/6}} \frac{(k+5)^{6/6}}{(k+5)^{6/6}} = \left( 1 + \frac{5}{k}\right)^{ {4/6} + {6/6}} \frac{1}{(k+5)^{6/6}} = \left(1 + \frac{5}{k} \right)^{5/3} \frac{1}{k+5} < \frac{6^{5/3}}{k+5} < \frac{21}{k+5}
\end{equation*}

\vspace{5mm}
We can now invoke Lemma~\ref{lem: mokh} for $b = 18\sigma_f^2 +  112L_f^2\diam^2 +  \frac{522L_A^2\diam^2}{\beta_0^2} $ and $c = \frac{3}{2}$, $\alpha = \frac{2}{3}$  and $\beta = 1$, $k_0=5$ to conclude the result.\qed
\end{proof}

\vspace{10mm}

\begin{reptheorem}{thm:mokhtari}
Consider Algorithm~\ref{alg:homotopy-mokhtari} with parameters $\rho_k = \frac{3}{(k+5)^{2/3}}, ~~ \gamma_k = \frac{2}{k + 1}, ~~ \beta_k = \frac{\beta_0}{(k+1)^{1/6}}, \, \beta_0 > 0$ ( the same as Lemma~\ref{lem:mokhtari}). Then, for all $k$,
\begin{equation}
\mathbb{E}\left[ S_{\beta_{k}}(x_{k+1}) \right] \leq \frac{C_2}{k^{1/6}},
\end{equation}
where $C_2 = \max \left\{ S_0(x_1), \;b=2 \diam\sqrt{C_1} + 2\diam^2\left( L_f + \frac{L_A}{\beta_0} \right)\right\} $, and $C_1$ is defined in Lemma~\ref{lem:mokhtari}.
\end{reptheorem}

\begin{proof}

We essentially follow the steps for proving \textbf{Theorem 9} of \citep{locatello2019stochastic}, modified to suit our setting. Using Technical~observation~\ref{list:technical-obs-g-smoothness} and the definition of $\diam$: 
\begin{align}
    F_{\beta_k}(x_{k+1}) &= \mathbb{E}_{k+1} \left[ F_{\beta_k} (x_{k+1}, \xi)\right] \nn \\
    &\leq \mathbb{E}_{k+1} \left[ F_{\beta_k} (x_{k}, \xi) + \langle \nabla F_{\beta_k} (x_{k}, \xi), x_{k+1} - x_{k} \rangle + \frac{1}{2}\left(L_f + \frac{L_A}{\beta_k}\right)\normsqr{x_{k+1} - x_k} \right] \nn \\
    &\leq  F_{\beta_k}(x_{k}) + \gamma_k \langle \nabla F_{\beta_k}(x_k), w_k - x_k\rangle + \frac{\gamma_k^2}{2}\left( L_f + \frac{L_A}{\beta_k} \right)\diam^2 \label{eq:mokh-original} 
\end{align}

\vspace{5mm}
We treat the term $\langle \nabla F_{\beta_k}(x_k), w_k - x_k\rangle$ separately, using the fact that $w_k \in \underset{x}{\arg\min}\langle d_k, y\rangle$ and the definition of $\diam$:
\begin{align}
    \langle \nabla F_{\beta_k}(x_k), w_k - x_k\rangle  & = \langle \nabla F_{\beta_k}(x_k) - d_k, w_k - x_k\rangle + \langle d_k,  w_k - x_k\rangle \nn \\
        &= \langle \nabla F_{\beta_k}(x_k) - d_k, w_k - \xopt\rangle + \langle \nabla F_{\beta_k}(x_k) - d_k, \xopt - x_k\rangle + \langle d_k,  w_k - x_k\rangle  \nn\\
        &\leq \langle \nabla F_{\beta_k}(x_k) - d_k, w_k - \xopt\rangle + \langle \nabla F_{\beta_k}(x_k) - d_k, \xopt - x_k\rangle + \langle d_k,  \xopt - x_k\rangle \nn \\
        &=  \langle \nabla F_{\beta_k}(x_k) - d_k, w_k - \xopt\rangle + \langle \nabla F_{\beta_k}(x_k), \xopt - x_k\rangle \nn \\
        &\leq \| \nabla F_{\beta_k}(x_k) - d_k\| \|w_k - \xopt\| + \langle \nabla F_{\beta_k}(x_k), \xopt - x_k\rangle \nn \\
        &\leq \| \nabla F_{\beta_k}(x_k) - d_k\| \diam + \langle \nabla F_{\beta_k}(x_k), \xopt - x_k\rangle \nn \\
        &= \| \nabla F_{\beta_k}(x_k) - d_k\| \diam + \langle \nabla f(x_k) + \nabla_x G_{\beta_k}(Ax_k), \xopt - x_k\rangle \label{eq:mokh-cvx}
\end{align}

\vspace{5mm}

Using Technical~observation~\ref{list:technical-obs-trandinh-relationto-g} we observe that 
\begin{align*}
    \langle \nabla_x G_{\beta_k}(Ax_k), \xopt - x_k\rangle &=     \mathbb{E}_k \left[ \langle \nabla_x g_{\beta_k}(\Astoc x_k), \xopt - x_k\rangle \right] \\
                                        &= \mathbb{E}_k \left[ \langle \nabla g_{\beta_k}(\Astoc x_k), \Astoc \xopt - \Astoc x_k\rangle \right] \\
                                        &\leq \mathbb{E}_k \left[ g(\Astoc \xopt) - g_{\beta_k}(\Astoc x_k)  - \frac{\beta_k}{2} \normsqr{\lambda^*_{\beta_k}(\Astoc x_k)}\right]\\
                                        &= G(A\xopt) - G_{\beta_k}(A x_k) - \frac{\beta_k}{2} \mathbb{E}_k \left[\normsqr{\lambda^*_{\beta_k}(\Astoc x_k)}\right]
\end{align*}

\vspace{5mm}

Using the above and the convexity of $f$, we obtain:
\begin{align*}
    \langle \nabla F_{\beta_k}(x_k), w_k - x_k\rangle  &\leq \| \nabla F_{\beta_k}(x_k) - d_k\| \diam + \fopt + G(A\xopt) \underbrace{- f(x_k) - G_{\beta_k}(A x_k)}_{= - F_{\beta_k}(x_k)} - \frac{\beta_k}{2} \mathbb{E}_k \left[\normsqr{\lambda^*_{\beta_k}(\Astoc x_k)}\right]
\end{align*}

\vspace{5mm}

Substituting everything back into Equation~\eqref{eq:mokh-original} and noting that $G(A\xopt) = 0$:
\begin{align*}
        F_{\beta_k}(x_{k+1}) &\leq (1-\gamma_k)F_{\beta_k}(x_{k}) + \gamma_k \| \nabla F_{\beta_k}(x_k) - d_k\| \diam + \gamma_k\fopt - \frac{\gamma_k\beta_k}{2} \mathbb{E}_k \left[\normsqr{\lambda^*_{\beta_k}(\Astoc x_k)}\right]  + \frac{\gamma_k^2}{2}\left( L_f + \frac{L_A}{\beta_k}\right) \diam^2.
\end{align*}
\vspace{5mm}

Using Technical~observation~\ref{list:technical-obs-trandinh-relationto-prev-beta} we observe that 
\begin{align*}
    \F{\beta_{ k}}{x_{ k}} &= \mathbb{E}_k \left[ f(x_{ k}, \xi) + \g[\beta_{ k}]{\Astoc x_{k}} \right] \\
        &\leq \mathbb{E}_k \left[ f(x_{ k},\xi) + \g[\beta_{k - 1}]{\Astoc x_{ k}} + \frac{\beta_{ k-1} - \beta_{k}}{2} \normsqr{\lambda^*_{\beta_{k}}(\Astoc x_{k})} \right] \\
        &= \F{\beta_{ k-1}}{x_{ k}} + \mathbb{E}_k \left[ \frac{\beta_{ k-1} - \beta_{k}}{2} \normsqr{\lambda^*_{\beta_{k}}(\Astoc x_{k})} \right] \\
\end{align*}

\vspace{5mm}

Substituting the above, we obtain:
\begin{align}
F_{\beta_k}(x_{k+1}) &\leq (1-\gamma_k)\F{\beta_{ k-1}}{x_{ k}} + \gamma_k \| \nabla F_{\beta_k}(x_k) - d_k\| \diam + \gamma_k \fopt   \nn \\
                            &\hspace{10mm} + \frac{(1-\gamma_k)(\beta_{ k-1} - \beta_{k}) - \gamma_k \beta_k}{2} \mathbb{E}_k \left[\normsqr{\lambda^*_{\beta_k}(\Astoc x_k)}\right] + \frac{\gamma_k^2}{2}\left(L_f + \frac{L_A}{\beta_k}\right) \diam^2 \nn \\[2mm]
                            &\leq (1-\gamma_k)\F{\beta_{ k-1}}{x_{ k}} + \gamma_k \| \nabla F_{\beta_k}(x_k) - d_k\| \diam + \gamma_k \fopt  + \frac{\gamma_k^2}{2}\left(L_f + \frac{L_A}{\beta_k}\right) \diam^2 \label{eq:mokh-prefinal},
\end{align}
where the last line comes from the fact that $(1-\gamma_k)(\beta_{k-1}-\beta_k) -\gamma_k\beta_k < 0$: 
\begin{align}
(1-\gamma_k)(\beta_{k-1}-\beta_k) -\gamma_k\beta_k &= \beta_{k-1}-\beta_k -\gamma_k \beta_{k-1} = \frac{\beta_0}{k^{1/6}} - \frac{\beta_0}{(k+1)^{1/6}} - \frac{ 2\beta_0}{(k + 1)k^{1/6}}\notag\\
&=\frac{\beta_0}{k^{1/6}} \left( 1 - \frac{k^{1/6}}{(k+1)^{1/6}} - \frac{2}{k+1}\right) \notag\\
&=\frac{\beta_0}{k^{1/6}} \left( \frac{k-1}{k+1}  - \frac{k^{1/6}}{(k+1)^{1/6} }\right) \notag\\
&< \frac{\beta_0}{k^{1/6}} \left( \underbrace{\frac{k}{k+1}}_{\in (0, 1)}  - \frac{k^{1/6}}{(k+1)^{1/6}} \right)   \nn \\
&< 0.\nn
\end{align}

\vspace{5mm}
Starting from Equation~\eqref{eq:mokh-prefinal} and subtracting $\fopt$ from both sides, noting the definition of $S_{\beta_{k}}(x) \defeq F_{\beta}(x) - \fopt$ and taking the expectation on both sides:
\begin{align}\label{eq: th_pf_1}
\mathbb{E}\left[S_{\beta_{k}}(x_{k+1})\right] \leq (1-\gamma_k)\mathbb{E} \left[ S_{\beta_{k-1}}(x_{k})\right] + \frac{\gamma_k^2}{2} D_\mathcalorigin{X}^2 \left(L_f + \frac{L_A}{\beta_k}\right) + \gamma_k \mathbb{E} \left[ \| \nabla F_{\beta_k}(x_k) - d_k \| \right] \diam.
\end{align}

\vspace{5mm}
Replacing the parameter rates for the second term, we bound by

\begin{align*}
\frac{\gamma_k^2}{2} \diam^2 (L_f + \frac{L_A}{\beta_k}) &= \frac{2\diam^2 L_f}{k^2} + \frac{2\diam^2L_A}{\beta_0k^{11/6}} \\
        &\leq \frac{2\diam^2}{k^{7/6}}\left( L_f + \frac{L_A}{\beta_0}\right)
\end{align*}

For the last term we use the parameter rates and Lemma~\ref{lem:mokhtari} together with Jensen's inequality $\mathbb{E} \left[ \| \nabla F_{\beta_k}(x_k) - d_k \| \right] = \sqrt{\mathbb{E} \left[ \| \nabla F_{\beta_k}(x_k) - d_k \| \right]^2} \leq \sqrt{\mathbb{E} \left[ \| \nabla F_{\beta_k}(x_k) - d_k \|^2 \right]}$ to get

\begin{align*}
\gamma_k \diam \mathbb{E} \left[ \| \nabla F_{\beta_k}(x_k) - d_k \| \right]  &= \frac{2\diam}{k+1}  \frac{\sqrt{C_1}}{(k + 5)^{1/6}} \\
    &\leq \frac{2\diam\sqrt{C_1}}{k^{7/6}},
\end{align*}

Substituting the above into~\eqref{eq: th_pf_1}, we get 
\begin{equation*}
\mathbb{E}\left[ S_{\beta_{k}}(x_{k+1}) \right] \leq \left( 1-\frac{2}{k}\right)\mathbb{E} \left[ S_{\beta_{k-1}}(x_{k})\right] + \frac{2 \diam\sqrt{C_1} + 2\diam^2\left( L_f + \frac{L_A}{\beta_0} \right)}{k^{7/6}}.
\end{equation*}

\vspace{5mm}
Finally, we use Lemma~\ref{lemma: recursion} with $\alpha = 1$, $\beta=7/6$, $c=2$, $b=2 \diam\sqrt{C_1} + 2\diam^2\left( L_f + \frac{L_A}{\beta_0} \right)$ to arrive at the statement. \qed
\end{proof}

\vspace{5mm}

\begin{repcorollary}{corollary:mokhtari-residual-and-feasibility}
The expected convergence in terms of objective suboptimality and feasibility is, respectively 
\begin{equation*}
\mathbb{E}\left[ \lVert f(x_k, \xi) - \fopt \rVert \right] \in \bigO{k^{-1/6}}, \;\;\; \sqrt{\mathbb{E} \left[ \dist(\Astoc x_k, \bstoc)^2 \right] } \in \bigO{k^{-1/6}}.
\end{equation*}

Consequently, the oracle complexity of Algorithm~\ref{alg:homotopy-mokhtari} is $\sfo \in \bigO{\epsilon^{-6}}$ and $\lmo \in \bigO{\epsilon^{-6}}$.
\end{repcorollary}

\begin{proof}
The stated result comes from applying Lemma~\ref{lemma:smoothed-gap} in conjunction with the convergence smoothed-gap rate obtained in Theorem~\ref{thm:mokhtari}. Considering that at every iteration we take one stochastic sample and compute one lmo, along with the $\bigO{k^{-1/6}}$ convergence rate, we obtain the stated oracle complexities. \qed
\end{proof}


\vspace{10mm}
\subsection{ANALYSIS OF H-SPIDER-FW}
This section provides the omitted proofs of Section~\ref{subsec:conv-spiderfw} in the main text. We start with a supporting lemma, needed for the proof of Lemma~\ref{lem:spiderfw-finite-sum-var} and Lemma~\ref{lem:spiderfw-expectation-var}.

\vspace{5mm}

\begin{lemma}
\label{lemma: variance-bound}
Let $v_{t, k} = v_{t, k-1} - \tilde{\nabla} F_{\beta_{t, k-1}}(x_{t, k-1}, \xi_{\mathcal{S}_{t,k}})
+ \tilde{\nabla} F_{\beta_{t, k}}(x_{t, k}, \xi_{\mathcal{S}_{t,k}})$, with $\lvert \mathcalorigin{S}_{t, k} \rvert = K_t = 2^{t-1}$ and  $v_{t, 1} = \tilde{\nabla} F_{\beta_{t, 1}}(x_{t, 1}, \xi_{\mathcalorigin{Q}_{t}})$. Also, let $\gamma_{t, k} = \frac{2}{K_t + k}, ~\beta_{t, k} = \frac{\beta_0}{\sqrt{K_t + k}}$. Then, for a fixed $t$ and for all $k \leq K_t$,
    \begin{equation}
        \mathbb{E}_{t, 1}\left[ \normsqr{\gradF{\beta_{t, k}}{x_{t, k}} - v_{t, k}}\right] \leq \frac{2\diam^2}{K_t + k} \left(8L_f^2 + \frac{98L_A^2}{\beta_0^2}\right) +  \mathbb{E}_{t, 1}\left[ \normsqr{\gradF{\beta_1}{x_1} - v_1} \right] 
    \end{equation}
\end{lemma}

\begin{proof}

\begin{align*}
    \normsqr{\gradF{\beta_{t, k}}{x_{t, k}} - v_{t, k}} &= \normsqr{\gradF{\beta_{t, k}}{x_{t, k}} - v_{t, k-1} - \gradFstoc{\beta_{t, k}}{x_{t, k}} + \gradFstoc{\beta_{t, k-1}}{x_{t, k-1}}}  \\[2mm]
    &= \normsqr{\gradF{\beta_{t, k}}{x_{t, k}}  - \gradF{\beta_{t, k-1}}{x_{t, k-1}} +\gradF{\beta_{t, k-1}}{x_{t, k-1}} - v_{t, k-1} \\
            &\hspace{35mm}     - \gradFstoc{\beta_{t, k}}{x_{t, k}} + \gradFstoc{\beta_{t, k-1}}{x_{t, k-1}}} \\[2mm]
    &= \normsqr{\gradF{\beta_{t, k-1}}{x_{t, k-1}} - v_{t, k-1}}  \\
    & \hspace{5mm} + \normsqr{\gradF{\beta_{t, k}}{x_{t, k}}  - \gradF{\beta_{t, k-1}}{x_{t, k-1}} - \gradFstoc{\beta_{t, k}}{x_{t, k}} + \gradFstoc{\beta_{t, k-1}}{x_{t, k-1}}} \\
    & \hspace{5mm} + 2\dotprod{\gradF{\beta{k-1}}{x_{t, k-1}} - v_{t, k-1}}{\gradF{\beta_{t, k}}{x_{t, k}}  - \gradF{\beta{k-1}}{x_{t, k-1}} \\
        &\hspace{5mm} - \gradFstoc{\beta_{t, k}}{x_{t, k}} + \gradFstoc{\beta_{t, k-1}}{x_{t, k-1}}} \\   
\end{align*}

\vspace{5mm}
We now take the expectation on both sides $\Expk{X} = \Exp{X}{\mathcalorigin{F}_{t, k}}$ conditioned on all randomness up to step $(t, k)$ (i.e. the expectations are taken solely with regards to $\xi_{\mathcal{S}_{t,k}}$).   
\begin{align}
    \Expk{\normsqr{\gradF{\beta_{t, k}}{x_{t, k}} - v_{t, k}}} & \nn \\
        &\hspace{-42mm} = \normsqr{\gradF{\beta_{t, k-1}}{x_{t, k-1}} - v_{t, k-1}} + \mathbb{E}_{t,k}\left[\normsqr{\gradF{\beta_{t, k}}{x_{t, k}}  - \gradF{\beta_{t, k-1}}{x_{t, k-1}} - \gradFstoc{\beta_{t, k}}{x_{t, k}} + \gradFstoc{\beta_{t, k-1}}{x_{t, k-1}}} \right]\nn  \\
        & \hspace{-35mm} + 2\dotprod{\gradF{\beta_{t, k-1}}{x_{t, k-1}} - v_{t, k-1}}{ \underbrace{ \Expk{\gradF{\beta_{t, k}}{x_{t, k}}  - \gradF{\beta_{t, k-1}}{x_{t, k-1}} - \gradFstoc{\beta_{t, k}}{x_{t, k}} + \gradFstoc{\beta_{t, k-1}}{x_{t, k-1}}}}_{ = 0, \text{ since } \gradF{\beta}{x} = \mathbb{E} \left[ \tilde{\nabla} F(x, \xi_{\mathcal{S}_{t,k}})\right]}}   \nn \\[2mm]
    &\hspace{-42mm}=  \normsqr{\gradF{\beta_{t, k-1}}{x_{t, k-1}} - v_{t, k-1}} + \underbrace{\mathbb{E}_{t,k}\left[\normsqr{\gradF{\beta_{t, k}}{x_{t, k}}  - \gradF{\beta_{t, k-1}}{x_{t, k-1}} - \gradFstoc{\beta_{t, k}}{x_{t, k}} + \gradFstoc{\beta_{t, k-1}}{x_{t, k-1}}} \right]}_{=T}\label{eq: unbiasedness-grad-F1}
\end{align}

\vspace{5mm}
We now bound $T$:
\begin{align}
    T &= \Expk{\normsqr{\gradF{\beta_{t, k}}{x_{t, k}}  - \gradF{\beta_{t, k-1}}{x_{t, k-1}} - \gradFstoc{\beta_{t, k}}{x_{t, k}} + \gradFstoc{\beta_{t, k-1}}{x_{t, k-1}}} } \nn \\[3mm]
    &= \Expk {\normsqr{ \frac{1}{K_t} \sum_{i = 1}^{K_t} \gradF{\beta_{t, k}}{x_{t, k}}  - \gradF{\beta_{t, k-1}}{x_{t, k-1}} - \nabla F_{\beta_{t, k}}(x_{t, k}, \xi_i) + \nabla F_{\beta_{t, k-1}}(x_{t, k-1}, \xi_i)} } \label{eq:avggrad} \\[3mm] 
    &= \frac{1}{K^2_t} \Expk { \sum_{i = 1}^{K_t} \normsqr{\gradF{\beta_{t, k}}{x_{t, k}}  - \gradF{\beta_{t, k-1}}{x_{t, k-1}} - \nabla F_{\beta_{t, k}}(x_{t, k}, \xi_i) + \nabla F_{\beta_{t, k-1}}(x_{t, k-1}, \xi_i)}} \nn \\ 
            &\hspace{10mm }+ \frac{2}{K^2_t} \mathbb{E}_{t,k}\big[ \sum_{\substack{i,j < K_t \\ i < j}} \dotprod{\gradF{\beta_{t, k}}{x_{t, k}}  - \gradF{\beta_{t, k-1}}{x_{t, k-1}} - \nabla F_{\beta_{t, k}}(x_{t, k}, \xi_i) + \nabla F_{\beta_{t, k-1}}(x_{t, k-1}, \xi_i)}{\nn \\
            &\hspace{40mm} \gradF{\beta_{t, k}}{x_{t, k}}  - \gradF{\beta_{t, k-1}}{x_{t, k-1}} - \nabla F_{\beta_{t, k}}(x_{t, k}, \xi_j) + \nabla F_{\beta_{t, k-1}}(x_{t, k-1}, \xi_j)} \big] \label{eq:expanding-sum} \\[3mm]
    &= \frac{1}{K^2_t}  \sum_{i = 1}^{K_t} \Expk {\normsqr{\gradF{\beta_{t, k}}{x_{t, k}}  - \gradF{\beta_{t, k-1}}{x_{t, k-1}} - \nabla F_{\beta_{t, k}}(x_{t, k}, \xi_i) + \nabla F_{\beta_{t, k-1}}(x_{t, k-1}, \xi_i)}} \label{eq:indep-samples} \\[3mm] 
    &= \frac{K_t}{K^2_t} \Expk{\normsqr{\gradF{\beta_{t, k}}{x_{t, k}}  - \gradF{\beta_{t, k-1}}{x_{t, k-1}} - \nabla F_{\beta_{t, k}}(x_{t, k}, \xi) + \nabla F_{\beta_{t, k-1}}(x_{t, k-1}, \xi)}}  \nn \\[3mm]
    &= \frac{1}{K_t} \mathbb{E}_{t, k} \| \nabla f(x_{t, k}) - \nabla f(x_{t, k-1}) - \nabla f(x_{t, k}, \xi) + \nabla f(x_{t, k-1}, \xi)   \nn \\
            &\hspace{20mm }+ \gradG{\beta_{t, k}}{Ax_{t, k}} - \gradG{\beta_{t, k-1}}{Ax_{t, k-1}} - \nabla g_{\beta_{t, k}}(\Astoc x_{t, k}) + \nabla g_{\beta_{t, k-1}}(\Astoc x_{t, k-1}) \|^2 \nn \\[3mm]
    &\leq \underbrace{\frac{2}{K_t} \mathbb{E}_{t, k} \| \nabla f(x_{t, k}) - \nabla f(x_{t, k-1}) - \nabla f(x_{t, k}, \xi) + \nabla f(x_{t, k-1}, \xi) \|^2}_{=T_1} \nn \\
            &\hspace{10mm} + \underbrace{\frac{2}{K_t} \mathbb{E}_{t, k} \|  \gradG{\beta_{t, k}}{Ax_{t, k}} - \gradG{\beta_{t, k-1}}{Ax_{t, k-1}} - \nabla g_{\beta_{t, k}}(\Astoc x_{t, k}) + \nabla g_{\beta_{t, k-1}}(\Astoc x_{t, k-1}) \|^2 }_{= T_2} \nn
\end{align}
Line \eqref{eq:avggrad} comes from the use of an averaged gradient with batch size $K_t$. Line \eqref{eq:expanding-sum} comes from applying the square norm to the inner sum, and linearity of expectation. Line \eqref{eq:indep-samples} comes from plugging the expectation inside the inner product as allowed by the independence of the samples $\Astoci[i]$ and $\Astoci[j]$ (if $X \perp Y$, then $\Expk{XY} = \Expk{X}\Expk{Y}$). This results in each term being zero, due to stochastic gradient unbiasedness.

\vspace{5mm}

We evaluate the terms $T_1$ and $T_2$ separately:
\begin{align}
    T_2 &= \frac{2}{K_t} \mathbb{E}_{t, k} \|  \gradG{\beta_{t, k}}{Ax_{t, k}} - \gradG{\beta_{t, k-1}}{Ax_{t, k-1}} - \nabla g_{\beta_{t, k}}(\Astoc x_{t, k}) + \nabla g_{\beta_{t, k-1}}(\Astoc x_{t, k-1}) \|^2 \nn \\[3mm]
       &=  \frac{2}{K_t} \mathbb{E}_{t, k}\big[ \normsqr{\gradG{\beta_{t, k}}{Ax_{t, k}} - \gradG{\beta_{t, k-1}}{Ax_{t, k-1}}} \nn \\
       &\hspace{15mm } - 2\dotprod{\gradG{\beta_{t, k}}{Ax_{t, k}} - \gradG{\beta_{t, k-1}}{Ax_{t, k-1}}}{\gradg[\beta_{t, k}]{\Astoc x_{t, k}} -  \gradg[\beta_{t, k-1}]{\Astoc x_{t, k-1}}} \nn \\
       &\hspace{15mm} + \normsqr{ \gradg[\beta_{t, k}]{\Astoc x_{t, k}} -  \gradg[\beta_{t, k-1}]{\Astoc x_{t, k-1}} } \big]  \nn\\[3mm]
  &= \frac{2}{K_t} \big( \normsqr{\gradG{\beta_{t, k}}{Ax_{t, k}} - \gradG{\beta_{t, k-1}}{Ax_{t, k-1}}}  \nn \\
          &\hspace{15mm} - 2\dotprod{\gradG{\beta_{t, k}}{Ax_{t, k}} - \gradG{\beta_{t, k-1}}{Ax_{t, k-1}}}{\Expk{\gradg[\beta_{t, k}]{\Astoc x_{t, k}} -  \gradg[\beta_{t, k-1}]{\Astoc x_{t, k-1}}}}  \nn \\
         &\hspace{15mm} + \Expk{\normsqr{ \gradg[\beta_{t, k}]{\Astoc x_{t, k}} -  \gradg[\beta_{t, k-1}]{\Astoc x_{t, k-1}} }} \big) \nn \\[3mm]
    &= \frac{2}{K_t} \big( \Expk{\normsqr{ \gradg[\beta_{t, k}]{\Astoc x_{t, k}} -  \gradg[\beta_{t, k-1}]{\Astoc x_{t, k-1}} }} - \normsqr{\gradG{\beta_{t, k}}{Ax_{t, k}} - \gradG{\beta_{t, k-1}}{Ax_{t, k-1}}} \big) \nn \\[3mm]
    &\leq \frac{2}{K_t} \Expk{\normsqr{ \gradg[\beta_{t, k}]{\Astoc x_{t, k}} - \nabla g_{\beta_{t, k}}(\Astoc x_{t, k-1})  + \nabla g_{\beta_{t, k}}(\Astoc x_{t, k-1}) -  \gradg[\beta_{t, k-1}]{\Astoc x_{t, k-1}}  }}  \nn \\[3mm]
    &= \frac{2}{K_t} \mathbb{E}_{t, k} \Big[ \| \frac{1}{\beta_{t, k}} A^T(\xi)\Astoc \left( x_{t, k} - x_{t, k-1}\right) + \frac{1}{\beta_{t, k}} A^T(\xi) \left[\Pi_{\bstoc} \left( \Astoc x_{t, k-1} \right) - \Pi_{\bstoc} \left( \Astoc x_{t, k} \right) \right] \nn \\
    &\hspace{50mm}+ \left( \frac{1}{\beta_{t, k}} - \frac{1}{\beta_{t, k-1}}\right)A^T(\xi) \left[ \Astoc x_{t, k -1} - \Pi_{\bstoc}\left(\Astoc x_{t, k-1}\right) \right] \|^2 \Big]  \nn \\[3mm]
    &\leq \frac{2}{K_t} \mathbb{E}_{t, k} \Big[ \frac{3L_A^2}{\beta_{t, k}^2} \| x_{t, k} - x_{t, k-1}\|^2 + \frac{3L_A}{\beta_{t, k}^2} \| \Pi_{\bstoc} \left( \Astoc x_{t, k-1} \right) - \Pi_{\bstoc} \left( \Astoc x_{t, k} \right) \|^2 \nn \\
    &\hspace{50mm}+ 3L_A\left( \frac{1}{\beta_{t, k}} - \frac{1}{\beta_{t, k-1}}\right)^2 \|  \Astoc x_{t, k -1} - \Pi_{\bstoc}\left(\Astoc x_{t, k-1}\right) \|^2 \Big]  \nn \\[3mm]
    &\leq \frac{2}{K_t} \mathbb{E}_{t, k} \Big[ \frac{3L_A^2}{\beta_{t, k}^2} \| x_{t, k} - x_{t, k-1}\|^2 + \frac{3L_A^2}{\beta_{t, k}^2} \| x_{t, k-1}  -  x_{t, k}  \|^2 + 3L_A\left( \frac{1}{\beta_{t, k}} - \frac{1}{\beta_{t, k-1}}\right)^2 \|  \Astoc x_{t, k -1} - \Astoc \xopt \|^2 \Big]  \label{eq:nonexpansiveness-and-optimality-of-proj} \\[3mm]
    &\leq \frac{2}{K_t} \Big[ \frac{6L_A^2\gamma_{t, k-1}^2 \diam^2}{\beta_{t, k}^2} + 3L_A^2 \diam^2\left( \frac{1}{\beta_{t, k}} - \frac{1}{\beta_{t, k-1}}\right)^2 \Big]  \label{eq:def-of-diam-and-iterate-update} \\[3mm]
        &\leq \frac{2L_A^2 \diam^2}{\beta_0^2 K_t(K_t + k - 1)} \Big[ \frac{24(K_t + k)}{(K_t + k -1)} + \frac{3}{4}\Big]  \label{eq:param-replacement} \\[3mm]
        &\leq \frac{98L_A^2 \diam^2}{\beta_0^2 K_t(K_t + k - 1)} \label{eq:T2}
\end{align}
where line \eqref{eq:nonexpansiveness-and-optimality-of-proj} comes from the nonexpansiveness of projections and $\|  \Astoc x_{t, k -1} - \Pi_{\bstoc}\left(\Astoc x_{t, k-1}\right) \| \leq \|  \Astoc x_{t, k -1} - y \|, \, \forall y \in \bstoc$, and line~\eqref{eq:def-of-diam-and-iterate-update} comes from the iterate update rule and the definition of $\diam$. Line~\eqref{eq:param-replacement} comes from replacing the parameter rates and the fact that:

\begin{align*}
    0 \leq \frac{1}{\beta_{t, k}} - \frac{1}{\beta_{t, k-1}} &= \frac{1}{\beta_0}\left( \sqrt{K_t + k} - \sqrt{K_t + k - 1}\right) \\
        &= \frac{1}{\beta_0}\left( \frac{1}{\sqrt{K_t + k} + \sqrt{K_t + k - 1}}\right) \\
        &\leq \frac{1}{2\beta_0 \sqrt{K_t + k - 1}}
\end{align*}

\vspace{5mm}
Now we evaluate $T_1$ and use the fact that $\nabla f(x, \xi)$ are $L_f$-Lipschitz:
\begin{align}
    T_1 &= \frac{2}{K_t} \mathbb{E}_{t, k} \left[ \| \nabla f(x_{t, k}) - \nabla f(x_{t, k-1}) - \nabla f(x_{t, k}, \xi) + \nabla f(x_{t, k-1}, \xi) \|^2\right] \nn\\[2mm]
            &= \frac{2}{K_t} \Bigg( \| \nabla f(x_{t, k}) - \nabla f(x_{t, k-1})\|^2 +  \mathbb{E}_{t, k} \left[ \| \nabla f(x_{t, k}, \xi) - \nabla f(x_{t, k-1}, \xi) \|^2\right] \nn \\
            &\hspace{50mm}- 2\langle \nabla f(x_{t, k}) - \nabla f(x_{t, k-1}) , \mathbb{E}_{t, k} \left[ \nabla f(x_{t, k}, \xi) - \nabla f(x_{t, k-1}, \xi)\right] \rangle \Bigg)\nn \\[2mm]
            &\leq \frac{2}{K_t}  \mathbb{E}_{t, k} \left[\| \nabla f(x_{t, k}, \xi) - \nabla f(x_{t, k-1}, \xi) \|^2\right] \nn \\[2mm]
             &\leq \frac{2L_f^2}{K_t} \| x_{t, k} - x_{t, k-1} \|^2\nn \\[2mm]
             &\leq \frac{2L_f^2\gamma_{t, k-1}^2 \diam^2}{K_t} \nn \\[2mm]
             &= \frac{8L_f^2 \diam^2}{K_t(K_t + k - 1)^2} \label{eq:T1}
\end{align}

\vspace{5mm}
Plugging in~\eqref{eq:T1} and~\eqref{eq:T2} into the expression of $T$, we get that 
\begin{align}
    T  &\leq \frac{8L_f^2 \diam^2}{K_t(K_t + k - 1)^2}  + \frac{98L_A^2 \diam^2}{\beta_0^2 K_t(K_t + k - 1)} \nn \\
        &\leq \frac{\diam^2 \left(8L_f^2 + \frac{98L_A^2}{\beta_0^2}\right) }{K_t(K_t + k - 1)} 
\end{align}

\vspace{5mm}

Now we telescope the sum in Equation~\eqref{eq: unbiasedness-grad-F1} and get 
\begin{align}
 \mathbb{E}_{t, 1}\left[ \normsqr{\gradF{\beta_{t, k}}{x_{t, k}} - v_{t, k}} \right] &= \mathbb{E}_{t, 1}\left[ \mathbb{E}_{t, 2}\left[ \ldots \mathbb{E}_{t, k}\left[  \normsqr{\gradF{\beta_{t, k}}{x_{t, k}} - v_{t, k}} \right]\right] \right] \nn \\[3mm]
 &\leq \frac{\diam^2}{K_t} \left(8L_f^2 + \frac{98L_A^2}{\beta_0^2}\right) \sum_{i = 2}^k   \frac{1}{K_t + i - 1} +  \mathbb{E}_{t, 1}\left[ \normsqr{\gradF{\beta_{t,1}}{x_{t,1}} - v_{t, 1}} \right] \nn \\
        &\leq \frac{\diam^2}{K_t} \left(8L_f^2 + \frac{98L_A^2}{\beta_0^2}\right)  \sum_{i = 2}^k   \frac{1}{\frac{K_t + k}{2}} +  \mathbb{E}_{t, 1}\left[ \normsqr{\gradF{\beta_{t,1}}{x_{t,1}} - v_{t, 1}}\right]  \label{eq:upper-bd-k-term} \\
        &= \frac{2\diam^2}{K_t} \left(8L_f^2 + \frac{98L_A^2}{\beta_0^2}\right)  \frac{k - 1}{K_t + k} +  \mathbb{E}_{t, 1}\left[ \normsqr{\gradF{\beta_{t,1}}{x_{t,1}} - v_{t, 1}} \right]  \nn \\
        &\leq \frac{2\diam^2}{K_t + k} \left(8L_f^2 + \frac{98L_A^2}{\beta_0^2}\right) +  \mathbb{E}_{t, 1}\left[ \normsqr{\gradF{\beta_{t,1}}{x_{t,1}} - v_{t, 1}}\right] \label{eq:pre-var-bound}
\end{align}
where line~\eqref{eq:upper-bd-k-term} comes from the fact that 
\begin{align*}
    2 \leq k \leq 2^{t - 1} = K_t  &\implies 2^{t - 2} + 1 \leq \frac{K_t +k}{2}\leq 2^{t-1}\textbf{ and }\\
    2 \leq i \leq k \leq 2^{t - 1} &\implies 2^{t - 1} + 1 \leq K_t + i - 1 \leq 2^t - 1
\end{align*}
and line~\eqref{eq:pre-var-bound} comes from $k-1 \leq K_t$. \qed
\end{proof}

\vspace{10mm}

\begin{replemma}{lem:spiderfw-finite-sum-var}[Estimator variance for finite-sum problems]
Consider Algorithm~\ref{alg:homotopy-spider-fw}, and let $\xi$ be finitely sampled from set $[n]$, $\xi_{\mathcalorigin{Q}_t} = [n]$ and $\xi_{\mathcalorigin{S}_{t, k}}, \text{ such that } |\mathcalorigin{S}_{t, k}| = K_t= 2^{t-1}$. Also, let $\gamma_{t, k} = \frac{2}{K_t + k}, ~\beta_{t, k} = \frac{\beta_0}{\sqrt{K_t + k}}$. Then, for a fixed $t$ and for all $k \leq K_t$,
    \begin{equation}
        \Exp{\normsqr{\gradF{\beta_{t, k}}{x_{t, k}} - v_{t, k}}}{} \leq \frac{C_1}{K_t + k}, 
    \end{equation}
where $C_1 = 2\diam^2 \left(8L_f^2 + \frac{98L_A^2}{\beta_0^2}\right)$.
\end{replemma}

\begin{proof}   

The result directly follows from the fact that we take a full gradient in the outer loop ($\xi_{\mathcalorigin{Q}_t} = [n]$), thus zeroing out the term  $\mathbb{E}_{t, 1}\left[ \normsqr{\gradF{\beta_{t,1}}{x_{t,1}} - v_{t,1}} \right]$ of Lemma~\ref{lemma: variance-bound}. Taking the full expectation on both sides gives us the stated result. \qed
\end{proof}

\vspace{5mm}
\begin{replemma}{lem:spiderfw-expectation-var}[Estimator variance for generic problems]

Consider Algorithm~\ref{alg:homotopy-spider-fw} and let $\xi \sim P(\xi)$ and $\xi_{\mathcalorigin{Q}_t} \text{ such that } \lvert  \mathcalorigin{Q}_t \rvert = \lceil \frac{2K_t}{\beta_{t, 1}^2}\rceil$. Also, let $\xi_{\mathcalorigin{S}_{t, k}}, \text{ such that } |\mathcalorigin{S}_{t, k}| = K_t= 2^{t-1}$, $\gamma_{t, k} = \frac{2}{K_t + k}, ~\beta_{t, k} = \frac{\beta_0}{\sqrt{K_t + k}}$. Then, for a fixed $t$ and for all $k \leq K_t$,
    \begin{equation}
        \Exp{\normsqr{\gradF{\beta_{t, k}}{x_{t, k}} - v_{t, k}}}{} \leq \frac{C_2}{K_t + k}, 
    \end{equation}
where $C_2 =  16L_f^2\diam^2 + 2L_A^2\diam^2 \left( \frac{98}{\beta_0^2} + 1\right) + 2\beta_0^2\sigma_f^2$.
\end{replemma}

\begin{proof}

    From the use of averaged gradient and Technical~observation~\ref{list:technical-obs-g-variance}:
    \begin{align*}    
        \mathbb{E}_{t, 1}\left[ \normsqr{\gradF{\beta_{t,1}}{x_{t,1}} - v_{t,1}} \right] &\leq \frac{1}{\lvert  \mathcalorigin{Q}_t \rvert } \mathbb{E}_{t, 1}\left[ \normsqr{\nabla f(x_{t, 1}) - \nabla f(x_{t, 1}, \xi) + \gradG{\beta_{t, 1}}{Ax_{t,1}} - \gradg[\beta_{t, 1}]{\Astoc x_{t, 1}}} \right] \\[2mm]
            &\leq \frac{1}{\lvert  \mathcalorigin{Q}_t \rvert } \left( 2\mathbb{E}_{t, 1} \left[ \normsqr{\nabla f(x_{t, 1}) - \nabla f(x_{t, 1}, \xi)} \right] + 2 \mathbb{E}_{t, 1} \left[ \normsqr{\gradG{\beta_{t, 1}}{Ax_{t,1}} - \gradg[\beta_{t, 1}]{\Astoc x_{t, 1}}}\right) \right]\\[2mm]
            &\leq \frac{\beta_{t, 1}^2}{2K_t} \left( 2\sigma_f^2 + \frac{2L_A^2 \diam^2}{\beta_{t, 1}^2} \right) \\[2mm]
            &\leq \frac{\beta_0^2}{2K_t(K_t + 1)} \left( 2\sigma_f^2 + \frac{2L_A^2 \diam^2(K_t + 1)}{\beta_0^2 } \right) \\[2mm]
            &\leq \frac{\beta_0^2\sigma_f^2}{K_t^2} + \frac{L_A^2 \diam^2}{K_t } \\[2mm]
            &\leq \frac{1}{K_t + k} \left(2\beta_0^2\sigma_f^2 + 2L_A^2 \diam^2\right)
\end{align*}
Where we have used that $2K_t \geq K_t + k$ and $K_t^2 \geq K_t = \frac{2K_t}{2} \geq \frac{K_t + k}{2}, \; \forall K_t \in \mathbb{N}, K_t \geq 1, \forall k \leq K_t$. Replacing in~\eqref{eq:pre-var-bound}, we obtain the desired result. \qed

\end{proof}

\vspace{5mm}

\begin{reptheorem}{thm:spiderfw}

Consider Algorithm~\ref{alg:homotopy-spider-fw} with parameters $\gamma_{t, k} = \frac{2}{K_t + k}$, $\beta_{t, k} = \frac{\beta0}{\sqrt{K_t+k}}$ and $\xi_{\mathcalorigin{S}_{t, k}}, \text{ such that } |\mathcalorigin{S}_{t, k}| = K_t= 2^{t-1}$. Then,
\begin{itemize}[itemsep=15pt]
\item For $\xi$ be finitely sampled from set $[n]$ and $\xi_{\mathcalorigin{Q}_t} = [n]$,
                        \begin{equation*}
                                \Exp{S_{\beta_{t, k}}(x_{t, k + 1})}{}\leq \frac{C_3}{\sqrt{K_t + k + 1}}, \;\; \forall t \in \mathbb{N}, \, 1 \leq k \leq 2^{t-1}
                        \end{equation*}
where $C_3 = \max \left\{S_{\beta_{1, 0} }( x_{1, 1}), \; 2\diam^2 L_f + 2\diam^2 \sqrt{16L_f^2 + \frac{196L_A^2}{\beta_0^2}} +  \frac{2\diam^2 L_A}{\beta_0}  \right\} $;
\item For $\xi \sim P(\xi)$ and $\xi_{\mathcalorigin{Q}_t} \text{ such that } \lvert  \mathcalorigin{Q}_t \rvert = \lceil \frac{2K_t}{\beta_{t, 1}^2}\rceil$,
                        \begin{equation*}
                                \Exp{S_{\beta_{t, k}}(x_{t, k + 1})}{} \leq \frac{C_4}{\sqrt{K_t + k + 1}}, \;\; \forall t \in \mathbb{N}, \, 1 \leq k \leq 2^{t-1}
                        \end{equation*}
                        where $C_4 = \max \left\{ S_{\beta_{1,0}}(x_{1, 1})  , \; 2\diam^2 L_f +  \frac{2\diam^2 L_A}{\beta_0}  + 2\diam \sqrt{16L_f^2\diam^2 + 2L_A^2\diam^2 \left( \frac{98}{\beta_0^2} + 1\right) + 2\beta_0^2\sigma_f^2}\right\}$.
\end{itemize}
\end{reptheorem}

\begin{proof}

The proof has two steps, coming from the nested loop structure of Algorithm~\ref{alg:homotopy-spider-fw}. We first determine the recursion for $S_{\beta_{t,k}}(x_{t, k+1})$ for all the iterates of the inner loop (constant $t$) and then show that the recursion holds at the `edges' i.e., when going from $t-1$ to $t$. 

\medskip

{\large\textbf{1. Convergence recursion}}\\
{\large\textbf{1.1 Recursion of $S_{\beta_{t, k}}$ for constant $t$ (inner loop)}}

Using Technical~observation~\ref{list:technical-obs-g-smoothness}, the definition of $\diam$ and the optimality of $w_{t, k}$:
\begin{align}
			\F{ \beta_{t, k} }{ x_{k + 1} } &= \mathbb{E}_{t, k} \left[ F_{\beta_{t,k}} (x_{t, k}, \xi)\right] \nn \\[2mm]
			&\leq \mathbb{E}_{t, k}  \left[ \F{ \beta_{t, k} }{ x_{t, k}, \xi } + \dotprod{ \gradF{ \beta_{t, k} }{ x_{t, k}, \xi } }{ x_{t, k+1} - x_{t, k} } + \frac{L_f + \frac{L_A}{\beta_{t, k}}}{2} \normsqr{x_{t, k+1} - x_{t, k}} \right] \nn\\[2mm]
				&\leq \F{ \beta_{t, k} }{ x_{t, k}} + \gamma_{t, k} \dotprod{ \gradF{ \beta_{t, k} }{ x_{t, k}} }{ w_{t, k} - x_{t, k}} + \frac{\gamma_{t, k}^2(L_f + \frac{L_A}{\beta_{t, k}})}{2}\normsqr{w_{t, k} - x_{t, k}} \nn \\[2mm]
				&\leq \F{ \beta_{t, k} }{ x_{t, k}} + \gamma_{t, k} \dotprod{ \gradF{ \beta_{t, k} }{ x_{t, k} } }{ w_{t, k} - x_{t, k} } + \frac{\diam^2\gamma_{t, k}^2}{2}(L_f + \frac{L_A}{\beta_{t, k}}) \nn\\[2mm]
				&\leq \F{ \beta_{t, k} }{ x_{t, k}} + \gamma_{t, k} \left( \dotprod{ \gradF{ \beta_{t, k} }{ x_{t, k} } - v_{t, k} }{ w_{t, k} - x_{t, k} } +  \dotprod{ v_{t, k} }{ \xopt - x_{t, k} }\right) + \frac{\diam^2\gamma_{t, k}^2}{2}(L_f + \frac{L_A}{\beta_{t, k}})\label{eq:quad-upper-bd}
\end{align}

\vspace{5mm}
We process the second term above separately, using the convexity of $f$, Technical~observation~\ref{list:technical-obs-trandinh-relationto-g} and noting that $ v_{t, k-1} - v_{t, k} =  \gradFstoc{ \beta_{t, k-1} }{ x_{t, k-1} } - \gradFstoc{ \beta_{t, k} }{ x_{t, k} }$:
\begin{align}
			\dotprod{ \gradF{ \beta_{t, k} }{ x_{t, k} } - v_{t, k} }{ w_{t, k} - x_{t, k} }  + \dotprod{ v_{t, k} }{ \xopt - x_{t, k} } & \nn \\[3mm]
			&\hspace{-60mm}=\dotprod{ \gradF{ \beta_{t, k} }{ x_{t, k} } - v_{t, k} }{ w_{t, k} - \xopt } + \dotprod{ \gradF{ \beta_{t, k} }{ x_{t, k} } - v_{t, k} }{ \xopt - x_{t, k} } \nn \\
			&\hspace{-30mm}+ \dotprod{ v_{t, k-1} - \gradFstoc{ \beta_{t, k-1} }{ x_{t, k-1} } + \gradFstoc{ \beta_{k} }{ x_{t, k} } }{ \xopt - x_{t, k} } \nn \\[3mm]
			    &\hspace{-60mm}=  \dotprod{ \gradF{ \beta_{t, k} }{ x_{t, k} } - v_{t, k} }{ w_{t, k} - \xopt } +  \dotprod{ \gradF{ \beta_{t, k} }{ x_{t, k} } - v_{t, k} + v_{t, k-1} - \gradFstoc{ \beta_{t, k - 1} }{ x_{t, k-1} }}{ \xopt - x_{t, k} } \nn\\
			    &\hspace{-20mm}+ \dotprod{ \tilde{\nabla} f(x_{t, k}, \xi_{\mathcalorigin{S}_{t,k}} )  }{ \xopt - x_{t, k} } + \dotprod{A^T(\xi_{\mathcalorigin{S}_{t,k}}) \avggradg[\beta_{t, k}]{A(\xi_{\mathcalorigin{S}_{t,k}}) x_{t, k}} }{ \xopt - x_{t, k} }\nn \\[3mm]                				
			    &\hspace{-60mm}\leq  \dotprod{ \gradF{ \beta_{t, k} }{ x_{t, k} } - v_{t, k} }{ w_{t, k} - \xopt }  + \dotprod{\gradF{ \beta_{t, k} }{ x_{t, k} } - \gradFstoc{ \beta_{t, k} }{ x_{t, k} }   }{ x_{t, k} - \xopt} \nn \\
			    &\hspace{-30mm}+ \tilde{f}(\xopt, \xi_{\mathcalorigin{S}_{t,k}}) - \tilde{f} (x_{t, k}, \xi_{\mathcalorigin{S}_{t,k}}) + \dotprod{\avggradg[\beta_{t, k}]{A(\xi_{\mathcalorigin{S}_{t,k}}) x_{t, k}} }{ A(\xi_{\mathcalorigin{S}_{t,k}})\xopt - A^T(\xi_{\mathcalorigin{S}_{t,k}})x_{t, k} } \nn\\[3mm]               
				&\hspace{-60mm}\leq \dotprod{ \gradF{ \beta_{t, k} }{ x_{t, k} } - v_{t, k} }{ w_{t, k} - \xopt }  + \dotprod{\gradF{ \beta_{t, k} }{ x_{t, k} } - \gradFstoc{ \beta_{t, k} }{ x_{t, k} }   }{ x_{t, k} - \xopt} \nn \\
				&\hspace{-50mm} + \tilde{f}(\xopt, \xi_{\mathcalorigin{S}_{t,k}}) + \underbrace{\tilde{g}(A(\xi_{\mathcalorigin{S}_{t,k}}) \xopt)}_{ = 0 \text{ a.s.}} \underbrace{- \tilde{f} (x_{t, k}, \xi_{\mathcalorigin{S}_{t,k}}) - \tilde{g}_{\beta_{t, k}}(A(\xi_{\mathcalorigin{S}_{t,k}}) x_{t, k})}_{ = - \tilde{F}_{\beta_{t, k}}(x_{t, k}, \xi_{\mathcalorigin{S}_{t,k}}) } - \frac{\beta_{t, k}}{2} \widetilde{\normsqr{\lambda^*_{\beta_{t, k}}(A(\xi_{\mathcalorigin{S}_{t,k}}) x_{t, k})}} \label{eq:inner-prod-bound}
\end{align}

\vspace{5mm}
We can now resume Equation \eqref{eq:quad-upper-bd} by plugging in the inequality in \eqref{eq:inner-prod-bound}, subtracting $\fopt$ from both sides, and taking the conditional expectation $\Expk{ X } = \Exp{X}{\mathcalorigin{F}_{t,k}}$. 
\begin{align}
	\Expk{\F{ \beta_{t, k} }{ x_{k + 1} } - \fopt} &\nn\\[3mm]
	        &\hspace{-35mm}\leq \mathbb{E}_{t, k} \left[ \F{ \beta_{t, k} }{ x_{t, k}} + \gamma_{t, k} \left( \dotprod{ \gradF{ \beta_{t, k} }{ x_{t, k} } - v_{t, k} }{ w_{t, k} - x_{t, k} } +  \dotprod{ v_{t, k} }{ \xopt - x_{t, k} }\right) + \frac{\diam^2\gamma_{t, k}^2}{2}\left(L_f + \frac{L_A}{\beta_{t, k}}\right) \right] - \fopt \nn \\[5mm]
	        &\hspace{-35mm}\leq \F{ \beta_{t, k} }{ x_{t, k}}  + \gamma_{t, k} \Bigg( \dotprod{ \gradF{ \beta_{t, k} }{ x_{t, k} } - v_{t, k} }{ w_{t, k} - \xopt }  + \underbrace{\mathbb{E}_{t, k} \left[\dotprod{\gradF{ \beta_{t, k} }{ x_{t, k} } - \gradFstoc{ \beta_{t, k} }{ x_{t, k} }   }{ x_{t, k} - \xopt} \right]}_{ = 0, \text{ \tiny  unbiasedness}}    \nn \\
	        &\hspace{-30mm}+ \mathbb{E}_{t, k} \left[\tilde{f}(\xopt, \xi_{\mathcalorigin{S}_{t,k}}) - \tilde{F}_{\beta_{t, k}}(x_{t, k}, \xi_{\mathcalorigin{S}_{t,k}}) - \frac{\beta_{t, k}}{2} \widetilde{\normsqr{\lambda^*_{\beta_{t, k}}(A(\xi_{\mathcalorigin{S}_{t,k}}) x_{t, k})}}\right]       \Bigg)+ \frac{\diam^2\gamma_{t, k}^2}{2}\left(L_f + \frac{L_A}{\beta_{t, k}}\right) - \fopt \nn \\[5mm]
	        &\hspace{-35mm}\leq \F{ \beta_{t, k} }{ x_{t, k}}  + \gamma_{t, k} \Bigg( \dotprod{ \gradF{ \beta_{t, k} }{ x_{t, k} } - v_{t, k} }{ w_{t, k} - \xopt }  + \fopt - \F{ \beta_{t, k} }{ x_{t, k}} - \mathbb{E}_{t, k} \left[ \frac{\beta_{t, k}}{2} \widetilde{\normsqr{\lambda^*_{\beta_{t, k}}(A(\xi_{\mathcalorigin{S}_{t,k}}) x_{t, k})}}\right]  \Bigg)     \nn \\
	        &\hspace{80mm} + \frac{\diam^2\gamma_{t, k}^2}{2}\left(L_f + \frac{L_A}{\beta_{t, k}}\right) - \fopt \nn \\[5mm]
	        &\hspace{-35mm}= (1 - \gamma_{t, k}) (\F{ \beta_{t, k} }{ x_{t, k} } - \fopt) + \gamma_{t, k} \dotprod{ \gradF{ \beta_{t, k} }{ x_{t, k} } - v_{t, k} }{ w_{t, k} - \xopt }  - \frac{\gamma_{t, k} \beta_{t, k}}{2} \Expk{\widetilde{\normsqr{ \lambda^*_{\beta_{t, k}}(A(\xi_{\mathcalorigin{S}_{t,k}}) x_{t, k})}}}  \nn\\
			&\hspace{90mm} + \frac{\diam^2\gamma_{t, k}^2}{2}\left(L_f + \frac{L_A}{\beta_{t, k}}\right) \nn
\end{align}

\vspace{5mm}
Using Technical~observation~\ref{list:technical-obs-trandinh-relationto-prev-beta} we observe that 
\begin{align*}
    \F{\beta_{t, k}}{x_{t, k}} &=  \mathbb{E}_{t,k}\left[ \tilde{f}(x_{t, k}, \xi_{\mathcalorigin{S}_{t,k}}) + \tilde{g}_{\beta_{t, k}}(A(\xi_{\mathcalorigin{S}_{t,k}}) x_{t, k})\right] \\
        &\leq \mathbb{E}_{t,k}\left[ \tilde{f}(x_{t, k}, \xi_{\mathcalorigin{S}_{t,k}}) + \tilde{g}_{\beta_{t, k-1}}(A(\xi_{\mathcalorigin{S}_{t,k}}) x_{t, k}) + \frac{\beta_{t, k-1} - \beta_{t, k}}{2} \widetilde{\normsqr{\lambda^*_{\beta_{t, k}}(A(\xi_{\mathcalorigin{S}_{t,k}})  x_{t, k})}} \right] \\
        &= \F{\beta_{t, k-1}}{x_{t, k}} +  \mathbb{E}_{t,k}\left[ \frac{\beta_{t, k-1} - \beta_{t, k}}{2} \widetilde{\normsqr{\lambda^*_{\beta_{t, k}}(A(\xi_{\mathcalorigin{S}_{t,k}})  x_{t, k})}} \right]. \\
\end{align*}

\vspace{5mm}
Using the above and the definition of $\diam$, we continue the inequality as:
\begin{align}
	\Expk{\F{ \beta_{t, k} }{ x_{k + 1} } - \fopt} &\leq (1 - \gamma_{t, k}) (\F{\beta_{t, k - 1}}{x_{t, k}} - \fopt) + \gamma_{t, k} \diam\norm{\gradF{\beta_{t, k}}{x_{t, k}} - v_{t, k}}    \nn \\
					&  + \frac{(1 - \gamma_{t, k})(\beta_{t, k-1} - \beta_{t, k}) - \gamma_{t, k} \beta_{t, k}}{2} \Expk{\widetilde{\normsqr{ \lambda^*_{\beta_{t, k}}(A(\xi_{\mathcalorigin{S}_{t,k}}) x_{t, k})}}} + \frac{\diam^2\gamma_{t, k}^2}{2}\left(L_f + \frac{L_A}{\beta_{t, k}}\right)  \label{quad-upper-bd-cond-exp}
\end{align}

\vspace{5mm}
Using the stated parameter rates, we notice that $(1 - \gamma_{t, k})(\beta_{t, k-1} - \beta_{t, k}) - \gamma_{t, k} \beta_{t, k} < 0$, as follows:
\begin{align}
 \MoveEqLeft[7] \left(1 - \frac{2}{K_t + k}\right)\left(\frac{\beta_0}{\sqrt{K_t + k - 1}}  - \frac{\beta_0}{\sqrt{K_t + k}}\right) - \frac{2\beta_0}{(K_t + k) \sqrt{K_t + k}}  \nn\\
={}& \frac{\beta_0}{\sqrt{K_t + k - 1}} - \frac{\beta_0}{\sqrt{K_t + k}} - \frac{2\beta_0}{(K_t + k)\sqrt{K_t + k - 1}} \nn\\
={}&  \beta_0\frac{K_t + k - \sqrt{K_t + k} \sqrt{K_t + k - 1} - 2 }{(K_t + k)\sqrt{K_t + k - 1}}\nn\\
={}& \beta_0 \frac{(K_t + k - 1) - 2 \sqrt{\frac{K_t + k}{4}} \sqrt{K_t + k - 1} + \frac{K_t + k}{4} - \frac{K_t + k}{4} - 1}{(K_t + k)\sqrt{K_t + k - 1}} \nn\\
={}& \beta_0 \frac{(\sqrt{K_t + k - 1} - \frac{\sqrt{K_t + k}}{2})^2 - \frac{K_t + k}{4} - 1}{(K_t + k)\sqrt{K_t + k - 1}} \nn\\
={}& \beta_0 \frac{ (\sqrt{K_t + k - 1} - \frac{\sqrt{K_t + k}}{2} - \frac{\sqrt{K_t + k}}{2}) (\sqrt{K_t + k - 1} - \frac{\sqrt{K_t + k}}{2} + \frac{\sqrt{K_t + k}}{2}) - 1}{(K_t + k)\sqrt{K_t + k - 1}} \nn\\
={}& \beta_0 \frac{ \overbrace{(\sqrt{K_t + k - 1} - \sqrt{K_t + k})}^{< 0} \sqrt{K_t + k - 1} - 1}{(K_t + k)\sqrt{K_t + k - 1}} \nn\\
<{}& 0 \label{eq:beta-gamma-difference-smaller-than-zero}\\
\nn
\end{align}

\vspace{5mm}
Finally, noting the definition of $S_{\beta_{t, k} }( x_{t, k + 1})$ and taking full expectation non both sides, we arrive at:
\begin{align}
        \Exp{S_{\beta_{t, k} }( x_{t, k + 1}) }{} &\leq  (1 - \gamma_{t, k}) \Exp{S_{\beta_{t, k - 1} }( x_{t, k}) }{} + \gamma_{t, k} \diam \Exp{\norm{\gradF{\beta_{t, k}}{x_{t, k}} - v_{t, k}}}{}  + \frac{\diam^2\gamma_{t, k}^2}{2}\left(L_f + \frac{L_A}{\beta_{t, k}}\right) \label{eq:var-upper-bd2}
\end{align}

\medskip

\textbf{1.2 Recursion of $S_{\beta_{t, k}}$ at the `edges'}

We now want to show that the same recursion holds when going for $S_{\beta_{t, 1} }( x_{t, 2}) $ and $S_{\beta_{t - 1, K_{t - 1} }}(x_{t - 1, K_{t - 1} + 1})$. We follow similar steps as in the previous section (which we shorten this time for conciseness). Using smoothness and the fact that from Algorithm~\ref{alg:homotopy-spider-fw} we have $x_{t, 1} = x_{t - 1, K_{t - 1} + 1}$:
\begin{align}
			\F{ \beta_{t, 1} }{ x_{t, 2} } &\leq \F{ \beta_{t,  1} }{ x_{t, 1} } + \gamma_{t, 1} \dotprod{ \gradF{ \beta_{t, 1 }}{ x_{t, 1} } }{ w_{t, 1} - x_{t, 1} } + \frac{\diam^2\gamma_{t, 1}^2}{2}\left( L_f + \frac{L_A}{\beta_{t, 1}}\right)\label{eq:quad-upper-bd-edges}
\end{align}

\vspace{5mm}
Since $v_{t, 1} = \gradF{ \beta_{t, 1 }}{ x_{t, 1} } $ and $w_{t, 1} = \text{lmo}_{\X}(v_{t, 1})$, we have that $ \dotprod{ \gradF{ \beta_{t, 1 }}{ x_{t, 1} } }{ w_{t, 1} - x_{t, 1} } \leq  \dotprod{ \gradF{ \beta_{t, 1 }}{ x_{t, 1} } }{ x^* - x_{t, 1} }$. Further using the definition of $F_{\beta}$, the convexity of $f$ and Technical~observation~\ref{list:technical-obs-trandinh-relationto-g}, we have:
\begin{align}
        \dotprod{ \gradF{ \beta_{t, 1 }}{ x_{t, 1} } }{ w_{t, 1} - x_{t, 1} } &\leq  \dotprod{ \gradF{ \beta_{t, 1 }}{ x_{t, 1} } }{ x^* - x_{t, 1} } \nn \\[2mm]
                    &= \dotprod{ \nabla f(x_{t, 1}) + \nabla_x G_{\beta_{t, 1}}(Ax_{t, 1}) }{ x^* - x_{t, 1} } \nn \\[2mm]
                    &\leq \fopt - f(x_{t, 1}) + \mathbb{E}_{t, 1} \left[ \dotprod{ \tilde{\nabla}_x g_{\beta_{t, 1}}(A(\xi_{\mathcalorigin{Q}_t})x_{t, 1}) }{ x^* - x_{t, 1} } \right]\nn \\[2mm] 
                    &\leq \fopt - f(x_{t, 1}) + \mathbb{E}_{t, 1} \left[ \underbrace{\tilde{g}(A(\xi_{\mathcalorigin{Q}_t}) \xopt)}_{= 0 \text{ \tiny a.s.}} - \tilde{g}_{\beta_{t, 1}} (A(\xi_{\mathcalorigin{Q}_t}) x_{t, 1}) - \frac{\beta_{t, 1}}{2} \widetilde{\normsqr{\lambda_{\beta_{t, 1}}^*(A(\xi_{\mathcalorigin{Q}_t}) x_{t, 1})} } \right]\nn \\[2mm]
                    &\leq \fopt \underbrace{- f(x_{t, 1}) - G_{\beta_{t, 1}}(Ax_{t, 1})}_{= - F_{\beta_{t, 1}}(x_t, 1)}  - \frac{\beta_{t, 1}}{2} \mathbb{E}_{t, 1} \left[\widetilde{\normsqr{\lambda_{\beta_{t, 1}}^*(A(\xi_{\mathcalorigin{Q}_t}) x_{t, 1})} } \right] \label{eq:tran-dinh-G-inequality-application}
\end{align}


\vspace{5mm}
Another remark is that we can still make the transition from $\F{ \beta_{t,  1} }{ x_{t, 1} }$ to $\F{ \beta_{t -1,  K_{t-1}} }{ x_{t, 1} }$ using Technical~observation~\ref{list:technical-obs-trandinh-relationto-prev-beta}, since the $\beta$'s are `continuous' at the edge: $\beta_{t -1,  K_{t-1}} = \frac{\beta_0}{\sqrt{K_{t-1} + K_{t-1}}} = \frac{\beta_0}{\sqrt{K_{t}}}$ and $\beta_{t,  1} = \frac{\beta_0}{\sqrt{K_{t} + 1}}$. We thus have:
\begin{align}
    \F{\beta_{t, 1}}{x_{t, 1}} &=  \mathbb{E}_{t,1}\left[ \tilde{f}(x_{t, 1}, \xi_{\mathcalorigin{Q}_{t}}) + \tilde{g}_{\beta_{t, 1}}(A(\xi_{\mathcalorigin{Q}_{t}}) x_{t, 1})\right] \nn\\
        &\leq \mathbb{E}_{t,1}\left[ \tilde{f}(x_{t, 1}, \xi_{\mathcalorigin{Q}_{t}}) + \tilde{g}_{\beta_{t-1, K_{t-1}}}(A(\xi_{\mathcalorigin{Q}_{t}}) x_{t, 1}) + \frac{\beta_{t-1, K_{t-1}} - \beta_{t, 1}}{2} \widetilde{\normsqr{\lambda^*_{\beta_{t, 1}}(A(\xi_{\mathcalorigin{Q}_{t}})  x_{t, 1})}} \right] \nn\\
        &= \F{\beta_{t-1, K_{t-1}}}{x_{t, 1}} +  \mathbb{E}_{t,1}\left[ \frac{\beta_{t-1, K_{t-1}} - \beta_{t, 1}}{2} \widetilde{\normsqr{\lambda^*_{\beta_{t, 1}}(A(\xi_{\mathcalorigin{Q}_{t}})  x_{t, 1})}} \right] \label{eq:tran-at-the-edge}
\end{align}

\vspace{5mm}
Inserting ~\eqref{eq:tran-at-the-edge} and~\eqref{eq:tran-dinh-G-inequality-application} into~\eqref{eq:quad-upper-bd-edges}:
\begin{align}
\F{ \beta_{t, 1} }{ x_{t, 2} } &\leq (1 - \gamma_{t, 1})\F{ \beta_{t,  1} }{ x_{t, 1} }  + \gamma_{t, 1} \fopt - \frac{\gamma_{t, 1}\beta_{t, 1}}{2} \mathbb{E} \left[ \widetilde{\normsqr{\lambda^*_{\beta_{t, 1}}(A(\xi_{\mathcalorigin{Q}_{t}})  x_{t, 1})}}\right]   + \frac{\diam^2\gamma_{t, 1}^2}{2}\left( L_f + \frac{L_A}{\beta_{t, 1}}\right)\nn \\
                    &\leq (1 - \gamma_{t, 1}) \F{\beta_{t-1, K_{t-1}}}{x_{t, 1}}  + \gamma_{t, 1} \fopt + \underbrace{\frac{ (1 - \gamma_{t, 1})(\beta_{t-1, K_{t-1}} - \beta_{t, 1})- \gamma_{t, 1}\beta_{t, 1}}{2}}_{<0, \text{ \tiny as before}} \mathbb{E}_{t, 1} \left[ \widetilde{\normsqr{\lambda^*_{\beta_{t, 1}}(A(\xi_{\mathcalorigin{Q}_{t}})  x_{t, 1})}} \right]  \nn \\
                    &\hspace{120mm}+ \frac{\diam^2\gamma_{t, 1}^2}{2}\left( L_f + \frac{L_A}{\beta_{t, 1}}\right) \nn
\end{align}

Finally, subtracting $\fopt$ from both sides and taking the expectation, we have:
\begin{align}
    \mathbb{E} \left[ S_{\beta_{t, 1} }( x_{t, 2} ) \right] &\leq (1 - \gamma_{t, 1}) \mathbb{E} \left[ S_{\beta_{t - 1, K_{t - 1}}}(x_{t, 1} )\right] + \frac{\diam^2\gamma_{t, 1}^2}{2}\left( L_f + \frac{L_A}{\beta_{t, 1}}\right)\label{eq:rec-at-edges}
\end{align}

{\large\textbf{2. Convergence rates for the finite sum case}}

For ease, we first cast the index pairs $(t, k)$ to their corresponding global index counterparts (in a sense, we flatten the double loop structure). The variables indexed by $(t, k)$ can be seen as equivalently indexed by $\pmb{\kappa}(t, k) = K_t + k \defeq 2^{t - 1} + k, \, t \in \mathbb{N}, \, k \in \{1, \ldots 2^{t-1}\}$. 

\vspace{3mm}

The following properties hold for $\pmb{\kappa}$: 
\begin{itemize}
    \item $\pmb{\kappa}(t, k + 1) = \pmb{\kappa}(t, k) + 1$
    \item $\pmb{\kappa}(t - 1, K_{t-1} + 1) = \pmb{\kappa}(t - 1, K_{t-1}) + 1 = \pmb{\kappa}(t, 1)$ (the `increment-by-one'  rule holds between the last iteration of epoch $t-1$ and the first iteration of epoch $t$)
\end{itemize}
In other words, $\pmb{\kappa}(t, k)$ returns for iteration $(t, k)$ its global index since the beginning of Algorithm~\ref{alg:homotopy-spider-fw}.

\vspace{5mm}
We use this new indexing scheme and its properties to rewrite relations~\ref{eq:var-upper-bd2}~and~\ref{eq:rec-at-edges} into a single, global inequality. Note that here $\pmb{\kappa}$ should be read as $\pmb{\kappa}(t, k)$, for some given, arbitrary $t, k$.
\begin{align} \Exp{S_{\beta_{\pmb{\kappa}} }( x_{\pmb{\kappa}+ 1}) }{} &\leq (1 - \gamma_{\pmb{\kappa}}) \Exp{S_{\beta_{\pmb{\kappa}- 1} }( x_{\pmb{\kappa}}) }{} + \gamma_{\pmb{\kappa}} \diam \Exp{\norm{\gradF{\beta_{\pmb{\kappa}}}{x_{\pmb{\kappa}}} - v_{\pmb{\kappa}}}}{} + \frac{\diam^2\gamma_{\pmb{\kappa}}^2}{2} \left( L_f + \frac{L_A}{\beta_{\pmb{\kappa}}} \right) \label{eq:var-upper-bd1}
\end{align}

\vspace{5mm}
Further replacing the parameter rates and the variance bound of Lemma~\ref{lem:spiderfw-finite-sum-var} (subject to Jensen's inequality):                      
\begin{align*}
\Exp{S_{\beta_{\pmb{\kappa}} }( x_{\pmb{\kappa}+ 1}) }{} &= \left(1 - \frac{2}{\pmb{\kappa}}\right) \Exp{S_{\beta_{\pmb{\kappa} - 1} }( x_{\pmb{\kappa}}) }{} +  \frac{2\diam^2 \sqrt{16L_f^2 + \frac{196L_A^2}{\beta_0^2}}}{\pmb{\kappa}\sqrt{\pmb{\kappa}}} + \frac{2\diam^2 L_f}{\pmb{\kappa}^2} +  \frac{2\diam^2 L_A}{\beta_0\pmb{\kappa}\sqrt{\pmb{\kappa}}} \\
        &\leq \left(1 - \frac{2}{\pmb{\kappa}}\right) \Exp{S_{\beta_{\pmb{\kappa} - 1} }( x_{\pmb{\kappa}}) }{} + \frac{1}{\pmb{\kappa}^{3/2}}\left(2\diam^2 L_f + 2\diam^2 \sqrt{16L_f^2 + \frac{196L_A^2}{\beta_0^2}} +  \frac{2\diam^2 L_A}{\beta_0}\right)
\end{align*}

\vspace{5mm}
We can now apply Lemma \ref{lemma: recursion}, with $\alpha = 1$, $\beta = 3/2$, $b = 2\diam^2 L_f + 2\diam^2 \sqrt{16L_f^2 + \frac{196L_A^2}{\beta_0^2}} +  \frac{2\diam^2 L_A}{\beta_0}$, $c = 2$, $k_0 = 0$ and \\\mbox{$C_3 = \max \left\{S_{\beta_{1, 0} }( x_{1, 1}), \; 2\diam^2 L_f + 2\diam^2 \sqrt{16L_f^2 + \frac{196L_A^2}{\beta_0^2}} +  \frac{2\diam^2 L_A}{\beta_0}  \right\}$} to get:
\begin{align*}
    &\Exp{\F{ \beta_{\pmb{\kappa}} }{ x_{\pmb{\kappa} + 1} } - \fopt}{} \leq \frac{C_3}{\sqrt{\pmb{\kappa} + 1}}\\[2mm]
    &\hspace{20mm}\scalebox{1.8}{$\Updownarrow$} \\[2mm]
    &\Exp{S_{ \beta_{t, k} }( x_{t, k + 1} )}{} \leq \frac{C_3}{\sqrt{K_t + k + 1}}
\end{align*}

\medskip

{\large\textbf{2. Convergence rates for the general expectation case}}

Following the same steps for the general expectation case, we get:

\begin{align*}
\Exp{S_{\beta_{\pmb{\kappa}} }( x_{\pmb{\kappa}+ 1}) }{} &= \left(1 - \frac{2}{\pmb{\kappa}}\right) \Exp{S_{\beta_{\pmb{\kappa} - 1} }( x_{\pmb{\kappa}}) }{} +  \frac{2\diam \sqrt{16L_f^2\diam^2 + 2L_A^2\diam^2 \left( \frac{98}{\beta_0^2} + 1\right) + 2\beta_0^2\sigma_f^2}}{\pmb{\kappa}\sqrt{\pmb{\kappa}} } + \frac{2\diam^2 L_f}{\pmb{\kappa}^2} +  \frac{2\diam^2 L_A}{\beta_0\pmb{\kappa}\sqrt{\pmb{\kappa}}} \\        
&\leq \left(1 - \frac{2}{\pmb{\kappa}}\right) \Exp{S_{\beta_{\pmb{\kappa} - 1} }( x_{\pmb{\kappa}}) }{} + \frac{1}{\pmb{\kappa}^{3/2}} \left( 2\diam^2 L_f +  \frac{2\diam^2 L_A}{\beta_0}  + 2\diam \sqrt{16L_f^2\diam^2 + 2L_A^2\diam^2 \left( \frac{98}{\beta_0^2} + 1\right) + 2\beta_0^2\sigma_f^2} \right)
\end{align*}

\vspace{5mm}
We can now apply Lemma \ref{lemma: recursion}, with $b =  2\diam^2 L_f +  \frac{2\diam^2 L_A}{\beta_0}  + 2\diam \sqrt{16L_f^2\diam^2 + 2L_A^2\diam^2 \left( \frac{98}{\beta_0^2} + 1\right) + 2\beta_0^2\sigma_f^2}$, $c = 2$, $\alpha = 1$, $\beta = 3/2$,  and \mbox{$C_4 = \max \left\{ S_{\beta_{1,0}}(x_{1, 1})  , \; 2\diam^2 L_f +  \frac{2\diam^2 L_A}{\beta_0}  + 2\diam \sqrt{16L_f^2\diam^2 + 2L_A^2\diam^2 \left( \frac{98}{\beta_0^2} + 1\right) + 2\beta_0^2\sigma_f^2}\right\}$} to get 
\begin{equation*}
    \Exp{S_{\beta_{t, k} }(x_{1, k + 1}) }{} \leq \frac{C_4}{\sqrt{K_t + k + 1}} \hfill \qed
\end{equation*}

\end{proof}

\vspace{5mm}

\begin{repcorollary}{cor:spiderfw}
The expected convergence in terms of objective suboptimality and feasibility of Algorithm~\ref{alg:homotopy-spider-fw} is, respectively, 
\begin{align*}
    &\mathbb{E}\left[ \lVert f(x_{t, k}) - \fopt \rVert \right] &&\hspace{-42mm}\in \bigO{(K_t + k)^{-1/2}}\\
    &\sqrt{\mathbb{E} \left[ \dist(\Astoc x_{t, k}, \bstoc)^2 \right] } &&\hspace{-42mm}\in \bigO{(K_t + k)^{-1/2}}
\end{align*} 
for both the finite-sum and the general expectation setting, up to constants. Consequently, the oracle complexity is given by $\ifo \in \bigO{n \log_2(\epsilon^-2)+ \epsilon^{-4}}$ and $\lmo \in \bigO{\epsilon^{-2}}$ for the finite-sum setting, and by $\sfo \in \bigO{\epsilon^{-4}}$ and $\lmo \in \bigO{\epsilon^{-2}}$ for the more general expectation setting.
\end{repcorollary}

\begin{proof}
    A simple application of Lemma~3.1 in \citep{fercoq2019almost} for the previously derived convergence bounds of the smoothed gap, along with our chosen decrease rate for $\beta$ yield the stated results.
    
    For the oracle complexities, we  choose a total number of outer loops $T_{\epsilon}$ in order to achieve a desired $\epsilon$-accuracy. 
\begin{align*}
 \frac{1}{\sqrt{K_t + k }} \leq \epsilon \implies \frac{1}{\epsilon^2} \leq K_t + k  \leq 2^{t} \implies T_{\epsilon} \geq \log_{2} \left( \epsilon^{-2}\right)\\
\end{align*}

We can now state the corresponding complexity in terms of $\#(ifo)$ and $\#(lmo)$ for the finite-sum case of Algorithm~\ref{alg:homotopy-spider-fw}:
\begin{align*}
    \#(ifo) &= \sum_{t = 1}^{T_{\epsilon}} \left( n + \sum_{k=2}^{K_{t}} K_t \right) \\
        &= \sum_{t = 1}^{T_{\epsilon}} \left( n + 2^{2(t-1)}\right)  \\
        &= nT_{\epsilon} + \bigO{2^{2T_{\epsilon}}} \in \bigO{\epsilon^{-4}} \\[3mm]
     \#(lmo) &= \sum_{t = 1}^{T_{\epsilon}} K_t \leq 2K_{T_{\epsilon}} = 2^{T_{\epsilon}} \in \bigO{\epsilon^{-2}}    \\ 
\end{align*}

\vspace{5mm}
For the general expectation case, following the same steps, we get:
\begin{align*}
    \#(sfo) &= \sum_{t = 1}^{T_{\epsilon}} \left( \lvert \mathcalorigin{Q}_t \rvert + \sum_{k=2}^{K_{t}} K_t \right) \\
        &= \sum_{t = 1}^{T_{\epsilon}} \left( \left\lceil \frac{2K_t}{\beta_{t, 1}^2}\right\rceil + 2^{2(t-1)}\right)  \\
        &\leq \sum_{t = 1}^{T_{\epsilon}} \left(  \frac{2K_t}{\beta_{t, 1}^2} + 1 + 2^{2(t-1)}\right)  \\
                &= \sum_{t = 1}^{T_{\epsilon}} \left(  \frac{2^{t} (2^{t-1} + 1)}{\beta_0^2} + 1 + 2^{2(t-1)}\right)  \\
                &= \underbrace{\frac{1}{\beta_0^2} \sum_{t = 1}^{T_{\epsilon}} 2^{2t-1}}_{ \substack{\in \bigO{2^{2T_{\epsilon}}} \\[2pt] \equiv \bigO{\epsilon^{-4}} } } + \underbrace{\frac{1}{\beta_0^2} \sum_{t = 1}^{T_{\epsilon}} 2^{t}}_{ \substack{\in \bigO{2^{T_{\epsilon}}}  \\[2pt] \equiv \bigO{\epsilon^{-2}} }} + \underbrace{T_{\epsilon}}_{\in \bigO{\log_2(\epsilon^{-2})}} + \underbrace{\sum_{t = 1}^{T_{\epsilon}} 2^{2(t-1)}}_{\substack{ \in \bigO{2^{2T_{\epsilon}}} \\[2pt] \equiv \bigO{\epsilon^{-4}}} } \\
         &\in \bigO{\epsilon^{-4}} \\[3mm]
     \#(lmo) &= \sum_{t = 1}^{T_{\epsilon}} K_t \leq 2K_{T_{\epsilon}} = 2^{T_{\epsilon}} \in \bigO{\epsilon^{-2}} \qed
\end{align*} 
\end{proof}

\end{document}